\documentclass[epsfig,a4paper,11pt,titlepage,twoside,openany]{book}
\usepackage{epigraph} 
\usepackage{subcaption}
\usepackage{epsfig}
\usepackage{plain}
\usepackage{setspace}
\usepackage[paperheight=29.7cm,paperwidth=21cm,outer=1.5cm,inner=2.5cm,top=2cm,bottom=2cm]{geometry} 
\usepackage{titlesec} 
\providecommand{\introduce}[1]{\textit{#1}}

\usepackage{amsmath}
\usepackage{amssymb}
\usepackage{chngcntr}
\usepackage{stmaryrd}
\counterwithout{footnote}{section}
\usepackage{mathtools}

\usepackage{booktabs}       
\usepackage{amsfonts}       
\usepackage{nicefrac}       
\usepackage{microtype}      
\usepackage{xcolor}         
\usepackage{pifont}
\newcommand{\xmark}{\ding{55}}

\newcommand{\first}[1]{\mathbf{\textcolor{red}{#1}}}
\newcommand{\second}[1]{\mathbf{\textcolor{blue}{#1}}}
\newcommand{\third}[1]{\mathbf{\textcolor{violet}{#1}}}

\usepackage{array}
\newcolumntype{H}{>{\setbox0=\hbox\bgroup}c<{\egroup}@{}}

\usepackage[utf8x]{inputenc} 
\usepackage[english]{babel}
\usepackage{amsthm}
\theoremstyle{definition}
\newtheorem{theorem}{Theorem}[section]

\theoremstyle{definition}

\newtheorem{definition}{Definition}[section]
\theoremstyle{definition}
\newtheorem{proposition}{Proposition}[section]
\theoremstyle{remark}
\newtheorem{example}{Example}[section]
\newtheorem*{remark}{Remark}
\usepackage{amsmath}
\newcommand\restr[2]{{
  \left.\kern-\nulldelimiterspace 
  #1 
  \littletaller 
  \right|_{#2} 
  }}

\let\origdoublepage\cleardoublepage
\newcommand{\clearemptydoublepage}{%
  \clearpage
  {\pagestyle{empty}\origdoublepage}%
}
\usepackage[shortlabels]{enumitem}
\newcommand{\littletaller}{\mathchoice{\vphantom{\big|}}{}{}{}}
\begin{document}

  \pagenumbering{gobble} 
  \pagestyle{plain}

\thispagestyle{empty}

\begin{center}
  \begin{figure}[h!]
    \centerline{\psfig{file=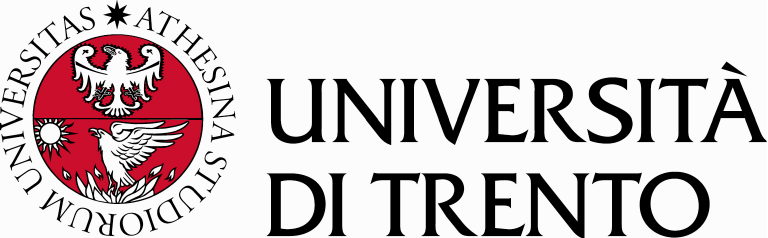,width=0.6\textwidth}}
  \end{figure}

  \vspace{2 cm} 

  \LARGE{Department of Information Engineering and Computer Science\\}

  \vspace{1 cm} 
  \Large{Master's Degree in\\
    Artificial Intelligence Systems
  }

  \vspace{2 cm} 
  \Large\textsc{Final Dissertation\\} 
  \vspace{1 cm} 
  \Huge\textsc{Nonlinear Sheaf Diffusion \\ in Graph Neural Networks\\}
  \Large{\it{}}

  \vspace{2 cm} 
  \begin{tabular*}{\textwidth}{ c @{\extracolsep{\fill}} c }
  \Large{Supervisor} & \Large{Student}\\
  \Large{Prof. Andrea Passerini}& \Large{Olga Zaghen}\\
  \Large{} & \Large{}\\
  \Large{Co-Supervisor} & \Large{}\\
  \Large{Prof. Pietro Liò}& \Large{}
  \end{tabular*}

  \vspace{2 cm} 

  \Large{Academic year 2022/2023}
  
\end{center}

  \clearpage
 
  \clearpage
  \;
  \newpage
  \clearpage
  \pagestyle{plain} 

  \mainmatter


    \tableofcontents
    \clearpage

    \begingroup
      \renewcommand{\cleardoublepage}{} 
      \renewcommand{\clearpage}{} 
      
      \titleformat{\chapter}
        {\normalfont\Huge\bfseries}{\thechapter}{1em}{}
        
      \titlespacing*{\chapter}{0pt}{0.59in}{0.02in}
      \titlespacing*{\section}{0pt}{0.20in}{0.02in}
      \titlespacing*{\subsection}{0pt}{0.10in}{0.02in}
      
      \chapter*{Abstract} 
\label{abtract}

\addcontentsline{toc}{chapter}{Abstract} 
This work focuses on exploring the potential benefits of introducing a nonlinear Laplacian in Sheaf Neural Networks for graph-related tasks. The primary aim is to understand the impact of such nonlinearity on diffusion dynamics, signal propagation, and performance of neural network architectures in discrete-time settings. The study primarily emphasizes experimental analysis, using real-world and synthetic datasets to validate the practical effectiveness of different versions of the model. This approach shifts the focus from an initial theoretical exploration to demonstrating the practical utility of the proposed model, despite its inherent complexity and dimensionality overhead.
The project's foundations are rooted in the pioneering work of Cristian Bodnar et al., known as Neural Sheaf Diffusion \cite{bodnar2022neural}. Their contributions have provided inspiration for this thesis and opened new research directions in the field. The collaboration of topological insights and deep learning techniques promises to enhance our understanding of complex data structures from a topological perspective.

      
      %
      %
      \clearemptydoublepage

\newpage
\;

\newpage

\chapter{Summary}
\label{cha:intro}

\epigraph{There is no branch of mathematics, however abstract, which may not some day be applied to phenomena of the real world.}{\textit{Nikolaj Ivanovič Lobačevskij}}

\section{Motivation and Problem Statement}

\label{sec:motivation}

The advent of \textit{Topological Deep Learning} (TDL) is a response to the increasing availability of diverse and complex data that traditional deep neural networks have excelled at analyzing only on regular Euclidean domains, such as images and text sequences. Many scientific datasets and other types of data possess different structures that defy Euclidean geometry's constraints \cite{hajijtopological}: in order to address this challenge, the more general field of \textit{Geometric Deep Learning} (GDL) has emerged as an extension of deep learning techniques to encompass non-Euclidean domains \cite{bronstein2021geometric,zhou2020graph,wu2020comprehensive}. GDL achieves this goal by incorporating principles of geometric regularity, such as symmetries, invariance, and equivariance, which enable appropriate inductive biases for processing arbitrary data domains. These domains include sets \cite{qi2017pointnet,rempe2020caspr}, grids \cite{boscaini2015learning,masci2015geodesic}, manifolds \cite{boscaini2015learning,masci2015geodesic}, and graphs \cite{scarselli2008graph,bronstein2021geometric,kipf2016semi}. In particular, the modeling and analysis of data with graph structures have been significantly enhanced by the advent of Graph Neural Networks (GNNs) within the framework of GDL \cite{bronstein2021geometric,kipf2016semi}.

While GNNs have been successfully deployed so far, they primarily focus on local abstractions and fail to capture non-local properties and dependencies present in data \cite{hajijtopological}. In order to address this limitation, there is a need to consider the \textit{topology} of data. \textit{Topological data}, which involves interactions of edges (in graphs), triangles (in meshes), or cliques, arises naturally in various applications such as complex physical systems \cite{battiston2021physics}, traffic forecasting \cite{jiang2022graph}, and molecular design \cite{schiff2020characterizing}.
TDL expands beyond graph-based abstractions to encompass extensions like simplicial complexes, cell complexes and hypergraphs, which generalize most data domains encountered in scientific computations. It involves the development of machine learning models capable of learning from data supported on these topological domains. Furthermore, by combining the power of deep learning with the richness of algebraic topology, it enables the analysis and understanding of complex data structures from a topological perspective.

\textit{Topological Deep Learning}  is also the title of Cristian Bodnar's PhD thesis defended at the University of Cambridge. His research has made significant contributions to the field, particularly through the proposal of a novel approach that utilizes sheaves—an abstract object from category theory and algebraic topology—in the context of GNNs. Bodnar et al.'s work on Sheaf Neural Networks (SNNs), known as \textit{Neural Sheaf Diffusion} \cite{bodnar2022neural}, is of particular relevance to this thesis, as our project builds upon it by  exploring new research directions and applications.

Broadly speaking, the aim of this thesis was to investigate the potential benefits of
introducing a nonlinear Laplacian \cite{hansen2020laplacians,hansen2019toward} in SNNs for graph-related tasks. We were initially driven by simple curiosity regarding the impact of a nonlinearity in the Laplacian
on diffusion dynamics, signal propagation in discrete-time settings, and overall neural network performance. Eventually, such curiosity led to the development of a thorough analysis of the phenomena. 

This study primarily focused on experimental analysis rather than theoretical exploration. The goal was to validate the practical effectiveness of different versions of the model through tests on real-world and synthetic datasets. The emphasis on experiments is motivated by the deviation that often takes place in discrete-time settings from the conditions and assumptions made in theoretical research. Furthermore, since the theoretical aspects of sheaf diffusion convergence properties were extensively covered in the previous work by Bodnar et al. \cite{bodnar2022neural}, this study shifted its focus to demonstrating the practical usefulness of the proposed model, despite its complexity and dimensionality overheads.

\;

\section{Outline}
\label{sec:outline}

Chapter \ref{cha:maths} provides an overview of the essential mathematical concepts that form the basis for the subsequent theory discussed in the thesis. Specifically, Section \ref{sec:topology} introduces  concepts in topology theory, while in Sections \ref{sec:sheaves_general} and \ref{sec:sheaves_on_graphs} sheaves are defined, first from a general perspective and then specifically on graphs (cellular sheaves), respectively. By first delving into these mathematical concepts, the theoretical foundations of Sheaf Neural Networks should appear more precise and transparent, enhancing the readers' comprehension of the subject matter.

In Chapter \ref{cha:snn}, an overview of the most common and important Graph Neural Network models is provided (Section \ref{sec:sec_gnn}), as well as a summary of all the research on Sheaf Neural Networks that has been carried out in the last few years (Section \ref{sec:section_snn}). The goal of this chapter is providing the proper background on current state-of-the-art models and techniques that, additionally to representing the starting point for the development of our model, also constitute important benchmarks to which it needs to be compared, in order to validate it in practice. 

Chapter \ref{cha:nonlinear} delves into the core subject of this manuscript, that is the introduction, definition, and analysis of the proposed method. We also address the various phases that led to achieve the final version of the model, during which the design and evaluation of numerous variations of the method were carried out. The chapter is structured by first introducing nonlinear sheaves and bounded confidence in the context of opinion dynamics (Section \ref{opinion_dynamics}) as an example of application study, as well as a nice interpretation to better understand how sheaves work. The nonlinear Laplacian, which constitutes the main object of study, is then formally introduced in Section \ref{sec:nonlin_lap}, after which the models that we propose, as well as the main questions we faced and implementation choices we made, are described in Section \ref{sec:model_def}.

The last Chapter (\ref{cha:experiments}) is devoted to all the experiments we carried out to first of all validate all the features we tried out for the model design, and then test the properties brought in by our model with respect to the others. We evaluate its performance on benchmark datasets, comparing its results with respect to other state-of-the-art architectures. In particular, Section \ref{sec:synthetic} focuses on tests carried out in a synthetic setting, while Section \ref{sec:real_world} describes the experiments executed on real-world benchmark datasets. 



\clearemptydoublepage

\newpage

\chapter{Mathematical Preliminaries}
\label{cha:maths}

This chapter aims at introducing the main mathematical concepts on which all the subsequent theory will be based. The focus is on algebraic topology, topological spaces and category theory. These mathematical preliminaries will hopefully constitute a useful resource that aims at making the intuitions and motivations behind Sheaf Neural Networks more precise and clear.

\section{Fundamentals in Topology}
\label{sec:topology}
The word \textit{topology} refers both to a general field of study and a specific mathematical object. In the first sense, the mathematical branch of topology studies the properties and spatial relations of geometric figures that are unaffected by continuous deformations, such as the change in shape and size. The second meaning is instead explained through the following definitions.

\begin{definition}[Topology]
    Let $\mathcal{P}$ denote the power set. Formally, a \textit{topology}  on a set $X \neq \emptyset$ is a collection $\tau \subseteq \mathcal{P}(X)$  that satisfies the following properties:
    \begin{enumerate}
        \item $X, \emptyset \in \tau$;
        \item $\forall I$ set of indices and $\forall A_{i} \in \tau, i \in I, \bigcup_{i \in I} A_{i} \in \tau$;
        \item $\forall J$ finite set of indices, that is $|J| = n < +\infty$, $\forall A_{j} \in \tau, j \in J, \bigcap_{j = 1}^{n} A_{j} \in \tau$. 
    \end{enumerate}
\end{definition}

\begin{definition}
    The elements of $\tau$ are called \textit{open sets} of the topology, the couple $(X, \tau)$ is a \textit{topological space} and the elements of \textit{X} are generally referred to as \textit{points}. Intuitively, the open sets provide a neighborhood structure for the points of \textit{X}.
\end{definition}

In short and in a less formal way, a \textit{topological space} is simply a set $X$ together with a collection $\tau$ of subsets of $X$, called the \textit{open sets} of $X$, that satisfy some conditions: the empty set and $X$ belong to $\tau$, and any finite intersection and arbitrary union of open sets is an open set.

\begin{figure}[ht]
    \centering
    \includegraphics[width=5.00in]{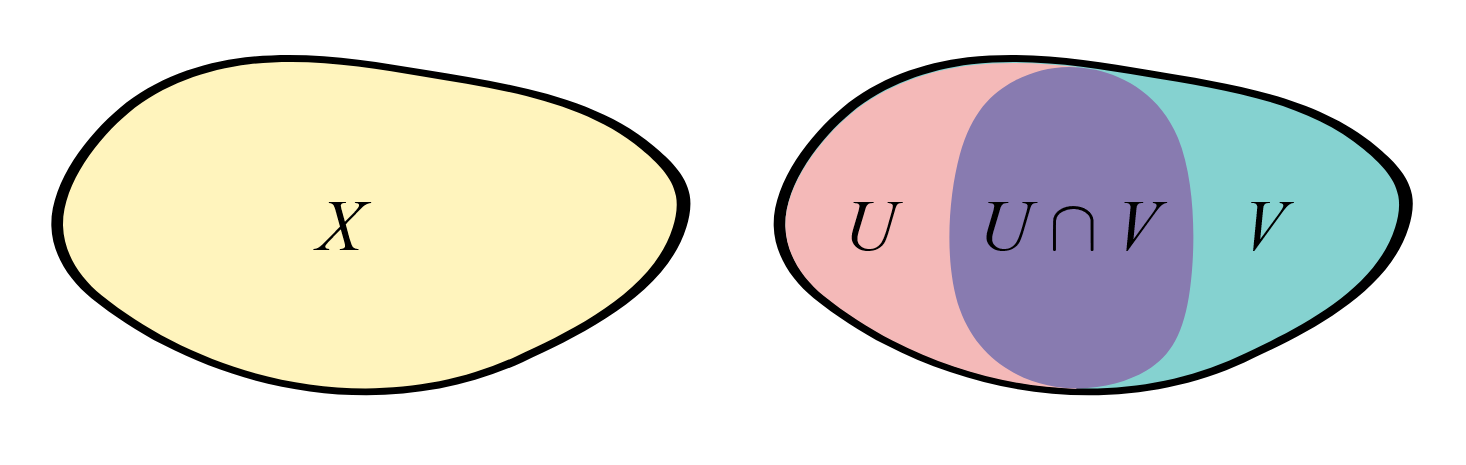}
    \caption{A topological space $(X, \tau)$ in which the topology is defined as $\tau = \{X, U, V, U \cap V, \emptyset\}$.}
    \label{fig:enter-label}
\end{figure}

\begin{example}[Topological spaces]
Additionally to the one shown in Figure ~\ref{fig:enter-label}, other simple examples of topological spaces are:
\begin{itemize}
    \item the \textit{trivial topology}, defined by $X \neq \emptyset, \tau = \{X, \emptyset\}$;
    \item the \textit{discrete topology}, defined by $X \neq \emptyset, \tau = \mathcal{P}(X)$.
\end{itemize}
\end{example}

\begin{definition}[Base of a topology]
    Let $(X, \tau)$ be a topological space. A \textit{base} $\mathcal{B}$ for the topology $\tau$ is a collection of open sets $\mathcal{B} \subseteq \tau$ such that any
    other open set in $\tau$ can be expressed as a union of elements in $\mathcal{B}$. The elements of $\mathcal{B}$
    are called \textit{basic open sets}.
\end{definition}

Another concept that will be useful in the following sections is the one of \textit{open covers} for a topological space.

\begin{definition}[Open cover]
    Let $(X, \tau)$ be a topological space. A \textit{cover} $\mathcal{C}$ of $X$  is a collection of subsets $\{U_{\alpha}\}_{\alpha \in A}$ of $X$ whose union is the whole space $X$. In this case we say that $\mathcal{C}$ covers $X$, or that the sets $\{U_{\alpha}\}$ cover $X$. We say that $\mathcal{C}$ is an \textit{open cover} if each of its members is an open set (i.e. each $\{U_{\alpha}\}$ is contained in $\tau$, $\tau$ being the topology on $X$). 
\end{definition}

\begin{remark}
    By definition, a base $\mathcal{B}$ for a topological space $(X, \tau)$ is an open cover of $X$. 
\end{remark}

\section{Sheaves: a General Definition}
\label{sec:sheaves_general}
In various branches of mathematics, numerous structures defined on a topological space $(X, \tau)$ can naturally undergo localization or restriction to open subsets $U \subseteq X$, for example continuous functions with real or complex values, $n-$times differentiable functions, bounded real-valued functions, vector fields, and sections of vector bundles on the space. The capability to confine data to smaller open subsets leads to the notion of \textit{presheaves}. Starting from this, essentially \textit{sheaves} are presheaves where local data can be glued to global data. Parts of the contents of this section were inspired by the presentation on  \textit{Topological Deep Learning} held at the Geometric Deep Learning Summer School in Pescara (July 2022) by Cristian Bodnar\footnote{https://www.youtube.com/watch?v=wACDSoDNTfE}\textsuperscript{,}\footnote{https://www.youtube.com/watch?v=90MbHphnPUU}.

\begin{definition}[Presheaf]
    Given a topological space $(X, \tau)$, a \textit{presheaf of sets} $\mathcal{F}$ on \textit{X} consists of:
    \begin{enumerate}
        \item For each open set $U \in \tau$, a set $\mathcal{F}(U)$. The elements of this set are also called the \textit{sections} of $\mathcal{F}$ over $U$, and the sections of $\mathcal{F}$ over $X$ are called the \textit{global sections} of $\mathcal{F}$.
        \item For each inclusion of open sets $U \subseteq V$ with $U, V \in \tau$, a function $\mathcal{F}_{U, V}\colon \mathcal{F}(U) \rightarrow \mathcal{F}(V)$; these functions are called \textit{restriction morphisms}. By analogy with restriction functions, given $s \in \mathcal{F}(U)$, its restriction $\mathcal{F}_{U, V}(s)$ is also denoted $\restr{s}{V}$.
    \end{enumerate}

    In turn, the restriction morphisms are required to satisfy the two following properties:
    \begin{enumerate}
        \item For every open set $U \in \tau$, the restriction morphism $\mathcal{F}_{U, U}\colon \mathcal{F}(U) \rightarrow \mathcal{F}(U)$ is the identity morphism on $ \mathcal{F}(U)$.
        \item Given three open sets $U, V, W \in \tau$ such that $W \subseteq V \subseteq U$, then $\mathcal{F}_{V, W} \circ \mathcal{F}_{U, V} =\mathcal{F}_{U, W}$.
    \end{enumerate}
    
    Informally, \textit{presheaves} are an assignment of some data to the open sets of a space $X$, as shown in Figure \ref{presheaf1}. For each open set $U$, we denote the data attached to it by $\mathcal{F}(U)$. Whenever $U \subseteq W$, with $W$ possibly being equal to $X$, we can follow an arrow $\mathcal{F}(W) \rightarrow \mathcal{F}(U)$ to ”restrict” the data of $W$ to a smaller region $U$. An explanatory representation is in Figure \ref{presheaf2}.
\end{definition}

\begin{figure}[ht]
\begin{subfigure}{0.48\textwidth}
    \includegraphics[width=3.20in]{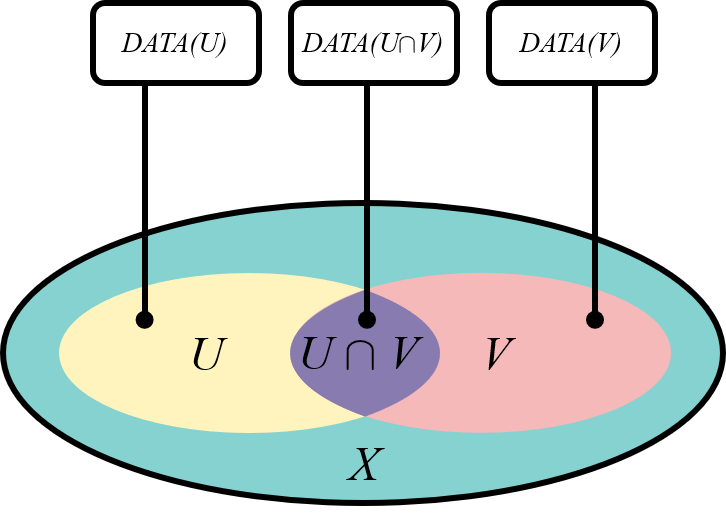}
    \caption{A presheaf is an assignment of some data to each open set of a topological space $(X, \tau)$. \newline In this case $\tau = \{X, U, V, U \cap V, \emptyset\}$.}
    \label{presheaf1}
\end{subfigure}
\hfill
\begin{subfigure}{0.49\textwidth}
    \includegraphics[width=3.20in]{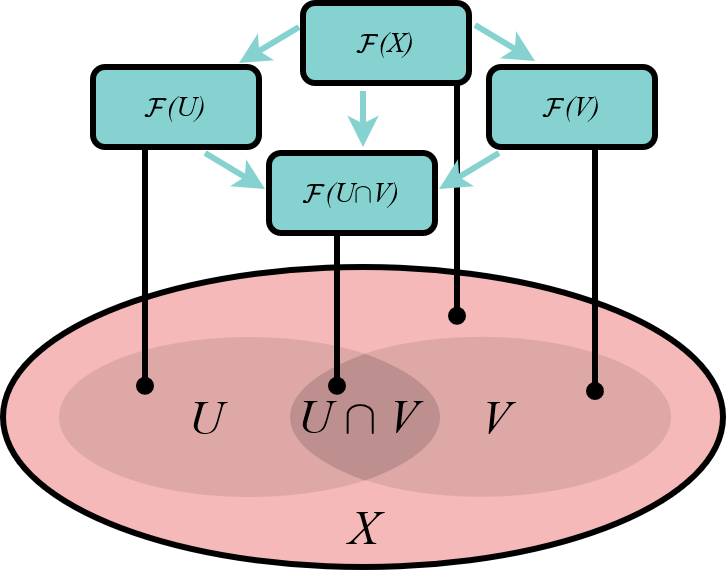}
    \caption{Given $U, W \in \tau$, whenever $U \subseteq W$, a restriction morphism $\mathcal{F}(W) \rightarrow \mathcal{F}(U)$ allows to ”restrict” the data of $W$ to a smaller region $U$.}
    \label{presheaf2}
\end{subfigure}
\caption{Visual representation of a simple presheaf example.}
\label{fig:image2}
\end{figure}

\begin{example}[Presheaf]\label{ex_presheaf}
The continuous functions over $\mathbb{R}$ constitute a presheaf with 
\begin{itemize}
    \item $X = \mathbb{R}$,
    \item $\mathcal{F}(U) = \{f : U \rightarrow \mathbb{R} | f$ is continuous$\}$,
\end{itemize} and $V \subseteq U, \mathcal{F}_{U, V} : \mathcal{F}(U) \rightarrow \mathcal{F}(V)$
being the restriction map sending $f \rightarrow \restr{f}{V}$.
\end{example}

Given a presheaf, a natural question is to what extent its sections over an open set $U$ are specified by their restrictions to smaller open sets $U_{i}$ of an open cover $\mathcal{U} = \{U_{i}\}_{i \in I}$ of $U$.

\begin{definition}[Sheaf]
    A \textit{sheaf} is a presheaf that also satisfies the \textit{locality} and \textit{glueing} conditions.
    \begin{enumerate}
        \item (\textit{Locality}) Suppose $U$ is an open set, $\mathcal{U} = \{U_{i}\}_{i \in I}$ is an open cover of $U$, and $s, t \in \mathcal{F}(U)$ are sections. If $\restr{s}{U_i} = \restr{t}{U_i}$ for all ${i \in I}$, then $s = t$.
        \item (\textit{Glueing}) Suppose $U$ is an open set, $\mathcal{U} = \{U_{i}\}_{i \in I}$ is an open cover of $U$ and $\{s_{i} \in \mathcal{F}(U_i)\}_{i \in I}$ is a family of sections.  If all pairs of sections agree on the overlap of their domains, that is, if $\restr{s_{i}}{U_i \cap U_j} = \restr{s_{j}}{U_i \cap U_j}$ for all $i, j \in I$,  then there exists a section $s \in \mathcal{F}(U)$ such that $\restr{s}{U_i} = s_i$ for all $i \in I$.
    \end{enumerate}
\end{definition}

\begin{example}[Sheaves]
\
\begin{enumerate}
    \item The presheaf consisting of continuous functions mentioned in \ref{ex_presheaf} is a sheaf. This assertion reduces to checking that, given continuous functions $f_{i}:U_{i}\to \mathbb{R}$ which agree on the intersections $U_{i}\cap U_{j}$, there is a unique continuous function $ f:U\to \mathbb{R}$ whose restriction equals $f_{i}$. 
    \item Another example is the sheaf of vector fields
    over a smooth manifold $M$, for which $\mathcal{F}(U) = \{f : U \rightarrow TU | f$ is a vector field$\}$. The restriction maps are simply restrictions
    of the vector field.
\end{enumerate}
\end{example}

The take-home idea about sheaves is that they allow creating bigger data from smaller data. An interesting fact that will be useful in the following sections is that also $\mathcal{B}-sheaves$ (and $\mathcal{B}-presheaves$) can be defined, by taking into account only basic open sets of the considered topological space. More formal definitions follow.

\begin{definition}[Base presheaf]
    Given a topological space $(X, \tau)$ and a base $\mathcal{B}$ for the topology $\tau$, a $\mathcal{B}-presheaf$ is a presheaf that consists in:
    \begin{enumerate}
        \item For each open set $U \in \mathcal{B}$, a set $\mathcal{F}(U)$.
        \item For each pair $V \subseteq U$ of members of $\mathcal{B}$, a function $\mathcal{F}_{U,V} : \mathcal{F}(U) \rightarrow \mathcal{F}(V)$ with the
usual properties.
    \end{enumerate}
\end{definition}

It is now straightforward to also define what a  $\mathcal{B}-sheaf$ is.

\begin{definition}[Base sheaf]
    A $\mathcal{B}-sheaf$ is a $\mathcal{B}$-presheaf that also satisfies the following conditions: 
    \begin{enumerate}
        \item (\textit{Locality}) Let $U \in \mathcal{B}$ and $s, t \in \mathcal{F}(U)$. If $U$ is covered by $ \{U_{i}\}_{i \in I} \subseteq \mathcal{B}$ such that $\restr{s}{U_i} = \restr{t}{U_i}$ for all ${i \in I}$, then $s = t$. 
        \item (\textit{Glueing}) Suppose $U \in \mathcal{B}$ and $U$ is covered by $\{U_{i}\}_{i \in I} \in \mathcal{B}$ with local sections $\{s_{i} \in \mathcal{F}(U_i)\}_{i \in I}$ such that for all $i, j \in I$ all sections agree on the overlap of their domains, that is, if $\restr{s_{i}}{U_i \cap U_j} = \restr{s_{j}}{U_i \cap U_j}$ for all $i, j \in I$. Then there exists a section $s \in \mathcal{F}(U)$ such that $\restr{s}{U_i} = s_i$ for all $i \in I$.
    \end{enumerate}
\end{definition}

\section{Sheaves on Graphs}
\label{sec:sheaves_on_graphs}
\subsection{Graphs}
First of all, it is necessary to properly define the notion of \textit{graph}. Also this section was in part inspired by the presentation on  \textit{Topological Deep Learning} held at the Geometric Deep Learning Summer School in Pescara (July 2022) by Cristian Bodnar.

\begin{definition}[Graph]
    A graph $G$ is defined as $G = (V, E)$, where $V$ represents the set of \textit{nodes} (also called \textit{vertices} or \textit{points}), and $E$ represents the set of edges connecting the nodes. Edges can be directed (in which case  ${ E\subseteq \{(x,y)\mid x,y\in V\}}$ or undirected (${ E\subseteq \{\{x,y\}\mid x,y \in V\}}$), and may carry weights or labels, capturing the relationships or attributes between nodes.
\end{definition}

 Graphs can be categorized into various types, including directed graphs, undirected graphs, labeled graphs, and attributed graphs, each presenting unique challenges and opportunities for analysis. A representation for a small undirected graph is in \ref{undir_graph}.

 If we suppose $|V| = n$ and we associate to each node $v$ an $f$-dimensional feature vector $\textbf{x}_v$, we can think of grouping all feature vectors in a unique $n \times f$ matrix \textbf{X}. The edge-level information can be also expressed in a compact way through the \textit{adjacency matrix} \textbf{A}.

\begin{definition}[Adjacency matrix]
     Given a graph with set of nodes described by $V = \{v_1, ..., v_n\}$, the adjacency matrix is a square $n \times n$ matrix \textbf{A} such that its element $A_{i,j}$ is 1 when there is an edge from vertex $v_i$ to vertex $v_j$, and 0 when such edge does not exist.
\end{definition}

\begin{figure}[t]
\begin{subfigure}{0.5\textwidth}
    \includegraphics[width=2.50in]{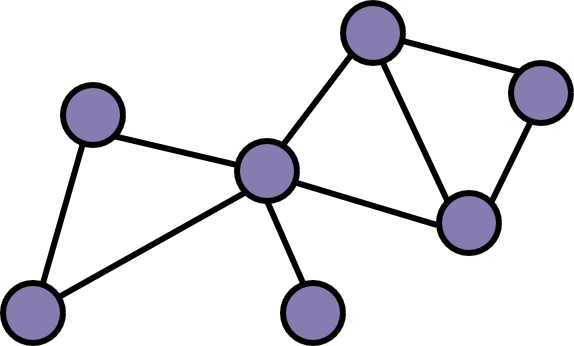}
    \caption{Undirected graph.}
    \label{undir_graph}
\end{subfigure}
\hfil
\begin{subfigure}{0.5\textwidth}
    \includegraphics[width=2.00in]{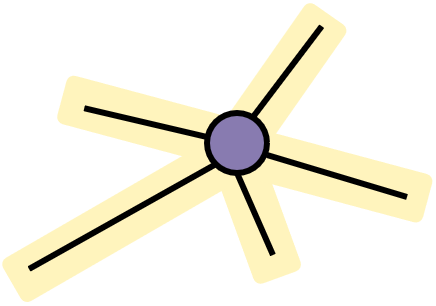}
    \caption{A \textit{star}, that is a basic open set in the star topology for graphs.}
    \label{star}
\end{subfigure}

    \centering
\begin{subfigure}{0.7\textwidth}
    \includegraphics[width=5.00in]{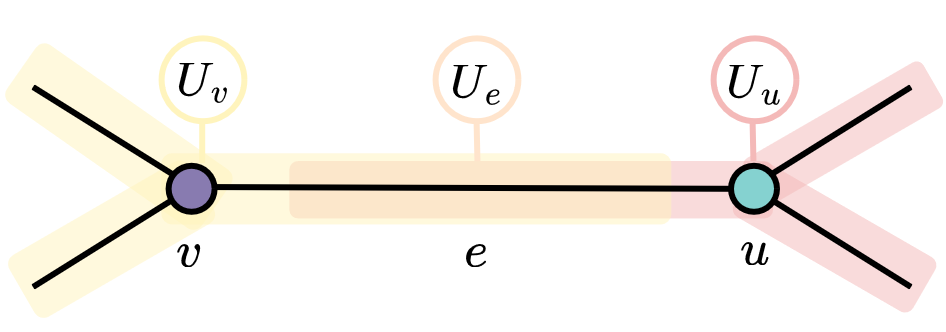}
    \caption{The \textit{stars} intersect exactly in correspondence of the edges.}
    \label{sheaf_star}
\end{subfigure}

\caption{A graph and the definition of a sheaf on it through the star topology.}
\label{fig:graphtopology}
\end{figure}

\textbf{Graphs as topological spaces}
A topology that can be straightforwardly defined on graphs is the \textit{star topology}, that is generated by the basis of all \textit{open stars} and their intersections. Given a node \textit{v}, an open star centered in \textit{v} consists in the union of such node and all its incidence edges, as highlighted in Figure \ref{star}.
As Figure \ref{sheaf_star} depicts, the basic open sets represented by stars intersect in correspondence to the edges, to which other open sets are associated.

At this point, it is easy to define a $\mathcal{B}$-presheaf on a graph $G$, by considering  $\mathcal{B}$ as the set of open stars in the graph (Figure \ref{presheaf_star}). 

\begin{figure}[ht]
    \centering
    \includegraphics[width=5.00in]{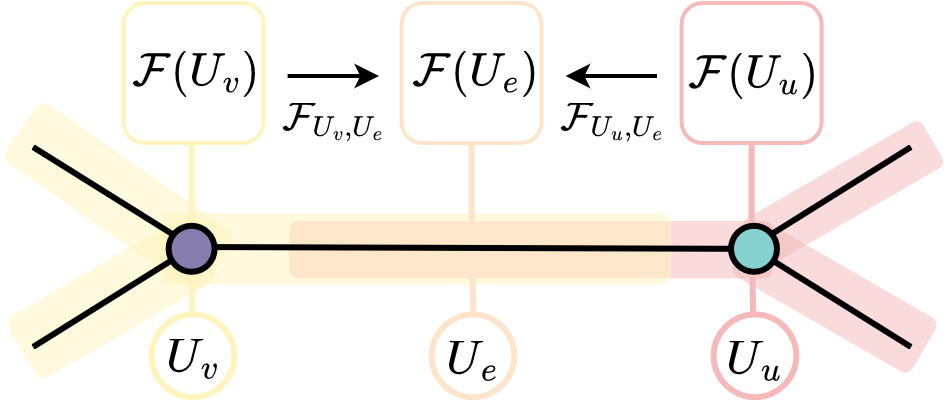}
    \caption{The star topology allows to define a $\mathcal{B}$-presheaf on graphs.}
    \label{presheaf_star}
\end{figure}

The good news is that this construction holds good properties in terms of \textit{locality} and \textit{glueing}:

\begin{theorem}
    Any $\mathcal{B}$-presheaf on a graph with the topology generated by the open stars is a $\mathcal{B}$-sheaf.
\end{theorem}
\begin{proof}
    \begin{enumerate}
        \item \textit{Locality}: There are only two types of open sets in $\mathcal{B}$, that are the $U_v$ (stars) and $U_e$ (intersection of stars) shown in \ref{presheaf_star}. For the type of $U_e$ the only cover is $\{U_e\}$ itself, thus $s = \restr{s}{U_e} = \restr{t}{U_e} = t$. For $U_v$, $U_v \in (U_i )_{i\in I}$ because the vertex $v$ cannot be covered by other open sets in $\mathcal{B}$. Then
again we have $s = \restr{s}{U_v} = \restr{t}{U_v} = t$.
        \item \textit{Glueing}: As for the case above, for open sets as $U_e$, the proof is trivial. For open sets of type $U_v$ we
        exploit again that $U_v$ = $U_k$ for some $k \in I$. Let $s_k \in \mathcal{F}(U_k) = \mathcal{F}(U_v)$. We have that
        $\restr{s_k}{U_i} = \restr{s_k}{U_k \cap U_i} = \restr{s_i}{U_k \cap U_i} = \restr{s_i}{U_v \cap U_i} = \restr{s_i}{U_i} = s_i$.
    \end{enumerate}
\end{proof}

Once given a basis $\mathcal{B}$, a natural question is how we construct a sheaf from this. Since sheaves behave similarly to linear operators, it is basically sufficient to specify how they behave on a basis to fully describe their behavior. The following result is stated without proof:

\begin{theorem}
    A $\mathcal{B}$-sheaf $\mathcal{F}$ on a space $X$ uniquely induces a sheaf $\mathcal{F}^{+}$ on $X$ such that $\mathcal{F}$ and $\mathcal{F}^{+}$ are
canonically isomorphic.
\end{theorem}

Hopefully, this section provided an idea for the formal definition of general type of sheaves on graphs. In practice, it makes sense to choose sets $\{\mathcal{F}(U)\}_{U}$ with a meaningful structure: one example could be vector spaces, or even richer Hilbert spaces. 

\subsection{Cellular Sheaves}
\label{sec:cellular_sheaves}

\begin{definition}[Cellular sheaf]\label{cellular_sheaf}
    Let $G = (V, E)$ be an undirected graph. A \textit{cellular sheaf} \cite{curry2014sheaves,shepard1985cellular} $(G, \mathcal{F})$ of vector spaces is composed by:
    
\begin{enumerate}
    \item an assignment of a vector space $\mathcal{F}(v)$ for each $v \in V$,
    \item an assignment of a vector space $\mathcal{F}(e)$ for each $e \in E$,
    \item a linear map $\mathcal{F}_{v \unlhd e}$ : $\mathcal{F}(v)\to \mathcal{F}(e)$ whenever $v$ is adjacent to the edge $e$. 
\end{enumerate}
\end{definition}

The vector spaces $\mathcal{F}(v)$ and $\mathcal{F}(e)$ are referred to as $stalks$, while the linear maps $\mathcal{F}_{v \unlhd e}$ are the \textit{restriction maps}. For the use case that is explored within this project, elements of a vertex stalk $\mathcal{F}(v)$ correspond to node-wise feature
vectors $\textbf{x}_v$, while the edge stalks $\mathcal{F}(e)$ only serve as auxiliary spaces for mixing node features.

Given a sheaf $(G,\mathcal{F})$, one can define the space
of 0-cochains $C^{0}(G,\mathcal{F})$ as the direct sum over the vertex
stalks $C^{0}(G,\mathcal{F}):=\varoplus_{v \in V}\mathcal{F}(v)$. The space of
1-cochains $C^{1}(G,\mathcal{F})$ is instead the direct sum over the edge stalks $C^{1}(G,\mathcal{F}) := \varoplus_{e \in E}\mathcal{F}(e)$. These can be thought of gathering all the stalks into a vector space; the space 0-cochains roughly consists of all possible collections of feature vectors $\textbf{x} = (\textbf{x}_v)_{v\in V}$. An intuition is given in Figure \ref{cochain}.

Hansen and Ghrist \cite{hansen2021opinion} have built a convenient mental model for these objects based on opinion dynamics, that will be extensively described in section \ref{opinion_dynamics}. In this
setting, $\textbf{x}_v$ is node $v$'s private opinion, while $\mathcal{F}_{v \unlhd e}$ expresses how that opinion manifests
publicly in a \textit{discourse space} constituted by $\mathcal{F}_{e}$. A particularly meaningful subspace of $C^{0}(G,\mathcal{F})$  is
the space of \textit{global sections}.

\begin{definition}
    Given the space of 0-cochains $C^{0}(G,\mathcal{F})$, its  \textit{global sections} correspond to  a subspace defined as $H^{0}(G,\mathcal{F}) := \{\textbf{x} \in C^{0}(G,\mathcal{F}) : \mathcal{F}_{v \unlhd e}\textbf{x}_v = \mathcal{F}_{u \unlhd e}\textbf{x}_u, \ \forall \ e = (v, u) \in E\}$, that contains exactly the
    private opinions $\textbf{x} \in C^{0}(G,\mathcal{F})$  for which all neighboring nodes agree with each other in the discourse space.
\end{definition}

\begin{figure}[ht]
    \centering
    \includegraphics[width=5.00in]{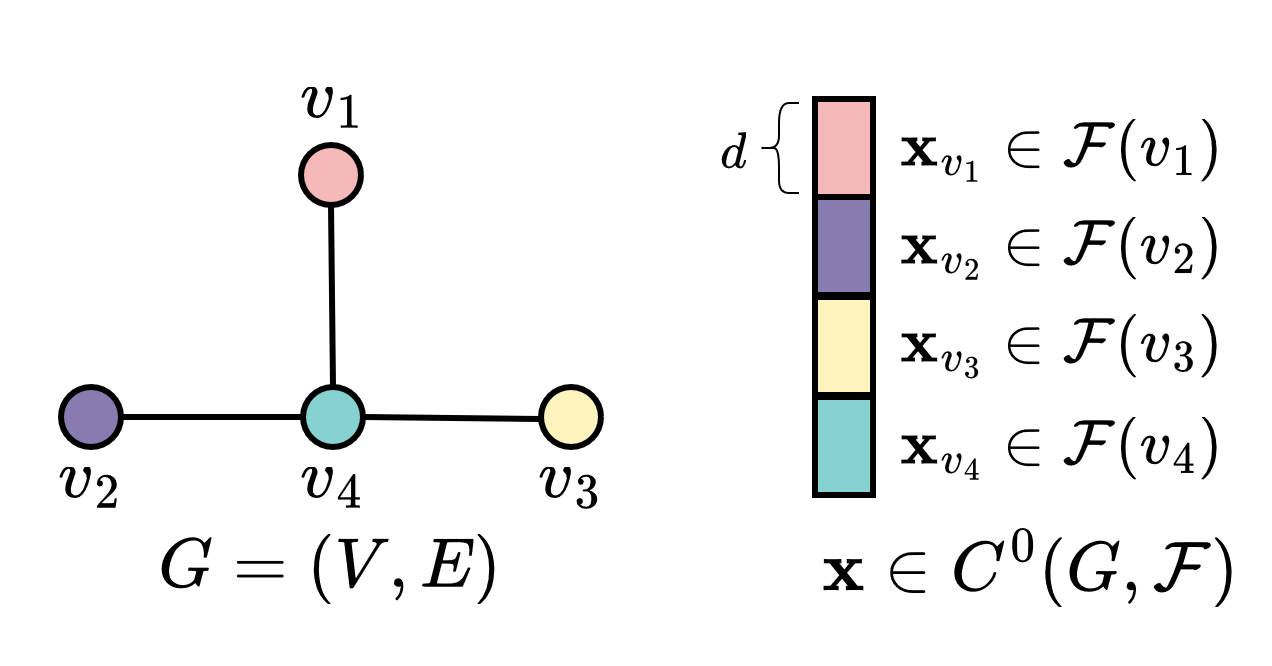}
    \caption{A graph $G =(V,E)$ and a representation of an element $\textbf{x}$ sampled from its corresponding space of 0-cochains $C^{0}(G,\mathcal{F})$.}
    \label{cochain}
\end{figure}

\begin{definition}[Coboundary]
    Once defined a specific orientation for each edge $e = v \rightarrow u \in E$ of the graph $G$, the \textit{linear coboundary  map} $\delta: C^{0}(G,\mathcal{F}) \rightarrow C^{1}(G,\mathcal{F}) $ is defined as $\delta(\textbf{x})_e := F_{u \unlhd e}\textbf{x}_{u} - F_{v \unlhd e}\textbf{x}_{v}$ (Figure \ref{fig:coboundary}, \ref{fig:coboundary_def}).
\end{definition}

\begin{figure}[ht]
    \centering
    \includegraphics[width=5.00in]{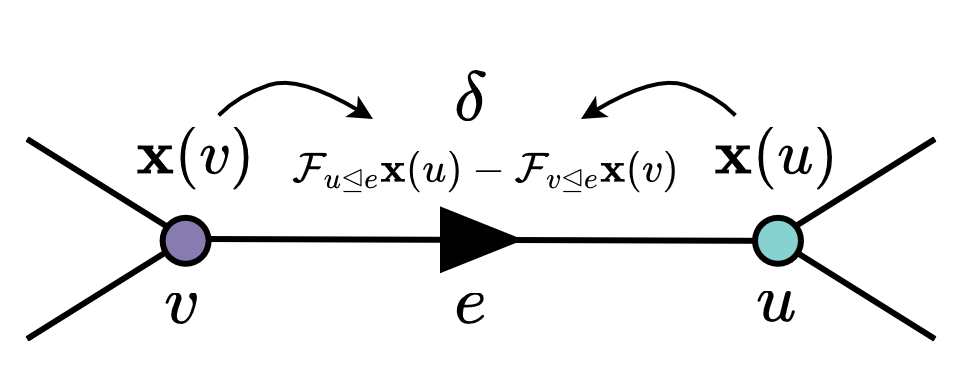}
    \caption{The sheaf coboundary operator acts on the space of 0-cochain by "projecting" the information of the node stalks into the edge stalks.}
    \label{fig:coboundary}
\end{figure}

\begin{figure}[ht]
    \centering\includegraphics[width=6.80in]{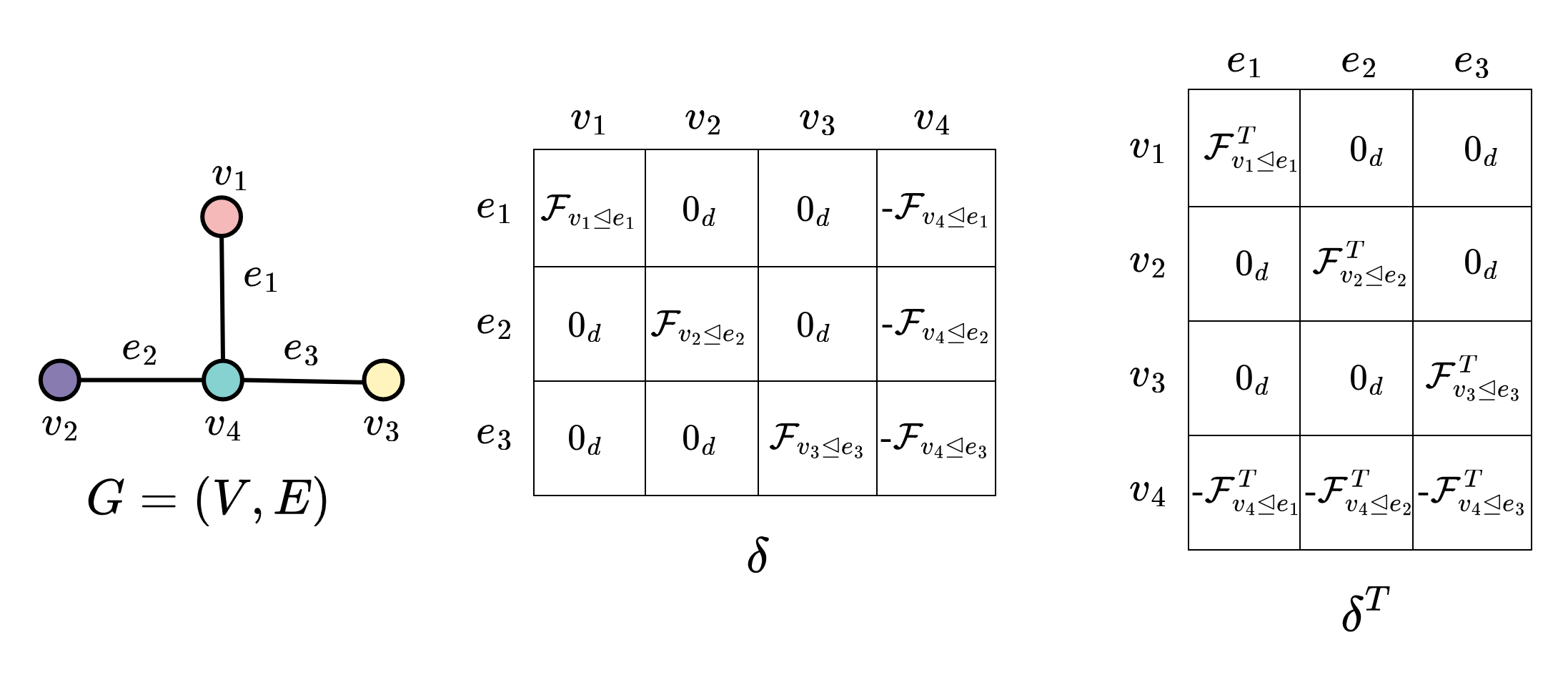}
    \caption{Given a graph $G = (V, E)$ with a sheaf structure on it as in Definition \ref{cellular_sheaf} and setting the dimensionality of all node stalks equal to $d$, the coboundary operator $\delta$ and its transpose $\delta^T$ have exactly this structure, in which $\textbf{0}_d$ stands for the $d$-dimensional null square matrix.}
    \label{fig:coboundary_def}
\end{figure}

\begin{definition}[Sheaf Laplacian]\label{laplac}
    The linear \textit{sheaf Laplacian}\cite{hansen2019toward} is a linear map $L_{\mathcal{F}}: C^{0}(G,\mathcal{F}) \to C^{0}(G,\mathcal{F})$ defined as $L_{\mathcal{F}} := \delta^{T} \circ \delta$, that acts node-wise in the following fashion:  
    \begin{equation}\label{eq:linear_laplacian}
        L_{\mathcal{F}}(\textbf{x})_{v} = \sum_{u, v \unlhd e} \mathcal{F}_{v \unlhd e}^{T}(\mathcal{F}_{v \unlhd e}\textbf{x}_{v} - \mathcal{F}_{u \unlhd e}\textbf{x}_{u})
    \end{equation}
    The (linear) sheaf Laplacian is a positive semidefinite block matrix, of which the diagonal blocks are $L_{\mathcal{F}_{v,v}} = \sum_{v \unlhd e} \mathcal{F}_{v \unlhd e}^{T}\mathcal{F}_{v \unlhd e}$ while the off-diagonal blocks are $L_{\mathcal{F}_{v,u}} = - \mathcal{F}_{v \unlhd e}^{T}\mathcal{F}_{u \unlhd e}$.
\end{definition}

\begin{remark}
    In what follows, when talking about the sheaf Laplacian without specifying whether it's linear or not, we will be referring to the linear one, which is commonly studied and used in spectral graph theory and among the Graph Neural Networks community.
\end{remark}

\begin{definition}[Normalized sheaf Laplacian]
Given a sheaf Laplacian as in  \ref{laplac}, the corresponding \textit{normalized sheaf Laplacian} $\Delta_{\mathcal{F}}$ is defined as $\Delta_{\mathcal{F}} = D^{-\frac{1}{2}}L_{\mathcal{F}}D^{-\frac{1}{2}}$ where $D$ is the block-diagonal of $L_{\mathcal{F}}$.
\end{definition}

For simplicity, we set the dimension of all node and edge stalks to \textit{d}: each restriction map will have dimensionality $d \times d$, and for the sheaf Laplacian matrix it will be $nd \times nd$. With this assumption, an intuition of its structure is provided in Figure \ref{laplacian_def}.

The sheaf Laplacian can be visualized as a  generalization of the well-known graph Laplacian on $G$: if we define a trivial sheaf where each stalk is isomorphic
to $\mathbb{R}$ ($d = 1$) and the restriction maps are the identity map over $\mathbb{R}$, we recover
the standard and well known $n \times n$ graph Laplacian from the sheaf
Laplacian.

\begin{figure}[h]
    \centering
    \includegraphics[width=3.80in]{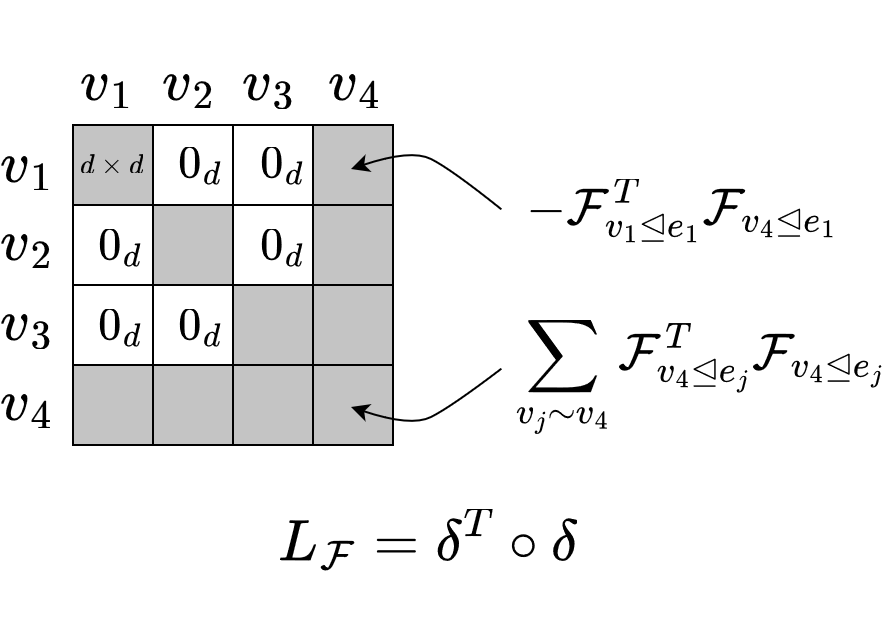}
    \caption{Given a graph $G = (V, E)$ with a sheaf structure and coboundary operators $\delta$ on it as they were defined in Figure \ref{fig:coboundary_def}, the sheaf Laplacian $L_{\mathcal{F}}$ has exactly this form.}
    \label{laplacian_def}
\end{figure}

A signal \textbf{x} is said to be \textit{harmonic} if it expresses a perfect \textit{agreement of opinions} with respect to the structure of the sheaf, that is $L_{\mathcal{F}}\textbf{x} = 0$. The central theorem of Hodge theory \cite{hansen2021opinion} formalizes this intuition, and it proves that harmonic signals and global sections of the sheaf coincide:

\begin{theorem}
    The vector spaces of harmonic signals $\text{ker}(L_{\mathcal{F}})$ and global sections $H^{0}(G,\mathcal{F})$ of a sheaf $\mathcal{F}$ are isomorphic.
\end{theorem}

An interesting geometric interpretation for sheaves, as well as a meaningful role in the study that follows, is given by sheaves with orthogonal maps, such that $\mathcal{F}_{v \unlhd e} \in O(d)$ (the Lie group of $d \times d$ orthogonal matrices).  These sheaves provide a geometric interpretation and serve as a discrete analogy to vector bundles in differential geometry \cite{loring2011introduction}. Discrete $O(d)$-bundles describe the attachment of vector spaces to points in a graph, similar to how vector bundles describe attachment to points in a manifold. The sheaf Laplacian on these bundles, referred to as the \textit{connection Laplacian} \cite{singer2012vector}, describes how elements of a vector space are transported through rotations in neighboring vector spaces. This analogy establishes a connection to parallel transport of tangent vectors across a manifold, as it is shown in Figure \ref{connection}.

\begin{figure}[h]
    \centering
    \includegraphics[width=3.80in]{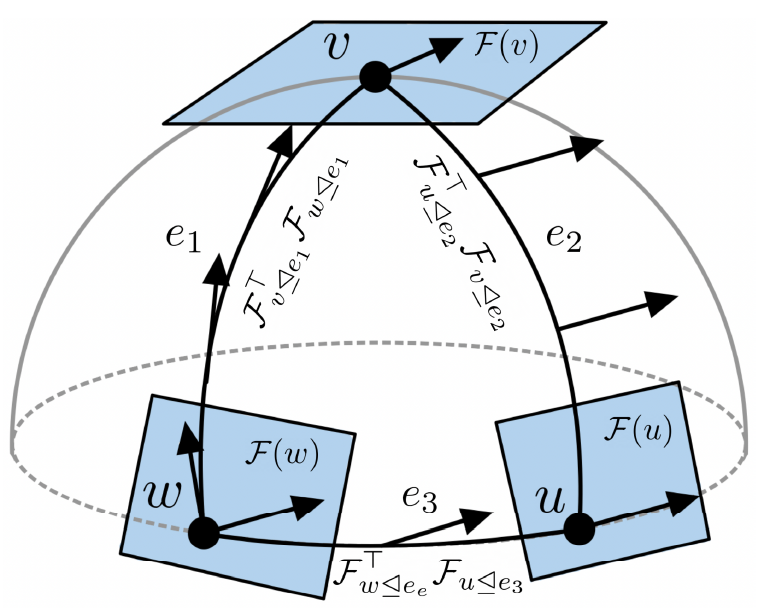}
    \caption{Analogy between parallel transport on a sphere and transport on a discrete vector bundle: a tangent vector is moved from $\mathcal{F}(w)$ to $ \mathcal{F}(v)$ to $ \mathcal{F}(u)$ and then back to $\mathcal{F}(w)$ (image from \cite{bodnar2022neural}).}
    \label{connection}
\end{figure}

\newpage

\chapter{Background}
\label{cha:snn}

\section{Graph Neural Networks}\label{sec:sec_gnn}

Graph Neural Networks have emerged as a powerful tool for analyzing and learning from data represented in the form of graphs \cite{gori2005new,scarselli2008graph,kipf2016semi}. With the increasing availability of graph-structured data in various domains, such as social science \cite{monti2019fake}, molecular chemistry \cite{gilmer2017neural}, recommendation systems \cite{monti2017geometric, ying2018graph}, and knowledge graphs \cite{schlichtkrull2018modeling,chami2020low}, GNNs have gained significant attention due to their ability to capture complex dependencies and relationships within graph data. 
This chapter aims at providing an overview of the fundamental concepts and techniques related to GNNs, including their historical context, basic components, and key advancements.

\subsection{Fundamental Concepts}

In what follows, some important notions for the GNN framework are introduced.

\paragraph{Neural Networks} Neural networks (and deep learning models in general) have revolutionized various domains, ranging from computer vision to natural language processing. These models consist of interconnected layers of artificial neurons, also known as perceptrons \cite{rosenblatt1958perceptron}, which perform weighted computations and nonlinear transformations on input data. Neural networks excel at learning hierarchical representations, extracting meaningful features and making accurate predictions.

\paragraph{Basics of Graph Neural Networks}
GNNs extend the neural network paradigm to graph-structured data. They aim to generalize deep learning techniques to effectively capture and exploit the rich structural information present in graphs. If we define a graph as a tuple $G = (V, E)$, $V$ being its set of nodes and $E$ being its set of edges, and we suppose $|V| = n$ and we associate to each node $v$ an $f$-dimensional feature vector $\textbf{x}_v$, we can think of grouping all feature vectors in an $n \times f$ matrix \textbf{X}. The edge-level information can be also expressed in a compact way through the adjacency matrix \textbf{A}. Each GNN layer (there may be more than one in multi-layer GNNs) processes these matrices to produce, for each node, a new set of updated feature vectors: 

\[
 \textbf{X}^{(t)} = f(\textbf{X}^{(t-1)}, \textbf{A})   \label{eq:general_gnn}
\]
where $t$ stands for the layer index, and the first layer takes as input $\textbf{X}^{(0)} = \textbf{X}$ (the matrix stacking all input features).

This information-propagating process is often carried out through graph convolutions \cite{kipf2016semi}, which adapt convolutional operations from image analysis to graphs.

\paragraph{Components of Graph Neural Networks}
Several key components contribute to the functioning of GNNs. These components include:

\begin{enumerate}
    \item \textit{Node Representations:} Each node in a graph is associated with a representation that captures its features and attributes. Node representations can be initialized randomly \cite{abboud2020surprising} or through pre-training techniques. GNNs update these representations by aggregating information from neighboring nodes and their own attributes, as intuitively shown in Figure \ref{fig:message_passing}.
    \item \textit{Message Passing: } Message passing is a fundamental operation in GNNs that has been formalized in \cite{gilmer2017neural}. It involves passing information from one node to its neighbors and updating their representations accordingly. This process enables the diffusion of information across the graph and enables nodes to incorporate knowledge from their local surroundings.

    The two general update functions performed in this framework are:

    \begin{equation}
        \textbf{m}_{v}^{(t)} := \text{AGGREGATE}(\{\textbf{x}_{u}^{(t-1)} | u \in \mathcal{N}(v)\}) \; \; \;
        \textbf{x}_{v}^{(t)} := \text{COMBINE}(\textbf{x}_{v}^{(t-1)}, \textbf{m}_{v}^{(t)}).
    \end{equation} 

    \begin{figure}[ht]
        \centering
        \includegraphics[width=3.50in]{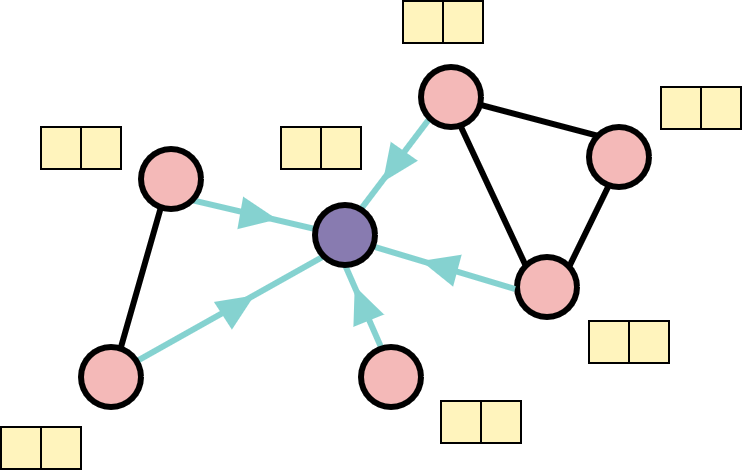}
        \caption{GNNs update node representations through the information from neighboring nodes and their own attributes, that comes in the form of "messages" and goes under some form of aggregation function.}
        \label{fig:message_passing}
    \end{figure}

    \item  \textit{Aggregation Functions: } Aggregation functions define how messages from neighboring nodes are combined to form a summary representation. Popular aggregation functions include summation, mean, max pooling, and attention-based and gating mechanisms \cite{velivckovic2017graph, li2015gated, bresson2017residual}. The choice of the aggregation function influences the expressiveness and effectiveness of the GNN model.

\end{enumerate}

\paragraph{Graph Neural Networks tasks}
Some of the most popular tasks in Graph Neural Networks include node classification, graph classification, and link prediction.

Node classification (Figure \ref{nodeclass}) involves predicting the labels or attributes of individual nodes in a graph. GNNs learn to propagate information across the graph, aggregating neighborhood information to make predictions about each node. 

Graph classification (Figure \ref{graphclass}), on the other hand, aims to classify entire graphs based on their structural properties. The input is a set of graphs, and the goal is to assign a label or class to each graph. GNNs learn to extract relevant features from the graph structure and capture global dependencies to make accurate predictions. 

\begin{figure}[t]
\begin{subfigure}{0.48\textwidth}
    \includegraphics[width=3.2in]{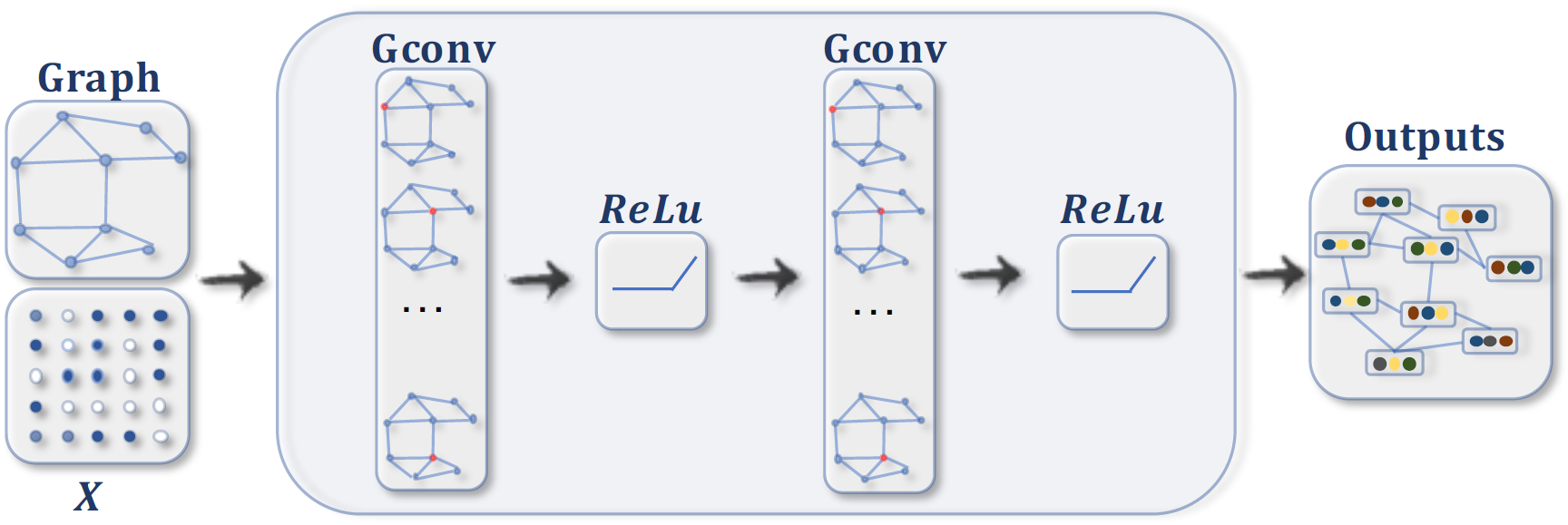}
    \caption{A ConvGNN with multiple convolutional layers that may be used for node classification tasks by adding an appropriate output layer on top of each obtained node embedding.}
    \label{nodeclass}
\end{subfigure}
\hfill
\begin{subfigure}{0.49\textwidth}
    \includegraphics[width=3.2in]{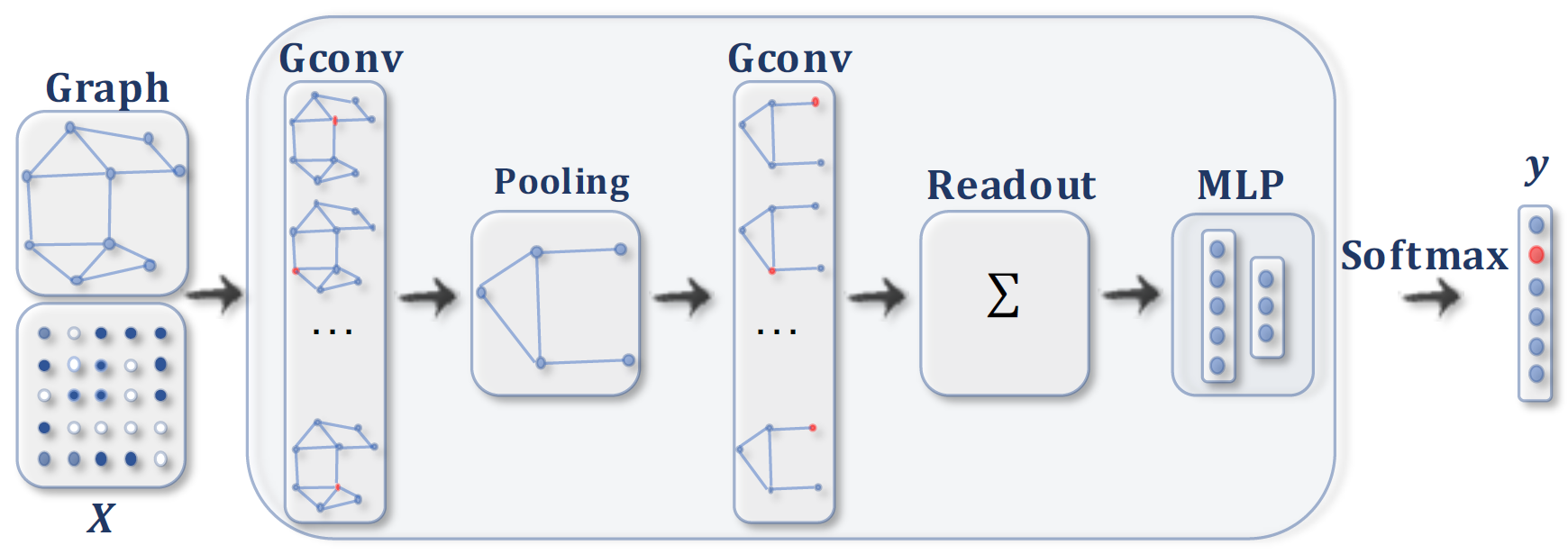}
    \caption{A ConvGNN with pooling and readout layers for graph classification. Pooling layers coarsen a graph into sub-graphs and progressively reduce the number of nodes. Readout layers, instead, pools all nodes into a single representation.}
    \label{graphclass}
\end{subfigure}
\caption{Examples of two different ConvGNN (Convolutional Graph Neural Network) architectures (images from \cite{wu2020comprehensive}).}
\label{fig:graphtasks}
\end{figure}

Link prediction focuses on predicting missing or future connections between nodes in a graph. GNNs learn the underlying patterns and relationships in the graph to infer the likelihood of a connection between two nodes that are not directly linked.

\subsection{Historical Context}
The roots of graph theory, the mathematical foundation of graph-based models, can be traced back to the 18th century with the pioneering work of Leonhard Euler. Euler's work on the Seven Bridges of Königsberg problem laid the foundation for analyzing the structural properties of networks and paved the way for the development of GNNs.

However, it was not until recently that GNNs gained substantial interest in the machine learning community. One of the earliest influential works in this area is the graph Laplacian framework proposed by Belkin and Niyogi in 2001 \cite{belkin2001laplacian}. They introduced the concept of spectral graph theory and demonstrated how Laplacian eigenvectors can be used for dimensionality reduction and semi-supervised learning on graphs.

In 2005, Scarselli et al. proposed a seminal work called Graph Neural Networks \cite{gori2005new}. They introduced a general framework for neural networks on graphs, where each node in the graph has an associated state vector, and information is propagated through the network via a recursive update rule. This work laid the foundation for the development of modern GNN architectures. A representation of a generic GNN layer is in Figure \ref{fig:generic_gnn}.

\begin{figure}[h]
    \centering
    \includegraphics[width=3.50in]{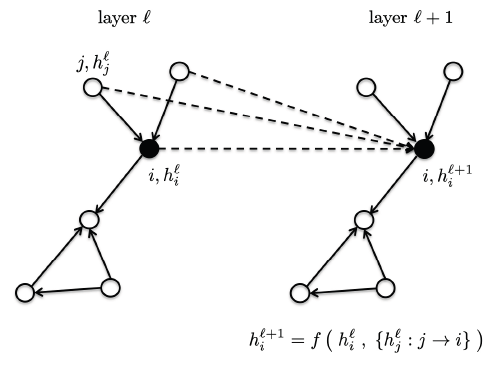}
    \caption{A generic Graph Neural Network Layer (image from \cite{dwivedi2020benchmarking}).}
    \label{fig:generic_gnn}
\end{figure}

\subsection{Key Advancements in GNNs}

Since the introduction of GNNs, several significant advancements have propelled the field forward. Some notable contributions include the ones that follow. The expressions for some of the layer-wise equations were taken from \cite{dwivedi2020benchmarking}.

\hfill

\paragraph{Graph Convolutional Networks (GCNs)} In 2016, Kipf and Welling proposed Graph Convolutional Networks (GCNs) \cite{kipf2016semi}, which revolutionized GNN research. They introduced a simplified version of spectral graph convolutions, leveraging localized first-order approximations of spectral filters. 
Graph convolutional layers \cite{kipf2016semi} enable GNNs to capture local graph structures and patterns. Inspired by convolutional operations in image analysis, graph convolutions learn filters that extract features by aggregating and transforming information from a node's neighbors (Figure \ref{fig:convolution}). 

\begin{figure}[h]
    \centering
    \includegraphics[width=4.00in]{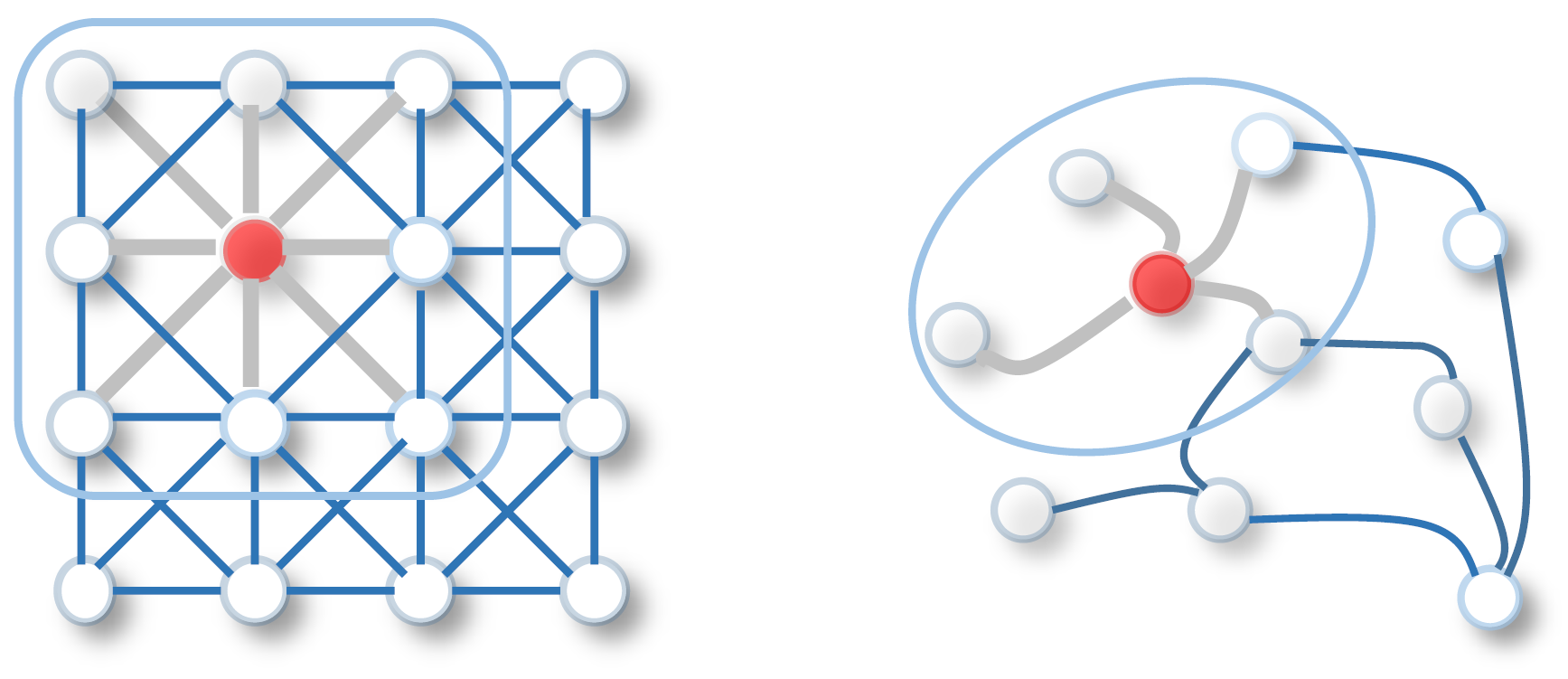}
    \caption{2D Convolution vs. Graph Convolution (images from \cite{wu2020comprehensive}).}
    \label{fig:convolution}
\end{figure}

GCNs provided a scalable and efficient approach for learning node representations by aggregating information from a node's immediate neighbors. This work significantly contributed to the popularity and practicality of GNNs, and a simple representation is in Figure \ref{fig:gcn}. 
One layer of the GCN proposed by Kipf and Welling \cite{kipf2016semi} adapts \eqref{eq:general_gnn} as:
\[
 \textbf{X}^{(t)} = \sigma(\textbf{\^D}^{-\frac{1}{2}}\textbf{\^A} \textbf{\^D}^{-\frac{1}{2}}\textbf{X}^{(t-1)}\textbf{W}^{(t)}). \label{eq:2}
\]
In this equation, $\sigma$ is a non-linear activation function, $\textbf{\^A} = \textbf{A} + \textbf{I}$, \textbf{\^D} is the diagonal node degree matrix of \textbf{\^A}, and $\textbf{W}^{(l)}$ is the $l$-th layer weight matrix, learnt from data through back-propagation. Due to the presence of the adjacency matrix, this kind of update process is local: the update of a node's feature vector at each step depends only on its neighbors. The update equation for a single node $i$ in layer $l$, setting $\sigma = \text{ReLU}$ as they do in the paper, can be expressed in the following way:

\begin{equation}
    \textbf{x}_{i}^{(t)} = \text{ReLU}(\textbf{W}^{(t)}\frac{1}{\sqrt{\text{deg}_{i}}\sqrt{\text{deg}_{j}}}\sum_{j \in \mathcal{N}_i}{\textbf{x}_{j}^{(t)}})
\end{equation}

\begin{figure}[t]
    \centering
    \includegraphics[width=4.00in]{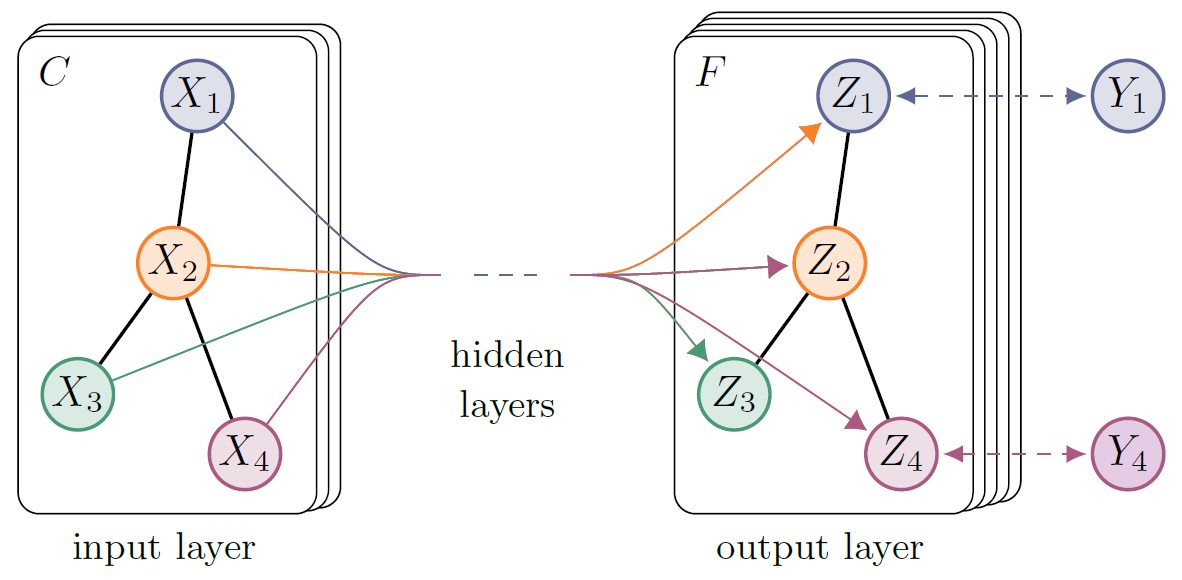}
    \caption{Graph Convolutional Network (image from \cite{kipf2016semi}).}
    \label{fig:gcn}
\end{figure}

\hfill

\paragraph{GraphSAGE}
Hamilton et al. introduced GraphSAGE (Graph Sample and Aggregated) in 2017 \cite{hamilton2017inductive}. GraphSAGE addressed the limitation of GCNs, which require the entire graph structure to be present during training. It proposed a scalable inductive learning framework that can generalize to unseen nodes by sampling and aggregating features from a node's local neighborhood. GraphSAGE achieved state-of-the-art performance on various graph-related tasks and opened doors for applying GNNs to large-scale graphs.
GraphSAGE explicitly incorporates each node's features from the previous layer in the update, in a different way with respect to its the neighborhood features:

\begin{equation}
    \hat{\textbf{x}}_{i}^{(t)} = \text{ReLU}(\textbf{W}^{(t)} \text{Concat} (\textbf{x}_{i}^{(t)},\text{Mean}_{j \in \mathcal{N}_{i}h_{j}^{(t)}})), \; \; \;
    \textbf{x}_{i}^{(t+1)} = \frac{\hat{\textbf{x}}_{i}^{(t+1)}}{\|\hat{\textbf{x}}_{i}^{(t+1)}\|_{2}}
\end{equation}

\hfill

\paragraph{Graph Attention Networks (GAT)}
GAT \cite{velivckovic2017graph}, proposed by Veličković et al. in 2018, introduced an attention mechanism for GNNs. Inspired by the success of attention mechanisms in natural language processing, GAT allows nodes to selectively attend to their neighbors during information aggregation. It learns a mean over each node’s neighborhood
features sparsely weighted by the importance of each neighbor. By assigning attention weights to different neighbors, GAT enabled more expressive and adaptive modeling of graph data, leading to improved performance in various tasks.
The node update equation is expressed by:

\begin{equation}
    {\textbf{x}}_{i}^{(t+1)} = \text{Concat}_{k=1}^{K}(\text{ELU}(\sum_{j \in \mathcal{N}_{i}}{e_{ij}^{k, (t)}\mathbf{W}^{k, (t)}\textbf{x}_{j}^{(t)}})),
\end{equation}
where
\begin{equation}
    e_{ij}^{k, (t)} = \frac{\text{exp}(\hat{e}_{ij}^{k,(t)})}{\sum_{j' \in \mathcal{N}_{i}}\text{exp}(\hat{e}_{ij'}^{k,(t)})},
\end{equation}
\begin{equation}
    \hat{e}_{ij}^{k,(t)} = \text{LeakyReLU}(V^{k,(t)}\text{Concat}(\textbf{W}^{k,(t)}\textbf{x}_{i}^{(t)}, \textbf{W}^{k,(t)}\textbf{x}_{j}^{(t)})).
\end{equation}

\begin{figure}[h]
    \centering
    \includegraphics[width=5.50in]{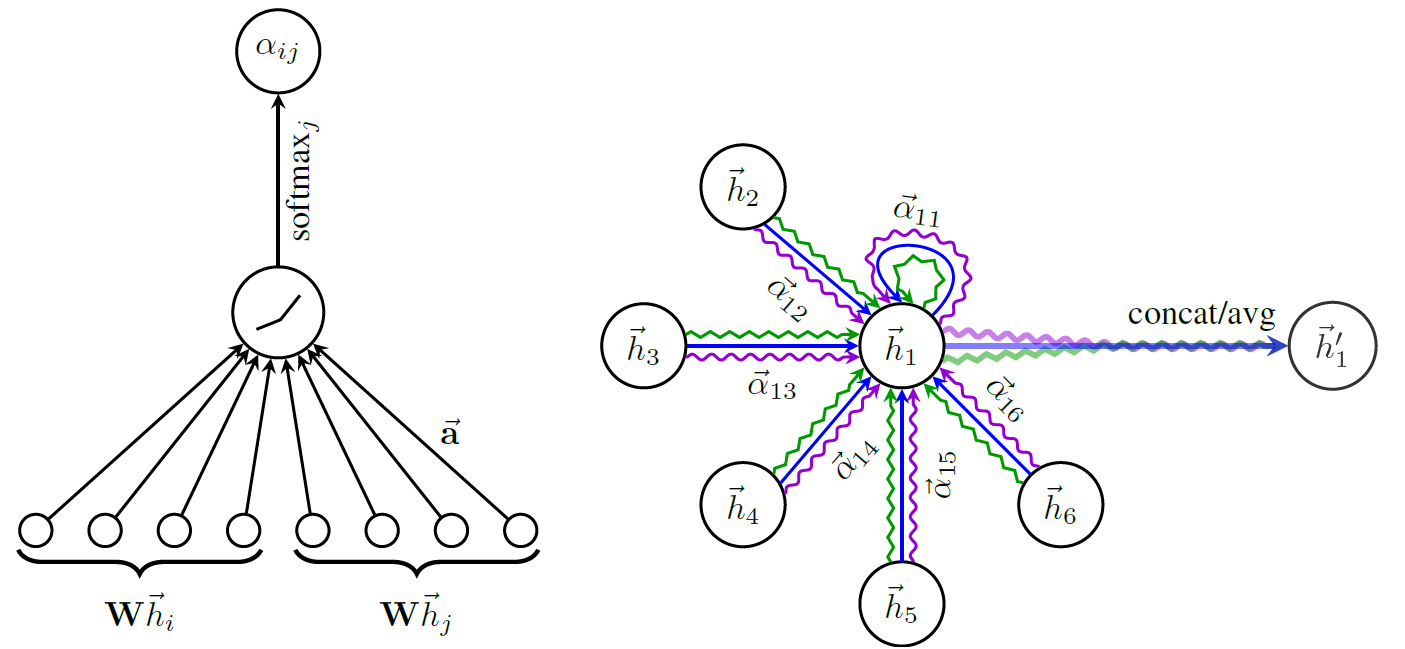}
    \caption{\textbf{Left}: The attention mechanism employed by Graph Attention Network \cite{velivckovic2017graph}. \textbf{Right}: An illustration of multihead
    attention (with 3 heads) by node 1 on its neighborhood, in which different arrow styles and
    colors imply independent attention computations. Average or concatenation is then performed to aggregate the features and obtain $\overrightarrow{h}_{1}'$
    (both images from \cite{velivckovic2017graph}).}
    \label{fig:gat}
\end{figure}

\hfill

\paragraph{Graph Isomorphism Networks (GIN)}
Xu et al. proposed Graph Isomorphism Networks (GIN) \cite{xu2018powerful} in 2019, which introduced a powerful message-passing scheme for GNNs. GIN employs a flexible aggregation function that is based on the Weisfeiler-Lehman Isomorphism Test \cite{weisfeiler1968reduction}, making it more expressive and capable of capturing structural information. GIN demonstrated superior performance on graph classification tasks and further expanded the repertoire of GNN architectures.

The node-wise update is performed as:

\begin{equation}
   \textbf{x}_{i}^{(t+1)} = \text{ReLU}(\textbf{W}_{1}^{(t)}(\text{ReLU}(\text{BN}(\textbf{W}_{2}^{(t)}\hat{\textbf{x}}_{i}^{(t+1)})))), 
\end{equation}
\begin{equation}
    \hat{\textbf{x}}_{i}^{(t+1)} = (1 + \epsilon)\textbf{x}_{i}^{(t)} + \sum_{j \in \mathcal{N}_{i}}{\textbf{x}_{j}^{(t)}},
\end{equation}
where $\text{BN}$ denotes Batch Normalization \cite{ioffe2015batch} and $\epsilon$ is a learnable constant.

\hfill

\paragraph{Gated Graph ConvNet (GatedGCN)}
Li et al. introduced Gated Graph Sequence Neural Networks (GatedGCN) \cite{li2015gated} as a powerful message-passing scheme for processing sequential graph-structured data. GatedGCNs combine the strengths of recurrent neural networks (RNNs \cite{elman1990finding}) and graph neural networks to effectively capture both temporal dependencies and structural information within graphs.

In GatedGCN, each node in the graph is associated with a hidden state vector that evolves over time. The model employs gated recurrent units (GRUs \cite{cho2014learning}) to update the hidden states based on information from neighboring nodes and previous time steps. The gating mechanism allows the network to selectively integrate and update node representations, facilitating the capturing of long-range dependencies and dynamic temporal patterns within the graph.

The node-wise update can be expressed as follows, in which $e_{ji}$ expresses the feature associated to edge $(j, i)$:

\begin{equation}
    \hat{\textbf{x}}_{i}^{(t)} = \sum_{j \in \mathcal{N}_{i}} {e_{ji}\Theta \textbf{x}_{j}^{(t-1)}},
\end{equation}
\begin{equation}
    \textbf{x}_{i}^{(t)} = \text{GRU}(\hat{\textbf{x}}_{i}^{(t)}, \textbf{x}_{i}^{(t-1)}).
\end{equation}

The initialization of the node embeddings is performed as

\begin{equation}
    \textbf{x}_{i}^{(0)} = \textbf{x}_{i} \| 0
\end{equation}

in which $\textbf{x}_{i}$ refers to the $i$-th node's input features.

\hfill

\paragraph{Residual Gated Graph ConvNet (ResGatedGCN)}
Bresson et al. introduced Residual Gated Graph ConvNets (ResGatedGCNs) \cite{bresson2017residual} as a powerful architecture for graph neural networks. ResGatedGCN incorporates residual connections, batch normalization and edge gates mechanisms into the convolutional layers, enabling the network to selectively propagate information from neighboring nodes. This gating mechanism enhances the model's ability to capture long-range dependencies and selectively fuse information from different nodes, resulting in improved expressiveness and effectiveness in capturing complex graph structures. The GatedGCN model explicitly maintains edge features at each layer, and they also go through update operations along with node features:

\begin{equation}
    \textbf{x}_{i}^{(t+1)} = \textbf{x}_{i}^{(t)} + \text{ReLU}(\text{BN}(\textbf{W}_{1}^{(t)}\textbf{x}_{i}^{(t)} + \sum_{j \in \mathcal{N}_{i}}{e_{ij}^{(t)}} \odot \textbf{W}_{2}^{(t)}\textbf{x}_{j}^{(t)})),
\end{equation}
where $\odot$ is the Hadamard product, and the edge gates $e_{ij}^{(t)}$ are defined as:
\begin{equation}\label{eq:gates}
    e_{ij}^{(t)} = \frac{\sigma(\hat{e}_{ij}^{(t)})}{\sum_{j' \in \mathcal{N}_{i}}{\sigma(\hat{e}_{ij'}^{(t)}) + \epsilon}},
\end{equation}
\begin{equation}
    \hat{e}_{ij}^{(t)} = \hat{e}_{ij}^{(t-1)} + \text{ReLU}(\text{BN}(\textbf{W}_{3}^{(t)}\textbf{x}_{i}^{(t-1)} + \textbf{W}_{4}^{(t)}\textbf{x}_{j}^{(t-1)} +
    \textbf{W}_{5}^{(t)}\hat{e}_{ij}^{(t-1)})),
\end{equation}

where $\sigma$ is the sigmoid function and $\epsilon$ is a small fixed constant for numerical stability. The edge gates that are put in place in Equation \ref{eq:gates} can be thought as a soft attention process.

\hfill




\subsection{Current Research Directions}
The field of GNNs continues to advance rapidly, with ongoing research exploring various directions. Some of the current research areas include:

\paragraph{Graph Attention Mechanisms}
Further developments in attention mechanisms for GNNs aim to enhance the model's ability to capture important features and relationships within graphs \cite{yun2019graph,velivckovic2017graph,barbero2022sheaf}. Attention mechanisms can be extended to capture long-range dependencies and enable more fine-grained interactions between nodes.

\paragraph{Incorporating Temporal Dynamics}
Many real-world graphs exhibit temporal dynamics, such as evolving social networks or dynamic molecular systems. Current research focuses on incorporating temporal information into GNNs to model and predict dynamic behavior. Temporal GNNs, such as ST-GCN \cite{yan2018spatial} and EvolveGCN \cite{pareja2020evolvegcn}, have shown promising results in capturing temporal dependencies \cite{longa2023graph}.


\paragraph{Graph Neural Networks for Graph Generation}
 In recent years, GNNs have also been applied to the task of graph generation. Notable works in this area include GraphRNN \cite{you2018graphrnn}, GraphVAE \cite{simonovsky2018graphvae}, and Junction Tree Variational Autoencoder (JT-VAE) \cite{jin2018junction}. These models leverage GNNs to learn the generative process of graphs, enabling the synthesis of new graph structures with desired properties. Graph generation has applications in drug discovery \cite{you2018graph}, molecule design \cite{fan2023generative,vignac2022digress}, and social network analysis \cite{bojchevski2018netgan}, among others.

\paragraph{Scalability and Efficiency}
As graph sizes continue to grow, there is a need for scalable and efficient GNN models. Research focuses on developing techniques to handle large-scale graphs, such as sampling-based methods \cite{chen2018fastgcn}, parallelization strategies \cite{chiang2019cluster}, and graph sparsification algorithms \cite{zeng2019graphsaint,rong2019dropedge}.

\subsection{Heterophily and Oversmoothing in GNNs}
Despite their effectiveness, GNNs suffer from several challenges that can impact their performance. Two prominent challenges are oversmoothing and heterophily, which can lead to the loss of discriminative power and poor generalization.

\paragraph{Oversmoothing} 
Oversmoothing refers to a phenomenon where GNNs tend to produce similar node representations regardless of their structural characteristics or individual properties. This occurs due to the repeated aggregation and transformation steps performed by GNNs, which can cause the diffusion of information throughout the graph. As the number of layers increases, the representations of nodes become increasingly similar, ultimately resulting in a loss of discriminative power.

The oversmoothing problem arises from the nature of message passing in GNNs. At each layer, nodes aggregate information from their neighbors and update their own representations based on the aggregated information. While this iterative process helps capture local dependencies, it can also cause the information to propagate too widely, leading to an oversmoothed representation. Oversmoothing is particularly prevalent in deep GNN architectures, where the repeated message passing exacerbates the problem. In general, the higher the number of layers the more oversmoothing affects node features for the considered graph.

\paragraph{Heterophily} Heterophily refers to the phenomenon where nodes in a graph have diverse structural or attribute properties, making it challenging for GNNs to effectively capture and model such heterogeneity. In many real-world networks, nodes may exhibit significant structural differences or have distinct attribute values, which can hinder the learning process of GNNs. This concept is strictly connected (as it is its opposite) to the \textit{homophily} level of a graph, that is generally defined as the rate of intra-class edges with respect to the whole set of edges. Intuitively, in homophilic graphs nodes tend to connect to other similar nodes, where similarity is to be intended as sharing the same node labels. 

\begin{figure}[ht]
\begin{subfigure}{0.48\textwidth}
    \includegraphics[width=3.7in]{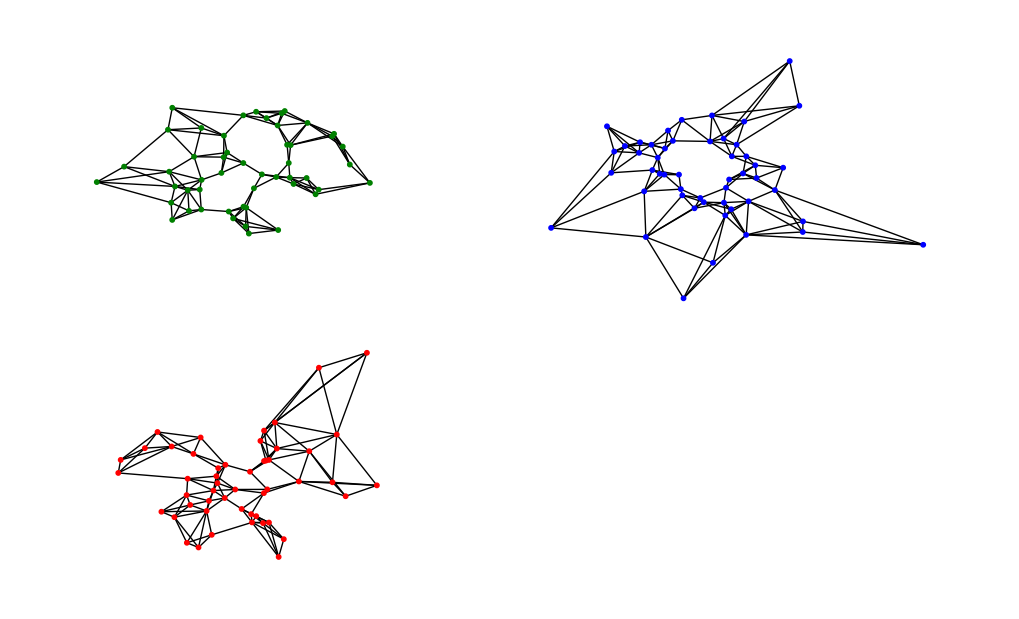}
    \caption{Fully homophilic graph, in which all edges connect nodes of the same class (homophily level 1).}
    \label{homophily}
\end{subfigure}
\hfill
\begin{subfigure}{0.49\textwidth}
    \includegraphics[width=3.7in]{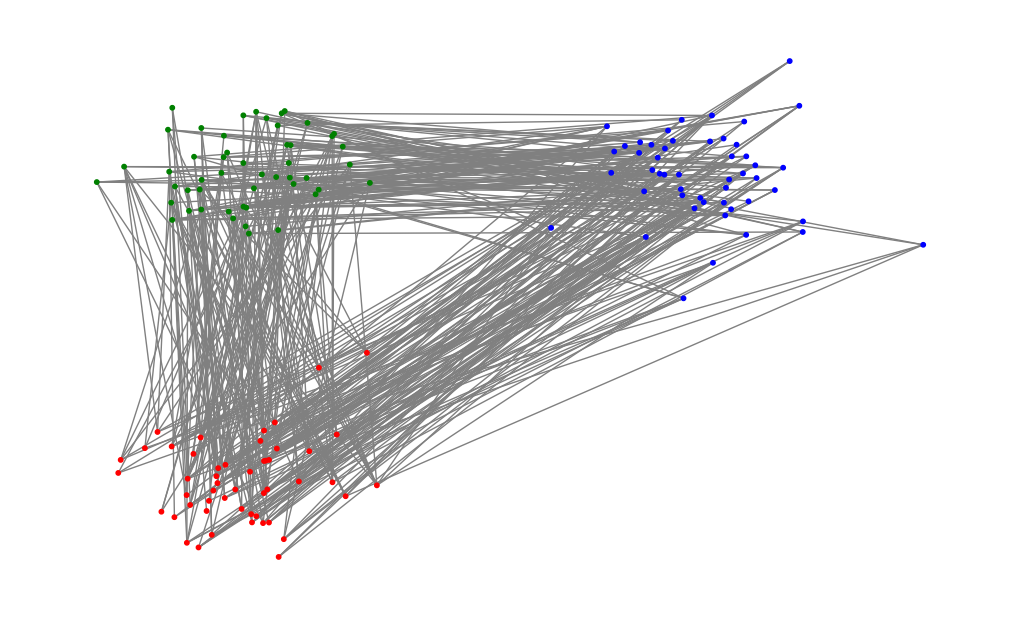}
    \caption{Fully heterophilic graph, in which all edges connect nodes belonging to different classes (homophily level 0).}
    \label{heterophily}
\end{subfigure}
\caption{Starting from the same set of nodes belonging to three classes red, green, and blue, it is possible to  generate a graph with highest or lowest possible homophily levels with just a different configuration of edges.}
\label{fig:homo_hetero}
\end{figure}

The heterophily problem arises due to the limitations of the neighborhood aggregation mechanism employed by GNNs. GNNs typically aggregate information from neighboring nodes, treating all neighbors equally without considering their heterogeneity and relying on the strong (and often wrong) assumption of homophily. This approach may result in the loss of important information or the propagation of irrelevant information, leading to suboptimal representations and predictions. Heterophily becomes more pronounced in graphs with high structural diversity or when the attributes of nodes exhibit significant variations.

\paragraph{Impact on GNN performance} Both oversmoothing and heterophily can have detrimental effects on the performance of GNNs, affecting their ability to learn meaningful representations and make accurate predictions. Oversmoothing leads to the convergence of node representations, making it difficult to distinguish between nodes with different characteristics. This can result in poor discriminative power and limit the network's ability to capture fine-grained patterns in the data.

Heterophily, on the other hand, introduces challenges in capturing the structural and attribute diversity present in the graph. GNNs may struggle to differentiate between nodes with distinct properties or fail to effectively leverage the available information. As a consequence, the predictive performance of GNNs can be compromised, especially in scenarios where heterogeneity plays a crucial role.

\paragraph{Addressing Oversmoothing and Heterophily} Addressing the problems of oversmoothing and heterophily is an active area of research in the GNN community. Various techniques have been proposed to mitigate these challenges and enhance the performance of GNNs.

To combat oversmoothing, researchers have proposed depth-aware architectures that control the diffusion of information by selectively aggregating neighbors' representations \cite{zeng2019graphsaint,rong2019dropedge,li2018deeper}. Additionally, graph coarsening techniques have been employed to reduce the number of layers in deep GNNs, preventing excessive information propagation \cite{gao2019graph,wu2019simplifying}.
To tackle heterophily, instead, a popular solution involves attention and gating mechanisms being integrated into GNNs to selectively attend to relevant neighbors based on their importance \cite{velivckovic2017graph, li2015gated, bresson2017residual, barbero2022sheaf}.

Along with well-performing designs that tackle the described problems separately, also works that theoretically connect the two have been proposed lately \cite{yan2022two,bodnar2022neural}. Their analysis is very different in terms of method and assumptions, and while Yan et al. in  \cite{yan2022two} have a probabilistic approach, Bodnar et al. in the Neural Sheaf Diffusion paper \cite{bodnar2022neural} focus on diffusion PDEs and bring into play new mathematical tools from cellular sheaf theory to analyze the problem.

\section{Sheaf Neural Networks}\label{sec:section_snn}

Sheaf Neural Networks were introduced in recent years \cite{hansen2020sheaf,bodnar2022neural,barbero2022sheaf,suk2022surfing} within the framework of Graph Neural Networks with the goal of bringing in additional useful properties with respect to standard GNNs in specific settings, first of all to address the oversmoothing issue and handle heterophilic datasets in graph classification. 
The basic idea consists in equipping the graph with a richer geometrical structure (a cellular sheaf) and taking advantage of its expressiveness in feature space, its properties in the diffusion equation and the characteristics of the convolutional models obtained by discretizing such equation.

For a formal definition of the terms used in this section, regarding cellular sheaves, their Laplacian and related concepts, please refer to Chapter 2 (Mathematical Preliminaries).
For reference, a schema of the (linear) sheaf Laplacian is reported in Figure \ref{fig:sheaf_lap}.

\begin{figure}[ht]
    \centering
    \includegraphics[width=4.50in]{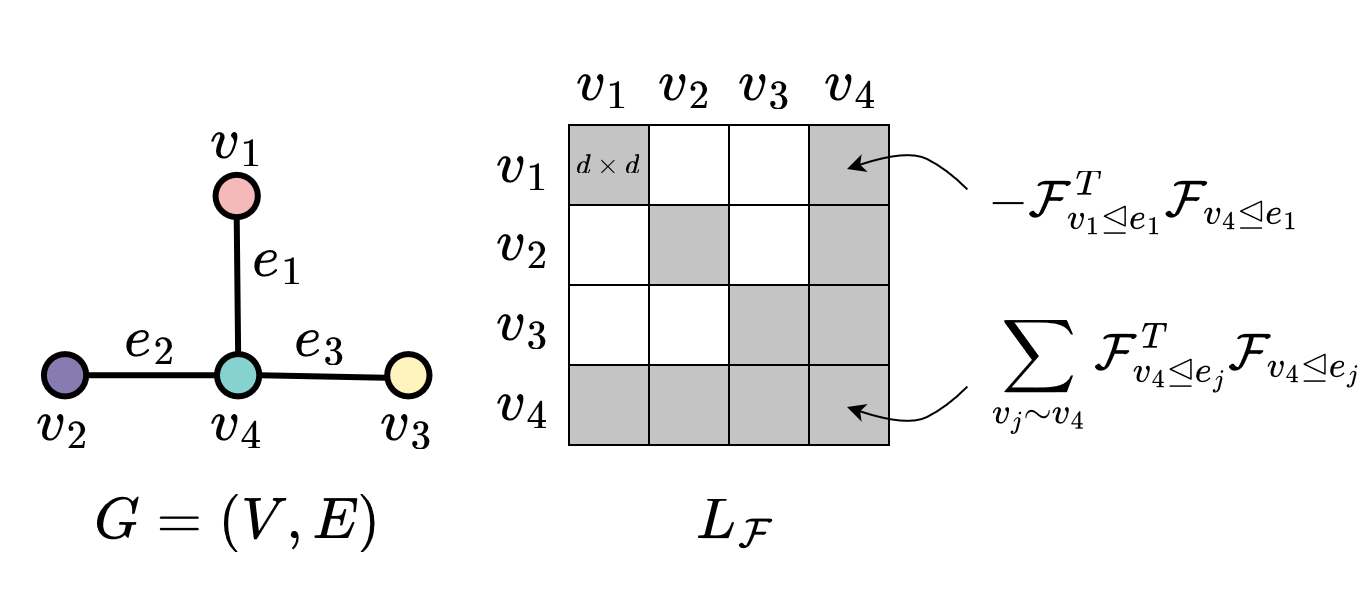}
    \caption{Given a graph $G = (V, E)$ with a sheaf structure on it, the sheaf Laplacian $L_{\mathcal{F}}$ has exactly this form.}
    \label{fig:sheaf_lap}
\end{figure}

As already stated, in the following we will always consider cellular sheaves with stalks being $d$-dimensional vector spaces on $\mathbb{R}$.

\subsection{Motivation: the Separation Power} \label{motivation}

Let $(G, \mathcal{F})$ be a cellular sheaf on an undirected graph $G = (V, E)$, with $d$-dimensional node feature vectors $\textbf{x}_v \in \mathcal{F}_v$. All the individual node features can be stacked as a single vector $\textbf{x} \in C^{0}(G, \mathcal{F})$. Allowing also for $f$ feature channels, everything is represented by a single matrix $\textbf{X} \in \mathbb{R}^{(nd) \times f}$, with columns being vectors in $C^{0}(G, \mathcal{F})$.

\textit{Sheaf diffusion} is a process on $(G, \mathcal{F})$ governed by the following partial
differential equation:
\begin{equation}
 \textbf{X}(0) = \textbf{X}, \, \, \, \,     \dot{\mathbf{X}}(t) = - \Delta_{\mathcal{F}}\textbf{X}(t). \label{eq:sheaf_diff_contin}
\end{equation}

It has been shown in \cite{hansen2019toward} that, in the time limit, each feature channel gets projected into the vector space of harmonic signals $ker(\Delta_{\mathcal{F}})$. As discussed in Chapter 2, this space is isomorphic to the space of global sections and contains the signals that agree with the restriction maps of the sheaf along all the edges. Hence, the process of sheaf diffusion can be interpreted as a \textit{synchronization} mechanism. Subsequently, we investigate the capacity of specific categories of sheaves to linearly distinguish the features in relation to the diffusion processes they generate, and how it can prevent excessive smoothing. The following results are taken from \cite{bodnar2022neural}, and their proofs and further details can be found in the appendix of such paper.

\begin{definition}[\cite{bodnar2022neural}]
A hypothesis class of sheaves with $d$-dimensional stalks $\mathcal{H}^d$ has \textit{linear separation power} over a family of graphs $\mathcal{G}$ if for any labelled graph $G = (V, E) \in \mathcal{G}$, there is a sheaf $(G, \mathcal{F}) \in \mathcal{H}^d$ that can linearly separate the classes of $G$ in the time limit of Equation \ref{eq:sheaf_diff_contin} for 
almost all initial conditions. 
\end{definition}

\begin{remark}
    The restriction to a set of initial conditions that is dense in the ambient space is necessary
because diffusion behaves in the limit like a projection in the harmonic space and there will always
be degenerate initial conditions that will yield a zero projection.
\end{remark} 

The choice of the sheaf impacts the behavior of the diffusion process, which leads to different separation capabilities in classification tasks. To show this, we now define a set of increasingly general classes of sheaves and analyze their properties:

\begin{enumerate}
    \item \textbf{Symmetric invertible}: $\mathcal{H}_{\text{sym}}^d := \{ (G, \mathcal{F}) : \mathcal{F}_{v \unlhd e} = \mathcal{F}_{u \unlhd e},\ \text{det}(\mathcal{F}_{v \unlhd e}) \neq 0 \}$.

    For
    $d$ = 1, the sheaf Laplacians induced by this class of sheaves coincide with the set of the 
    weighted graph Laplacians with strictly positive weights, hence this hypothesis class is of particular interest since it
    includes those graph Laplacians typically used by graph convolutional models such as GCN \cite{kipf2016semi}. These sheaf Laplacians can linearly separate the classes in
    binary classification settings under some homophily assumptions.
    On the contrary, under certain heterophilic conditions, this hypothesis class is not powerful enough to
    linearly separate the two classes no matter what the initial conditions are \cite{bodnar2022neural}.
    
    \item \textbf{Non-symmetric invertible}: $\mathcal{H}_{\text{non-sym}}^d := \{ (G, \mathcal{F}) :  \text{det}(\mathcal{F}_{v \unlhd e}) \neq 0 \}$.

    This larger hypothesis class
is able to address the limitations described for $\mathcal{H}_{\text{sym}}^d$ by allowing non-symmetric relations, and the intuition is that it would allow for negative transport maps $\mathcal{F}^\top_{v \unlhd e}\mathcal{F}_{u \unlhd e} = -1$ for inter-class edges and $+1$ for intra-class edges.

So far we have only studied the effects of changing the type of sheaves in the case $d=1$, for which the following holds:

    \begin{proposition}[\cite{bodnar2022neural}]\label{prop:impossible_separation}
        Let $G$ be a connected graph with $C \geq 3$ classes. Then, $\mathcal{H}^1$ cannot linearly separate the classes of $G$ for any initial conditions $\textbf{X} \in \mathbb{R}^{n \times f}$. 
    \end{proposition}

    This means that in the infinite depth setting, sufficiently high stalk dimension $d$ (which is not related to $f$) is needed in order to solve tasks involving more than two classes. 

    \item \textbf{Diagonal invertible}: $\mathcal{H}_{\text{diag}}^d := \{ (G, \mathcal{F}) :  \text{diagonal} \mathcal{F}_{v \unlhd e}, \text{det}(\mathcal{F}_{v \unlhd e}) \neq 0 \}$.

The sheaves in this class can be seen as $d$ independent sheaves from $\mathcal{H}^1$ encoded in the $d$-dimensional diagonals of their restriction maps. The following holds:

    \begin{proposition}[\cite{bodnar2022neural}]
        Let $\mathcal{G}$ be the set of connected graphs with nodes belonging to $C \geq 3$ classes. Then for $d \geq C$, $\mathcal{H}_{\text{diag}}^d$ has linear separation power over $\mathcal{G}$.  
    \end{proposition}

    \item \textbf{Orthogonal}: $\mathcal{H}_{\text{orth}}^d := \{ (G, \mathcal{F}) :  \mathcal{F}_{v \unlhd e} \in O(d)$\}.

    This is the class of $O(d)$-bundles, and these kinds of restriction maps are more complex with respect to the diagonal ones. They bring in the advantage that lower-dimensional stalks can be used to achieve linear separation in the
presence of even more classes with respect to the diagonal case:

    \begin{theorem}[\cite{bodnar2022neural}]
        Let $\mathcal{G}$ be the class of connected graphs with $C \leq 2d$ classes. Then, for all $d \in \{2, 4\}$, $\mathcal{H}_{\text{orth}}^{d}$ has linear separation power over $\mathcal{G}$. 
    \end{theorem}

    One example of a diffusion process over a $O(2)$-bundle is in Figure \ref{fig:disentanglement}.
    
\end{enumerate}

\begin{figure}[t]
    \centering
    \includegraphics[width=6.00in]{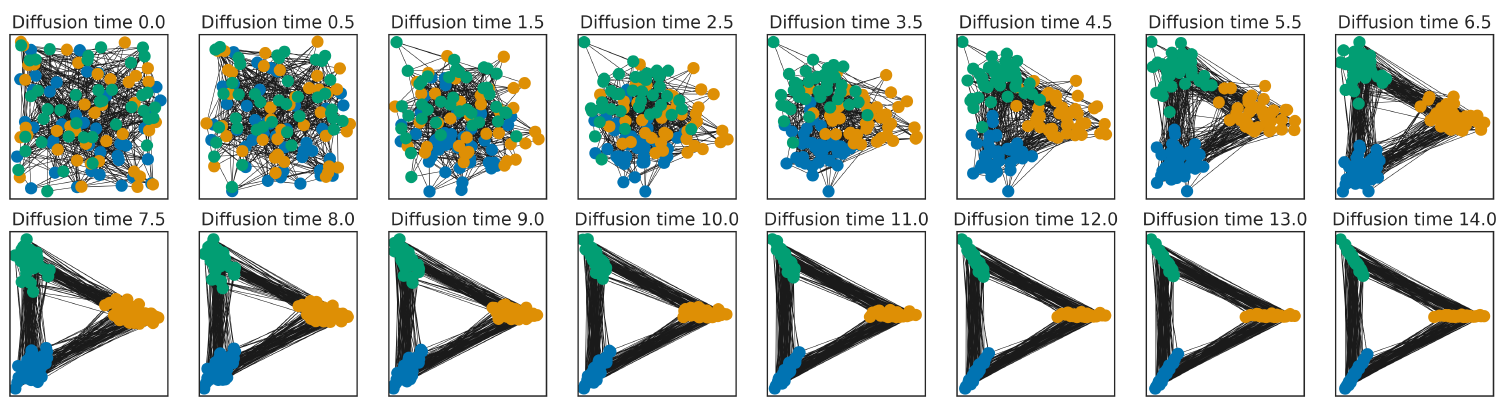}
    \caption{Diffusion process on $O(2)$-bundles progressively separates the classes of the graph (image from \cite{bodnar2022neural}).}
    \label{fig:disentanglement}
\end{figure}

In summary, these results show how different types of sheaves affect the outcome of node classification tasks.

\subsection{Sheaf Convolutions}

The continuous diffusion process expressed in Equation \ref{eq:sheaf_diff_contin} can be discretized via the explicit Euler scheme with unit
step-size:
\begin{equation}
 \textbf{X}(t + 1) = \textbf{X}(t) - \Delta_{\mathcal{F}}\textbf{X}(t) = (\textbf{I}_{nd} - \Delta_{\mathcal{F}})\textbf{X}(t) \label{eq:euler_diffusion}.
\end{equation}

Starting from the discretized version, Hansen and Gebhart \cite{hansen2020sheaf} proposed the first Sheaf Neural Network model, defined as follows:

\begin{equation}
 \textbf{Y} = \sigma((\textbf{I}_{nd} - \Delta_{\mathcal{F}})(\textbf{I}_{n} \otimes \textbf{W}_{1})\textbf{X}\textbf{W}_{2}) \label{eq:snn},
\end{equation}

where $\otimes$ denotes the Kronecker product.
This expression equips the standard discretized diffusion model with a non-linearity $\sigma$ and (assuming  $\textbf{X} \in \mathbb{R}^{nd \times f_1}$) with the weight matrices $\textbf{W}_1 \in \mathbb{R}^{d \times d}$, which multiplies independently stalk features of all nodes in all channels, and $\textbf{W}_2 \in \mathbb{R}^{f_1 \times f_2}$, which instead can modify the number of feature channels in a layer from a quantity $f_1$ to $f_2$.

It is immediate to observe that when $\Delta_{\mathcal{F}}$ is the standard normalized graph Laplacian, and $\mathbf{W}_{1}$ is set to be a scalar, the expression coincides with the GCN model proposed by Kipf and Welling \cite{kipf2016semi}. This means that SCNs can be seen as a \textit{generalization} of GCNs, and they both are a non-linear, parametric and discretized version of the sheaf diffusion process described in Equation \ref{eq:sheaf_diff_contin}. For these reasons this model is generally called \textit{Sheaf Convolutional Network} (SCN).

A natural question that arises at this point is how expressive these non-linear models
are compared to their base diffusion process. In order to provide an answer, Bodnar et al. \cite{bodnar2022neural} carried out an extensive investigation on how SCNs affect the so-called \textit{sheaf
Dirichlet energy}, which sheaf diffusion is known to minimize over time.
The results show that, additionally to sheaf diffusion being more expressive than heath diffusion, SCNs are more expressive than GCNs (in the sense that they are not constrained to decrease the Dirichlet energy, as GCNs instead do) and they have greater control over their asymptotic behavior. Further details and proofs can be found in \cite{bodnar2022neural}.

\subsection{Linear Sheaf Diffusion}

The previous section aimed at discussing the motivations leading to the introduction of a more complex geometric structure on graphs, pointing out the benefits it brings into play specifically for the task of node classification. 
When proceeding to practically implement the model, many questions may arise, for example regarding how to set the correct value for the stalk dimension $d$, or more importantly how to properly define the sheaf itself. Indeed, the ground truth sheaf is generally unknown or unspecified. In the SCN model that Hansen and Gebhart \cite{hansen2020sheaf} proposed as it is expressed in Equation \ref{eq:snn}, they use a \textit{hand-crafted} sheaf with $d = 1$, that is constructed with the assumption of fully knowing the data-generating process (in a synthetic setting).

Bodnar et al. \cite{bodnar2022neural}, instead, suggest a more scalable approach that allows to learn the sheaf end-to-end directly from data, in order to pick the right geometry for solving the specific task at hand and making the model applicable to any real-world graph dataset, even in the absence of a sheaf structure. This is not the only difference they propose with respect to the SCN model though, and their contributions to the model can be summarized in the following points:

\begin{enumerate}
    \item Not only the model learns the sheaf structure from data, but it learns multiple ones: they introduce a sheaf Laplacian $\Delta_{\mathcal{F}(t)}$ of a sheaf $(G, \mathcal{F}(t))$ that evolves over time.  Indeed, in order to manipulate the geometry of the graph and the diffusion process making use of the latest available features, they describe the evolution of the sheaf structure through a learnable function of the data: $(G, \mathcal{F}(t)) = g(G, \textbf{X}(t);\theta)$.
    \item They use stalks with $d \geq 1$ and higher-dimensional maps to fully exploit the generality of sheaves.
    \item The model performs a residual parametrization of
the discretized diffusion process, because it shows to empirically improve the performance.
\end{enumerate}

\paragraph{General formulation and practical implementation} 
Starting from the standard sheaf diffusion Equation \ref{eq:sheaf_diff_contin}, Bodnar et al. \cite{bodnar2022neural} define a more expressive diffusion-type model by enriching the equation with additional elements (two weight matrices $\textbf{W}_1$ and $\textbf{W}_2$ and a possibly nonlinear function $\sigma$), while still maintaining all the desirable properties outlined in Section \ref{motivation}:

\begin{equation}
     \dot{\mathbf{X}}(t) = - \sigma(\Delta_{\mathcal{F}(t)}(\textbf{I}_n\otimes \textbf{W}_{1})\textbf{X}(t)\textbf{W}_{2}) \label{eq:diffusion_bodnar}
\end{equation}

As mentioned before, $\Delta_{\mathcal{F}(t)}$ is the Laplacian of a sheaf that evolves over time. 
For the experiments a time-discretized version of \ref{eq:diffusion_bodnar} is used, that allows to use a new set of edges for each layer:

\begin{equation}
     {\mathbf{X}^{(t + 1)}} = \textbf{X}^{(t)} - \sigma(\Delta_{\mathcal{F}(t)}(\textbf{I}_n\otimes \textbf{W}_{1}^{(t)})\textbf{X}^{(t)}\textbf{W}_{2}^{(t)})  \label{eq:diffusion_bodnar_discrete}.
\end{equation}

In practice, this is further enriched in expressiveness by learning an additional parameter $\epsilon \in [-1,1]^d$ (i.e. a $d$-dimensional vector) that allows the model to adjust the relative magnitude of the features in each stalk dimension: 

\begin{equation}\label{eq:diffusion_bodnar_discrete_eps}
    {\mathbf{X}^{(t + 1)}} = \left( 1+\varepsilon \right) \textbf{X}^{(t)} - \sigma(\Delta_{\mathcal{F}(t)}(\textbf{I}_n\otimes \textbf{W}_{1}^{(t)})\textbf{X}^{(t)}\textbf{W}_{2}^{(t)}).
\end{equation}

Additionally to using a MLP to compute $\textbf{X}(0)$ from raw features and a final linear layer for node classification, Bodnar et al. \cite{bodnar2022neural} also came up with a method to learn sheaf structures \textit{locally}. This is done by parametrizing the $d \times d$ matrices of restriction maps $\mathcal{F}_{v \unlhd e}$ through a matrix valued function $\Phi: \mathbb{R}^{d \times 2} \rightarrow \mathbb{R}^{d \times d}, \, \mathcal{F}_{v \unlhd e := (v,u)} = \Phi(\textbf{x}_{v}, \textbf{x}_{u})$, in which $\Phi(\textbf{x}_{v}, \textbf{x}_{u}) = \sigma(\textbf{W}[\textbf{x}_{v}||\textbf{x}_{u}])$ where $\textbf{W}$ is the same for all couples of neighboring nodes. 
$\Phi$ is generally non-symmetric, in order to be able learn asymmetric transport maps along each edge. Additionally, if such function has enough
capacity and the features are diverse enough, the following result shows that it is possible to learn any sheaf over a graph \cite{bodnar2022neural}, and this motivates the choice of learning a different sheaf at each layer, which should grant the ability to distinguish more nodes after each aggregation step.

\begin{proposition}[\cite{bodnar2022neural}]\label{prop_mlp}
    Let $G = (V, E)$ be a finite graph with features $\textbf{X}$. Then, if $(\textbf{x}_{v}, \textbf{x}_{u}) \neq (\textbf{x}_{w}, \textbf{x}_{z})$ for any $(v, u) \neq (w, z) \in E$ and $\Phi$ is an MLP with sufficient capacity, $\Phi$ can learn any sheaf $(G, \mathcal{F})$. 
\end{proposition}

In the paper \cite{bodnar2022neural} they consider three types of function $\Phi$, according to the constraint imposed on the learnt matrix \textbf{W}:

\begin{itemize}
    \item \textit{Diagonal:} in this case the sheaf Laplacian is a block-diagonal matrix, hence fewer parameters need to be learned and there are fewer operations to perform in sparse matrix multiplication. The $d$ stalk dimensions, though, interact only through the multiplication with $\textbf{W}_1$.
    \item \textit{Orthogonal:} they allow the $d$ dimensions of the stalks to better interact, while still reducing the number of free parameters though the orthogonality constraint. The Laplacian is also easy to normalize. In this setting the model learns a discrete $O(d)$-bundle.
    \item \textit{General:} in this case, artitrary matrices are learned. The higher degree of flexibility comes with drawbacks, though: the risk of overfitting is higher, and the sheaf Laplacian is more difficult to normalize numerically. 
\end{itemize}

\paragraph{Results and further directions}
Neural Sheaf Diffusion (NSD) \cite{bodnar2022neural} models have been tested in both a synthetic and real-world setting. For what concerns real-world datasets, they considered 9 different benchmark datasets that span different levels of the  homophily coefficient $h$, ranging from $h = 0.11$ (very heterophilic) to $h =0.81$ (very homophilic). A summary of the best results is reported in Table \ref{tab:main_results}, and they show that NSD models are among the top three on 8/9 datasets, with $O(d)$-bundle models being the ones that perform best overall, followed by the ones learning simpler diagonal maps. In conclusion, they demonstrated that the sheaf diffusion method can achieve state-of-the-art results in heterophilic settings.

Starting from the NSD framework, various research direction have been recently explored, by providing some modifications to the underlying method or studying different application tasks. They include \textit{Sheaf Attention Networks} (SheafAN) \cite{barbero2022sheaf} and \textit{Neural Sheaf Propagation} (NSP) \cite{suk2022surfing}.

The central idea behind SheafAN \cite{barbero2022sheaf} is equipping the NSD model \cite{bodnar2022neural} with an explicit attention mechanism in the same fashion in which GAT \cite{velivckovic2017graph} does with the standard GCN model \cite{kipf2016semi}. This gives rise to a more expressive attention mechanism, equipped with geometric inductive biases, that consistently outperforms GAT on both synthetic and real-world benchmarks \cite{barbero2022sheaf}. In practice, a standard SheafAN layer is defined as:

\begin{equation}\label{eq:sheafan}
    \textbf{X}^{(t + 1)} = \sigma((\hat{\Lambda}(\textbf{X}^{(t)})\odot \textbf{A}_{\mathcal{F}}-\mathbf{I})(\mathbf{I}_n \otimes \mathbf{W}^{(t)}_{1})\mathbf{X}^{(t)}\mathbf{W}^{(t)}_{2}).
\end{equation}

Also a residual version, Res-SheafAN, is proposed, that is 

\begin{equation}\label{eq:res_sheafan}
    \textbf{X}^{(t + 1)} = 
    \textbf{X}^{(t)} + \sigma((\hat{\Lambda}(\textbf{X}^{(t)})\odot \textbf{A}_{\mathcal{F}}-\mathbf{I})(\mathbf{I}_n \otimes \mathbf{W}^{(t)}_{1})\mathbf{X}^{(t)}\mathbf{W}^{(t)}_{2}),
\end{equation}

and in both of them $\hat{\Lambda} = \Lambda \otimes \textbf{1}_d$ and each entry is computed through an attention function: $\Lambda_{ij} = a(\textbf{x}_{i}, \textbf{x}_{j})$.

For what concerns NSP \cite{suk2022surfing}, instead, they explore a different diffusion process with respect to heat equation, using in this case the PDE associated to the wave equation on the sheaf (Figure \ref{fig:wave}). This choice is motivated by the fact that the wave equation preserves a higher amount of total 
energy of the signal with respect to the heat equation \cite{suk2022surfing}. The time-dependent continuous sheaf diffusion process, differently from Equation \ref{eq:sheaf_diff_contin}, is described by:

\begin{equation}\label{eq:wave}
    \Ddot{\mathbf{X}}(t) = -\Delta_{\mathcal{F}(t)}\mathbf{X}(t).
\end{equation}

In the practical implementation, in order to obtain the layer-wise operation, Equation \ref{eq:wave} is discretized through the leapfrog method and weight matrices are added as well as an activation function $\sigma$, obtaining in the end the following expression:

\begin{equation}\label{eq:wave_discrete}
    {\mathbf{X}}^{(t+1)} = 2\textbf{X}^{(t)} - \textbf{X}^{(t-1)} - \sigma(\Delta_{\mathcal{F}_{(t)}}(\mathbf{I}_n \otimes \mathbf{W}^{(t)}_{1})\mathbf{X}^{(t)}\mathbf{W}^{(t)}_{2}).
\end{equation}

\begin{figure}[ht]
    \centering
    \includegraphics[width=4.50in]{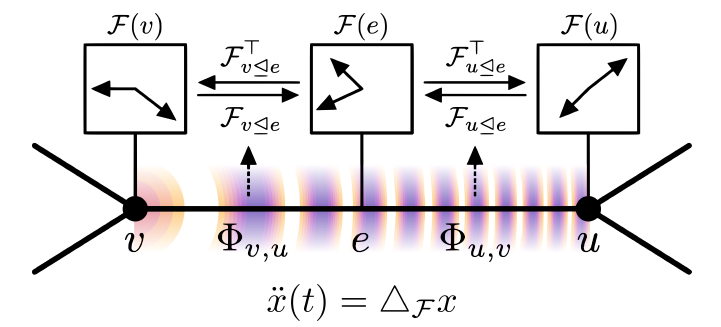}
    \caption{Neural sheaf propagation (NSP) induced by the wave equation on sheaves (image from \cite{suk2022surfing}).}
    \label{fig:wave}
\end{figure}

These models provide a valuable insight on some promising directions of research starting from the NSD \cite{bodnar2022neural} framework. A summary of the best results obtained by the two models described above are reported in Chapter 5 (Experiments), in Table \ref{tab:main_results}.

\clearemptydoublepage

\newpage

\;

\newpage

\chapter{Nonlinear Sheaf Diffusion}
\label{cha:nonlinear}

\section{A Perspective from Opinion Dynamics}\label{opinion_dynamics}

The study of opinion dynamics is a fascinating and challenging area of research that has attracted the attention of many scientists from different fields. Opinion dynamics models attempt to capture the complex and dynamic nature of social interactions that underlie the formation and evolution of opinions in human societies. These models have been used to study a wide range of phenomena, including political polarization \cite{holme2006nonequilibrium}, the spread of rumors \cite{del2016spreading}, the formation of echo chambers \cite{del2016echo}, and the emergence of consensus \cite{degroot1974reaching}.

In recent years, there has been growing interest in using computational tools to model opinion dynamics. Structural effects of the considered network on opinion dynamics began with
the analysis of linear dynamical models \cite{degroot1974reaching,friedkin1990social,taylor1968towards} and have evolved into more
sophisticated mathematical formulations \cite{deffuant2000mixing,dittmer2001consensus,rainer2002opinion}, including features such as \textit{bounded confidence}. In the context of opinion dynamics, bounded confidence refers to a model where individuals update their opinions based on their neighbors' only if the difference between the others' opinion and their own falls within a certain confidence bound, that is only if they are sufficiently similar.

Sheaf theory has emerged as a powerful mathematical framework for studying complex systems with local and global interactions \cite{hansen2021opinion,hansen2019toward,hansen2020laplacians}. In their paper  \textit{Opinion Dynamics on Discourse Sheaves} \cite{hansen2021opinion}, Hansen and Ghrist propose a novel approach to modeling opinion dynamics based on sheaf theory. Some of the concepts and contents in the following sections are inspired by their work, including the theoretical proofs and results. 

\subsection{Opinion Dynamics on Discourse Sheaves}

\providecommand{\introduce}[1]{\textit{#1}}

Consider a social network expressed as a graph, denoted by $G$, where the nodes represent individual agents and the edges represent pairwise communication. In order to capture the dynamics of opinions within this network, we introduce a concept called a \introduce{discourse sheaf}, denoted by $\mathcal{F}$. In this framework, each agent has an \introduce{opinion space} represented by a real vector space, consisting of various topics as its basis. The opinion space is analogous to traditional opinion dynamics models, where each axis represents negative, neutral, or positive opinions on a given topic, with the intensity measured by a scalar value. The opinion space serves as the stalk $\mathcal{F}(v)$ for each vertex $v$, from which each element $x_v$ represents the intensity of opinions or preferences on the \introduce{basis topic}.

All pairs of agents who are connected by an edge $e$, that we may for example denote as $u$ and $v$, engage in discussions regarding a specific set of topics. These topics may not be the same as those generating their respective opinion spaces, but they form the basis of an abstract \introduce{discourse space}, denoted by $\mathcal{F}(e)$, which serves as the stalk over the edge.

During discussions, each agent expresses their opinions on the relevant topics as a linear combination of their existing opinions on their personal basis topics. This expression of opinion is captured by the linear transformations $\mathcal{F}_{u \unlhd e}:\mathcal{F}(u)\to\mathcal{F}(e)$ and $\mathcal{F}_{v \unlhd e}:\mathcal{F}(v)\to\mathcal{F}(e)$. If agents $u$ and $v$ hold opinions $x_{u}\in\mathcal{F}(u)$ and $x_{v}\in\mathcal{F}(v)$ respectively, consensus is achieved when $\mathcal{F}_{u\unlhd e}(x_{u})=\mathcal{F}_{v \unlhd e}(x_{v})$. This condition represents the local consistency imposed by each edge, where the linear constraints ensure consistency between the stalks of the incident vertices. It is important to note that consensus in this context does not imply that agents $u$ and $v$ share the exact same opinions, but rather that their expressions of personally held opinions appear to align.

This framework of discourse sheaves provides a structured approach to analyze opinion dynamics within social networks, incorporating the idea of consensus while accounting for individual expressions of opinions on various topics, and a simple representation of what has just been described is in Figure \ref{fig:discoursesheaf}.

\begin{figure}
\centering
\includegraphics[width=4.75in]{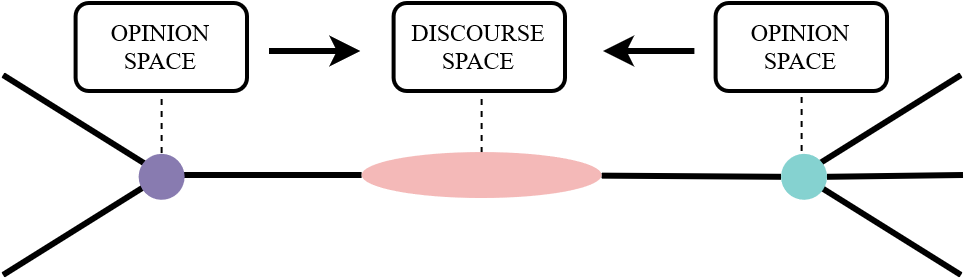}
\caption{Representation of a discourse sheaf: stalks over vertices are individual opinion spaces, stalks over edges represent discourse spaces, and restriction maps are expressions of opinions on the topics of discourse.}
\label{fig:discoursesheaf}

\end{figure}

The concepts discussed in Sections \ref{sec:sheaves_general} and \ref{sec:cellular_sheaves} can be better contextualized within the framework of discourse sheaves. In this context, we consider the following:

\begin{itemize}
\item A 0-cochain $x\in C^0(G,\mathcal{F})$ describes a private distribution of opinions among the agents.
\item A 1-cochain $\xi\in C^1(G,\mathcal{F})$ represents the distribution of expressed opinions resulting from pairwise discussions between agents in the network.
\item The coboundary map $\delta:C^0\to C^1$ captures the variation in opinion expression among agents based on their individual viewpoints.
\item The sheaf Laplacian $L_\mathcal{F}:C^0\to C^0$ quantifies the \textit{discord} within the system. The value of $L_\mathcal{F}(x)$ at each vertex $v$ measures the discrepancy between $x_v$ (the opinion of agent $v$) and the opinions that would lead to harmony between $v$ and its neighboring agents.
\end{itemize}

The most general discourse sheaf model extends the usual consensus problem over graphs by utilizing sheaves programmed with linear transformations. This allows for additional features not commonly found in the literature. In a three-agent system, where agents A, B, and C communicate pairwise, the discourse sheaf model enables agents to hold private opinions on topics that may not be shared among them. The model also allows for discussions on topics unrelated to their basis opinions. Policies can be formed by combining principles through restriction maps, enabling expressions of preferences that capture individual neutral opinions, preferences and dislikes. Restriction maps can be scaled positively or negatively to modify expressed opinions, including the potential for \textit{selective lying} or \textit{deception}. 

Opinion dynamics on a discourse sheaf can be set up in various ways, such as allowing agents to change their opinions over time or facilitating changes in the expression of opinions to foster a more harmonious discourse. This model allows for the co-evolution of both opinions and their expression \cite{hansen2021opinion}.

\subsection{Nonlinear Laplacian and Bounded Confidence}

In their sheaf-based model of opinion dynamics, Hansen and Ghrist \cite{hansen2021opinion} also introduce and analyze a nonlinear type of sheaf Laplacian, which is still defined as an operator on the space of 0-cochains of a sheaf. Specifically, the result of its application can be visualized as a matrix in which the entry corresponding to the $i$th row and $j$th column is given by a nonlinear function of the difference between the opinions of nodes $i$ and $j$.

The authors \cite{hansen2021opinion,hansen2020laplacians} show that the nonlinear Laplacian has several important properties that make it well-suited for modeling particular types of opinion-spreading models, such as \textit{bounded confidence} and \textit{antagonistic dynamics}. 

The nonlinear Laplacian can still be used to define a measure of consensus in the network and study the dynamics of opinions, and it is also more flexible than the traditional Laplacian matrix, which is based on simple linear differences between the opinions of nodes.

\section{Nonlinear Laplacian}\label{sec:nonlin_lap}
Let $(G,\mathcal{F})$ be a cellular sheaf on a graph $G = (V, E)$, with $\mathcal{F}(v)$ and $\mathcal{F}(e)$ being the stalks of a node $v$ and edge $e$, and  $\mathcal{F}_{v \unlhd e}: \mathcal{F}(v) \rightarrow \mathcal{F}(e)$  denoting the restriction map between an incident node-edge pair. A formal definition of the nonlinear sheaf Laplacian, as well as the introduction of bounded confidence opinion dynamics for discourse sheaves were proposed by Hansen and Ghrist in their publications \cite{hansen2019toward,hansen2021opinion,hansen2020laplacians}. 

\begin{definition}[Nonlinear sheaf Laplacian, \cite{hansen2019toward}]\label{def:nonlin_laplac}
    Let $G = (V, E)$ be a graph and let $(\mathcal{F}, G)$ be a cellular sheaf built on it. Furthermore, if $\phi_{e}: \mathcal{F}(e) \rightarrow \mathcal{F}(e) $ is a continuous and not-necessarily-linear map defined edge-wise for each $e \in G$, and $\Phi: C^{1}(G,\mathcal{F}) \rightarrow C^{1}(G,\mathcal{F})$ is the combination of all such maps, the corresponding \textit{nonlinear sheaf Laplacian} \cite{hansen2019toward} is $L_{\mathcal{F}}^{\Phi} = \delta^{T} \circ \Phi \circ \delta$, with $\delta: C^{0}(G,\mathcal{F}) \rightarrow C^{1}(G,\mathcal{F})$ denoting the coboundary map. Given $x \in C^{0}(G,\mathcal{F})$, since the nonlinear map $\Phi$ is applied edge-wise, $L_{\mathcal{F}}^{\Phi}x$ can still be computed locally in the network.  
\end{definition}

The \textit{nonlinear sheaf diffusion} is still described by the same PDE that held for the linear case, that is  Equation \ref{eq:sheaf_diff_contin}, with the only obvious difference that the newly defined nonlinear Laplacian has to be put in place of the standard linear one:

\begin{equation}
 \textbf{X}(0) = \textbf{X}, \, \, \, \,     \dot{\mathbf{X}}(t) = - L_{\mathcal{F}}^{\Phi}\textbf{X}(t).\label{eq:sheaf_diff_nonlin_contin}
\end{equation}


One way to construct a nonlinear Laplacian, but not the only possible one, is by defining a set of edge-wise potential functions $U_{e}: \mathcal{F}(e) \rightarrow \mathbb{R}$, that in turn allow to define a function $\Psi$ on the space of 0-cochains $C^{0}(G,\mathcal{F})$ as
\begin{equation}\label{eq:potential}
   \Psi(x) = \sum_{e}U_{e}(\delta_{e}x) = U(\delta x).
\end{equation}
The gradient of this function at a point $x$ is $\nabla \Psi(x) = \delta^{T} \circ \nabla U \circ \delta $. This is a nonlinear sheaf Laplacian $L_{\mathcal{F}}^{\Phi}$ with $\Phi = \nabla U $ and its corresponding heat equation is precisely gradient descent on $\Psi$.

The reason why using edge potentials is convenient for constructing nonlinear sheaf Laplacians is that they allow for a simplified analysis of the heat equation \cite{hansen2021opinion}: for example, if the potential functions $U_{e}$ are convex, $\Psi$ is a Lyapunov function that ensures stability of the dynamics. Furthermore, if each $U_e$ is radially unbounded and has one single local minimum in 0, the convergence behaviors of linear and nonlinear heat equations are the same \cite{hansen2021opinion}. A detailed study with proofs on this topic was carried out by Hansen and Ghrist \cite{hansen2021opinion}.

\subsection{Bounded Confidence Dynamics}\label{def_bounded}
A specific type of edge potential functions allows to extend the bounded confidence opinion dynamics to discourse sheaves. The central idea of this model is that individuals only have confidence in the opinions of neighbors that
are sufficiently similar to their own, and thus only take these opinions into account
when updating.
Formally, for each edge $e$ of $G$ a threshold $D_{e}$ is set, as well as a differentiable function $\psi_{e}: [0, \infty) \rightarrow \mathbb{R}$ such that $\psi'_{e}(y) = 0$ for $y \geq D_{e}$ and $\psi'_{e}(y) > 0$ for $y < D_{e}$. The edge potential function in this case is given by $U_{e}(y_{e}) = \psi_{e}(\lVert y_{e}\rVert^{2}).$ Consequently, the gradient becomes $\nabla U_{e} (y_{e}) = \psi'_{e}(\lVert y_{e}\rVert^{2})y_{e}$ and the associated nonlinear Laplacian $ L_{\mathcal{F}}^{\nabla U}x$ can be expressed node-wise as 
\begin{equation}\label{eq:nonlin_lap}
   (L_{\mathcal{F}}^{\nabla U}x)_{v} = \sum_{u, v \unlhd e} \mathcal{F}_{v \unlhd e}^{T}\psi'_{e}(\lVert \mathcal{F}_{v \unlhd e}x_{v} - \mathcal{F}_{u \unlhd e}x_{u}\rVert^{2})(\mathcal{F}_{v \unlhd e}x_{v} - \mathcal{F}_{u \unlhd e}x_{u}).
\end{equation}
In comparison with the formula for the standard sheaf Laplacian, there is a nonlinear scaling factor depending on the discrepancy over each edge, that allows to generate bounded opinion dynamics. 

The constraints that the function $\psi$ is required to satisfy to allow for the bounded confidence phenomenon also act positively in theoretically guaranteeing the convergence of the dynamics in the time limit, as it is investigated and proved in detail by Hansen and Ghrist \cite{hansen2021opinion}.

\section{Nonlinearity in Sheaf Neural Networks: Model Definition}\label{sec:model_def}

The primary objective of this project was to implement a Sheaf Neural Network with a nonlinear Laplacian based on the model proposed by Bodnar et al. and discussed in Section \ref{sec:section_snn}, thus also enabling sheaf learning. Because of this, we like referring to the explored model and procedure as \textit{Nonlinear Sheaf Diffusion}.
Broadly speaking, the aim of the work was to investigate the potential benefits of introducing a nonlinear Laplacian in Sheaf Neural Networks for node classification tasks. 

Two key ideas motivated the development of this study:

\begin{enumerate}
    \item Firstly, we were driven by simple curiosity regarding the impact of a nonlinearity in the Laplacian on diffusion dynamics, on signal propagation in discrete-time settings, and on the overall network performance. It was also intriguing to evaluate the nonlinear model using the same datasets as in \cite{bodnar2022neural}, in order to analyze how the nonlinearity in the Laplacian influences the results in relation to the heterophily level of the graph.
    \item Secondly, we aimed to explore the application of the phenomenon of \textit{ignoring neighbors' opinions} that emerges within the framework of bounded confidence dynamics. We wanted to investigate the possibility of performing \textit{edge pruning} for graph-related tasks by leveraging the nonlinear scaling factor in \ref{eq:nonlin_lap}. This factor can implicitly remove or weaken the connectivity between two nodes during diffusion based on the discrepancy of their features. Our question was whether the network could effectively utilize this element to detect and somehow \textit{ignore} useless or noisy edges in the graph when the signal is propagated through sheaf diffusion.
\end{enumerate}

The analysis conducted in this study was primarily experimental rather than theoretical: in order to validate their practical effectiveness, we tested different versions of the model on real-world and synthetic datasets. We focused on experiments because practical implementations in discrete-time settings rarely conform to the conditions and assumptions made in theoretical work. Consequently, the guarantees provided by theoretical analyses cannot always be observed in practice. Furthermore, since Bodnar et al. extensively covered the theoretical aspects of sheaf diffusion convergence properties and guarantees in \cite{bodnar2022neural}, our work shifted towards proving the practical usefulness of the proposed model, despite the complexity and dimensionality overheads it entails.

Although incorporating nonlinearity into the definition of the sheaf Laplacian may appear not so challenging or capable of inducing significant changes by looking at its definition, its practical implementation actually raised several questions due to the numerous decisions involved in handling the intricacies. Therefore, we designed and implemented various model variations before coming up with the finalized and best performing version, and the subsequent sections describe the process leading to their formulation.

\subsection{Main Questions and Implementation Choices} \label{questions}

\begin{figure}[!t]
\centerline{\includegraphics[scale=0.22]{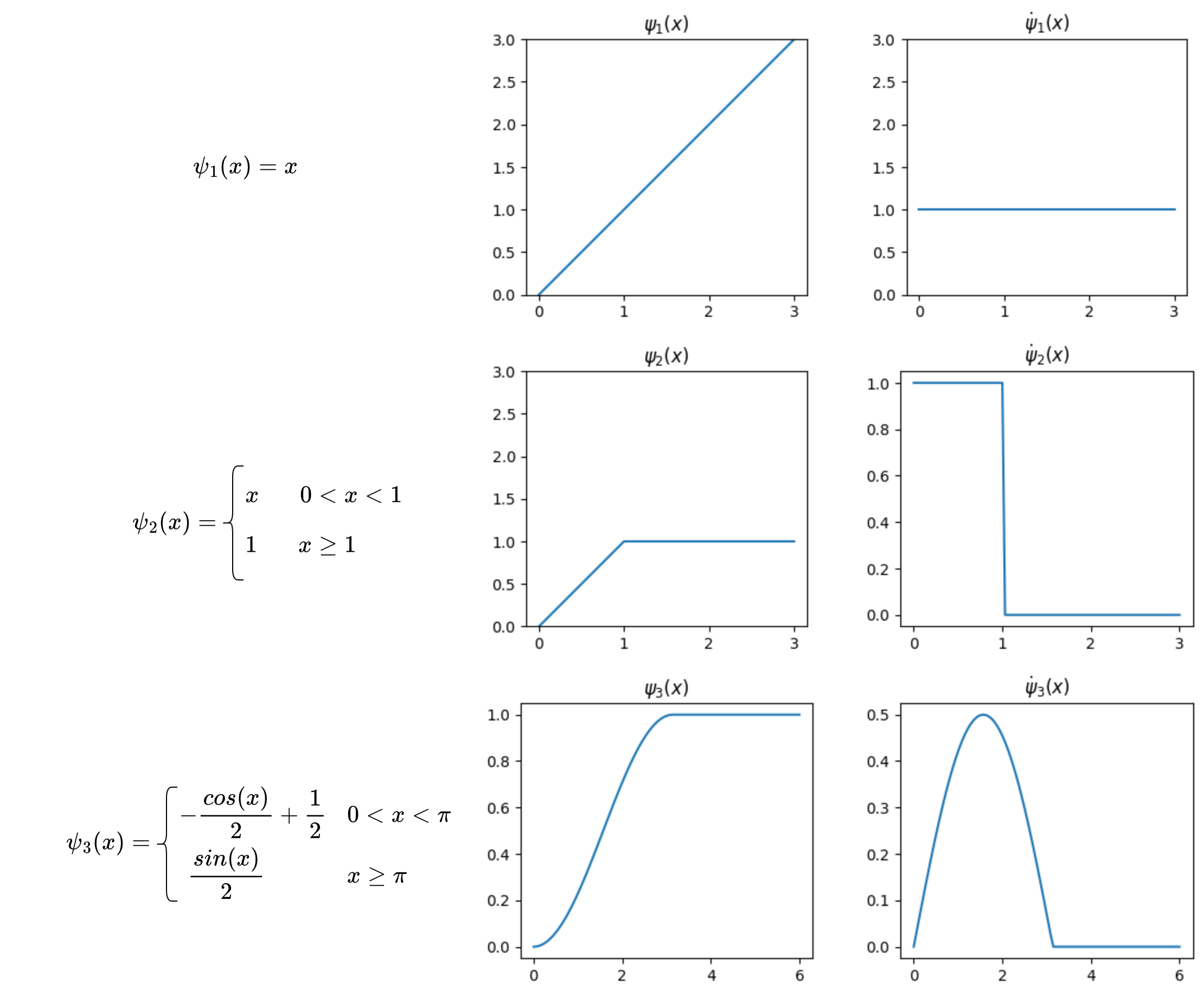}}
 \caption{Three examples for $\psi_{e}$ (referred to as $\psi_{1}$, $\psi_{2}$ and $\psi_{3}$) and its derivative $\psi'_{e}$. The first function is the one implicitly used in the standard linear sheaf Laplacian. The second and third ones, instead, define an edge potential generating bounded confidence dynamics, as both functions align with the requirements listed in \ref{def_bounded}. The threshold value $D_{e}$ is set to 1 and $\pi$, respectively.}
 \label{fig:1tr}
\end{figure}

Starting from the general Equation \ref{eq:sheaf_diff_nonlin_contin} describing nonlinear sheaf diffusion, and considering the practical discrete-time implementation of NSD \cite{bodnar2022neural} as reference (Equations \ref{eq:diffusion_bodnar_discrete}, \ref{eq:diffusion_bodnar_discrete_eps})
two fundamental points had to be properly investigated and analyzed, concerning the \textit{choice for the nonlinear function} and the \textit{normalization criteria}.

\hfill

\textbf{Nonlinear function definition}
The first and most relevant choice to be made is how to define the $\Phi$ function introduced in Definition \ref{def:nonlin_laplac} that corresponds to the nonlinearity of the Laplacian. Although using edge potentials as the ones that lead to a bounded confidence dynamic (in Equation \ref{eq:nonlin_lap}) is convenient for the guarantees they bring in terms of convergence in the time limit, in practice they do not represent the only option that could be explored. Indeed, during our study we adopted the bounded confidence schema first, and then a completely different approach in which $\Phi$ is fully defined as a Multi-Layer Perceptron (MLP \cite{rosenblatt1958perceptron}).

\begin{enumerate}
    \item \textbf{Bounded confidence} Motivated by the curiosity of verifying whether a bounded confidence behavior in the diffusion process could be someway exploited to perform edge pruning with noisy edges in the network, one option is defining the Laplacian as it is described in Subsection \ref{def_bounded}, and specifically as in expression \ref{eq:nonlin_lap}. Although this imposes numerous constraints, the definition still leaves a certain degree of freedom:
    \begin{itemize}

        \item An important issue is determining the amount and type of parameters to learn. We decided to explore the option of defining the overall shape of $\psi_{e}$ a priori and keep it fixed except for the threshold value $D_{e}$, this one to be learnt through backpropagation during training. The number of free parameters also depends on whether a unique threshold $D$ is considered for all edges, or multiple ones: according to the definition in \ref{def_bounded}, in an ideal setting $D_{e}$ may depend on the edge $e$. One way to learn edge-dependent  thresholds could be for example through the use of a parametric map $D: E \rightarrow I \subseteq \mathbb{R}$, such that $D(e) = D_{e}$. We investigated both the use of such parametric map, implementing it as a MLP, and also the case of learning a single threshold value for all the graph.

        \item The requirements listed in \ref{def_bounded} for $\psi_{e}$ are quite generic:  it must be constant for $y \geq D_{e}$, but it is only required to be strictly increasing for $0 < y < D_{e}$. This is why at an initial stage, we considered a set of different possible shapes for such function, and compared their behavior under distinct conditions and settings. Some examples are the $\psi_{2}$ and  $\psi_{3}$ in Figure \ref{fig:1tr}. In this way we could select the shape for $\psi_{e}$ that worked best in practice, that is $\psi_{3}$, and could stick to that for the experiments.

    \end{itemize}
    
    \item \textbf{MLP} In this case the $\Phi$ function in the expression $L_{\mathcal{F}}^{\Phi} = \delta^{T} \circ \Phi \circ \delta$ is defined as a Multi-Layer Perceptron, composed by a certain amount of linear layers (from 1 to 4) followed by an activation function (that we defined as ReLU \cite{brownlee2019gentle}), that is directly applied on the result of the coboundary operator. This grants full freedom in the shape and behavior of the nonlinearity, that is no more constrained to satisfy the strict requirements listed in Subsection \ref{def_bounded}. Although this implementation doesn't necessarily benefit from good theoretical convergence guarantees as it was the case for bounded confidence, this is the model that provided the most interesting results in practice. 
    
\end{enumerate}

\hfill

\textbf{Normalization criteria} 
Because of the nonlinearity of the Laplacian operator, it is not possible to express the map as a matrix multiplication as it is done for NSD \cite{bodnar2022neural} in Equation \ref{eq:diffusion_bodnar_discrete}. This becomes evident when considering Definition \ref{def:nonlin_laplac}: differently from the linear case, there is no matrix associated to $L_{\mathcal{F}}^{\Phi}$ that can be considered independently from $x$ because $\Phi$ acts as a nonlinear function of node features and cannot be resolved into a linear operator. Consequently, it is not straightforward to define a normalized sheaf Laplacian as in the linear case, $\Delta_{\mathcal{F}} = D^{-\frac{1}{2}}L_{\mathcal{F}}D^{-\frac{1}{2}}$, with $D$ being the block diagonal of $L_{\mathcal{F}}$. The Laplacian normalization guarantees stability for the diffusion operator $H_{\Delta_{\mathcal{F}}} = I - \Delta_{\mathcal{F}}$. 
    
    With this kind of normalization being not feasible, two main solutions were explored:

    \begin{enumerate} 
        \item It was useful to directly regulate the diffusion process by maintaining a scaling factor $\alpha$ in the definition of the sheaf diffusion equation:

        \begin{equation}
         \dot{\mathbf{X}}(t) = - \alpha L_{\mathcal{F}}^{\Phi}\textbf{X}(t). \label{eq:sheaf_diff_alpha}
        \end{equation}
        
        A scaling factor is usually used in the definition of heat diffusion processes, and as Hansen and Gebhart do and suggest in \cite{hansen2020sheaf}, it may be a good idea to also keep it in the discrete-time versions to stabilize the dynamics. In this case the diffusion operator becomes $H_{L_{\mathcal{F}}}^{\alpha} = I - \alpha L_{\mathcal{F}}$ and an additional question arises, that is how to determine the best value for $\alpha$, and whether it should be learnt or be fixed and pre-defined.

        \item The secondly investigated option was emulating the normalization procedure of the linear Laplacian matrix through $D^{-\frac{1}{2}}$, that is not directly possible in the nonlinear case, by separately "normalizing" the coboundary operator $\delta$ and its transpose $\delta^{T}$. Indeed, even though a single Laplacian matrix is not present, two matrices are associated to $\delta$ and $\delta^{T}$, as shown also in Figure \ref{fig:coboundary_def}, and any kind of matrix operation can be performed on them.
        The intuition is that, as in the linear case $\Delta_{\mathcal{F}} = D^{-\frac{1}{2}}L_{\mathcal{F}}D^{-\frac{1}{2}} = D^{-\frac{1}{2}}\delta_{\mathcal{F}}^{T} \circ \delta_{\mathcal{F}} D^{-\frac{1}{2}} $, then also in the nonlinear setting something similar may work as well: $\Delta_{\mathcal{F}}^{\Phi} = D^{-\frac{1}{2}}\delta_{\mathcal{F}}^{T} \circ \Phi \circ D^{-\frac{1}{2}} \delta_{\mathcal{F}} $.
    \end{enumerate}

The numerous approaches tackling the implementation questions, that led to the construction of multiple versions of the model, can be summarized as shown in Figure \ref{fig:implementation_choices}.

\begin{figure}[t]
\centerline{\includegraphics[scale=0.25]{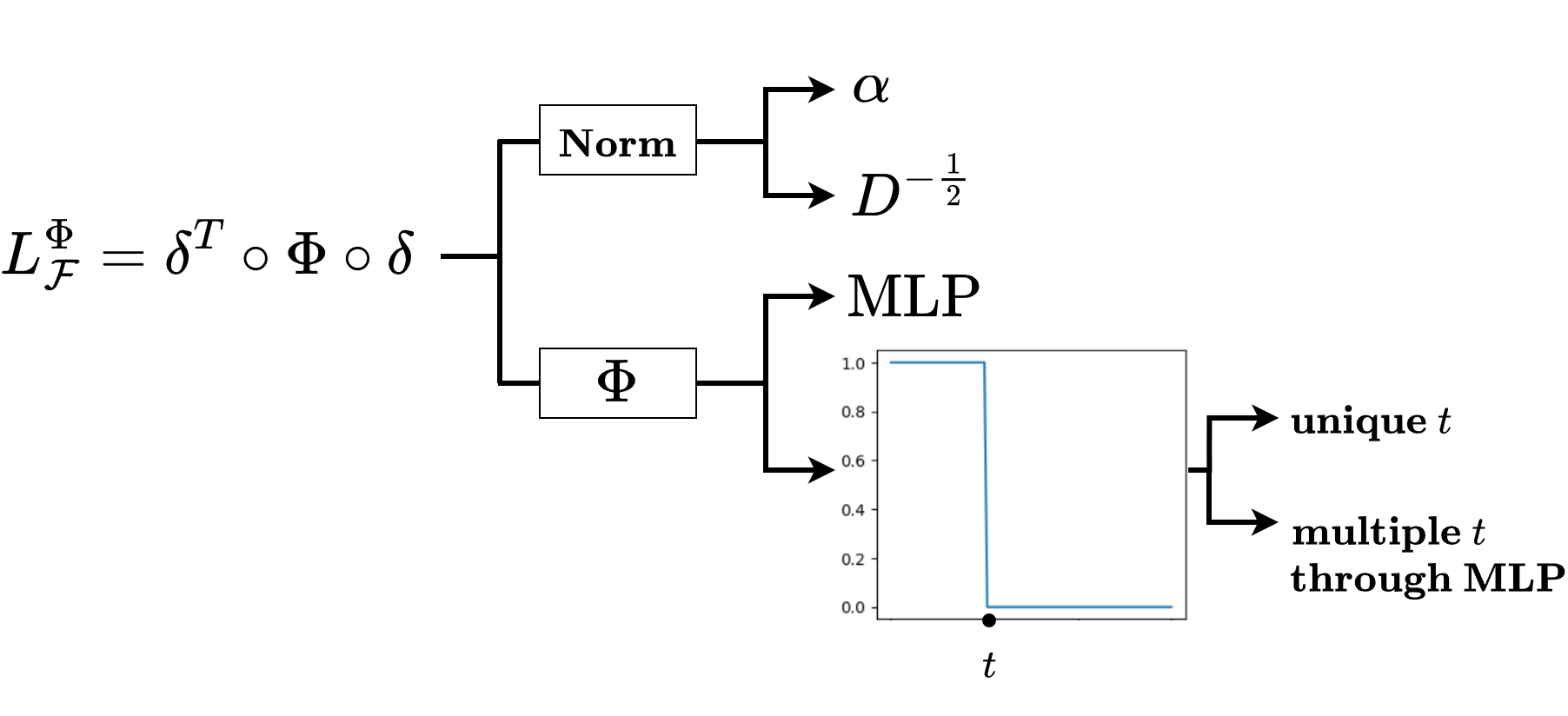}}
 \caption{Schema of the implementation choices that were explored in the process of finding the model with the best properties and characteristic in the nonlinear sheaf diffusion setting.}
 \label{fig:implementation_choices}
\end{figure}

\subsection{Types of Explored Models}

In this section, we will provide a comprehensive and detailed description of the various model variations that were developed and tested. The implementation choices align with the discussion points covered in the preceding section. The architectures will be presented in the precise chronological order of their development. We will delve into an analysis of their functionality, performance, and reasons behind the need for improvements due to identified drawbacks or malfunctions.

Building upon the practical implementation of NSD presented in Equation \ref{eq:diffusion_bodnar_discrete_eps}, we can derive the overarching expression for the Nonlinear Sheaf Diffusion models as follows, of which subsequent subsections will delve into specific cases or explore minor variations:

\begin{equation}
\textbf{X}^{(t+1)}=\left( 1+\varepsilon \right) \textbf{X}^{(t)}- \sigma \left( L_{\mathcal{F}(t)}^{\Phi}\left(\textbf{I}\otimes \textbf{W}_{1}^{(t)}\right) \textbf{X}^{(t)}\textbf{W}_{2}^{(t)}\right),
\end{equation}

which is equal to

\begin{equation}\label{eq:nonlin_expression_general}
\textbf{X}^{(t+1)}=\left( 1+\varepsilon \right) \textbf{X}^{(t)}- \sigma \left( \delta_{\mathcal{F}(t)}^{T}  \Phi \left( \delta_{\mathcal{F}(t)}\left(\textbf{I}\otimes \textbf{W}_{1}^{(t)}\right) \textbf{X}^{(t)}\textbf{W}_{2}^{(t)}\right)\right)
\end{equation}

with $\epsilon \in [-1, 1]^{d}$, learnt and different for each layer.

Additionally, it is important to specify that, even though the nonlinear Laplacian cannot be expressed as a matrix, the coboundary operator $\delta$ and its transpose $\delta^{T}$ can instead be singularly expressed as matrices as it was shown also in Figure \ref{fig:coboundary_def}. Thus, their associated operators in \ref{eq:nonlin_expression_general} are practically  implemented as sparse matrix multiplications. This holds for all model variations.

\begin{enumerate}
    \item
\textbf{Bounded Confidence, $\alpha$-normalization}
The first option to be explored for the definition of the nonlinearity was the bounded confidence model, while for what concerns the normalization criteria, it was first addressed by just utilizing a scaling factor $\alpha$. The chosen normalization criteria is inserted in the sheaf diffusion PDE as in Equation \ref{eq:sheaf_diff_alpha}, and in the discrete-time setting and practical implementation, the general expression above (Equation \ref{eq:nonlin_expression_general}) is updated in a very simple way:

\begin{equation}\label{eq:nonlin_expression_alpha_general}
\textbf{X}^{(t+1)}=\left( 1+\varepsilon \right) \textbf{X}^{(t)}- (1 + \alpha)\sigma \left( \delta_{\mathcal{F}(t)}^{T}  \Phi \left( \delta_{\mathcal{F}(t)}\left(\textbf{I}\otimes \textbf{W}_{1}^{(t)}\right) \textbf{X}^{(t)}\textbf{W}_{2}^{(t)}\right)\right).
\end{equation}

The value of $\alpha$ is chosen to be a learned vector $\alpha \in [-1, 1]^{d}$, and different for each layer, exactly as it holds for $\epsilon$. This set of conditions resulted to be the one granting the best performance in practice. 

For what concerns the nonlinear function $\Phi$, in order to give rise to bounded confidence dynamics in \ref{def_bounded}, it is defined as 

\begin{equation}
\Phi = \nabla U 
\end{equation}

such that 

\begin{equation}
\nabla U_{e} (y_{e}) = \psi'_{e}(\lVert y_{e}\rVert^{2})y_{e}\end{equation}

with $\psi'_{e}$ having the same properties as described in \ref{def_bounded}.

In practice, setting $\mathbf{Y} := \delta_{\mathcal{F}(t)}\left(\textbf{I}\otimes \textbf{W}_{1}^{t}\right) \textbf{X}_{t}\textbf{W}_{2}^{t}$:

\begin{equation}
    \Phi (\mathbf{Y}) = \psi'_{e}(\lVert \mathbf{Y}\rVert^{2})\mathbf{Y}.
\end{equation}

Regarding the shape of $\psi_{e}$ and its corresponding derivative $\psi'_{e}$, two main alternatives were explored: $\psi_{2}$ and $\psi_{3}$, as described in Figure \ref{fig:1tr}. In the first case, the derivative behaves as a simple step function, assuming a constant positive value (1) for inputs below a certain \text{threshold} and then transitioning to 0 for values exceeding the threshold. On the other hand, the second alternative involves a slightly more complex function. Instead of assuming a constant value for inputs below the threshold, the function exhibits an increasing and then decreasing behavior. Intuitively, this should capture the idea that when the result of the coboundary operator is almost equal to 0 (indicating high similarity between neighboring nodes' features), the signal is given a lower weight in the update compared to cases where the features are similar but not as much. The gradual decrease ensures a smoother transition to 0 when reaching the threshold value. In practice, the potential function $\psi_{2}$ demonstrated superior performance with respect to $\psi_{1}$, thus it was chosen as standard shape for $\psi_{e}$ for all subsequent experiments:

\begin{equation}
    \psi_{e}(x) = \psi_{2}(x) = 
    \begin{cases} 
      x & 0 < x < D_{e}  \\
      1 & x \geq D_{e}
   \end{cases}
   \label{eq:potential_bounded}
\end{equation}

\begin{equation}
    \psi'_{e}(x) = \psi'_{2}(x) = 
    \begin{cases} 
      1 & 0 < x < D_{e}  \\
      0 & x \geq D_{e}
   \end{cases}
\end{equation}

For what concerns the threshold value $D_{e}$, it is learnable and it could possibly assume any positive real value in $\mathbb{R}_{\ge 0}$. We investigated and implemented two cases:

\begin{itemize}
    \item A single threshold value $D$ is learnt for all the edges in the graph, but different ones for the different layers: $D_{e} = D \, \, \,  \forall e \in E$.
    \item A parametric map $D: E \rightarrow \mathbb{R}_{\ge 0}$, such that $D(e) = D_{e}$ is learnt, in a way that a specific threshold value is associated to each edge. This map is implemented through a MLP (that in practice is composed by just one layer, so it could also be referred to simply as a "Perceptron"), such that, given $e = (v, u)$, then $D(e) = \text{MLP}(|x_{v} - x_{u}|)$. In the last expression, $x_{v}$ and $x_{u}$ correspond to the node features in the layer right after applying the coboundary operator. Also in this case, different sets of weights for the MLP are learnt for different layers. 
\end{itemize}

The models resulting from this combination of implementation choices were tested on the same real-world datasets considered by Bodnar et al. in the NSD paper \cite{bodnar2022neural}, and their poor performance on datasets such as Chameleon and Squirrel led to look for a solution and to substantially modify part of the architecture.
The results of such experiments, that are reported in Appendix \ref{appendix_exp}, indicate that these models perform well when applied to small datasets; however, their accuracy noticeably decreases when dealing with datasets with a high number of edges. In these cases, the accuracy values are even lower than simpler GNN benchmarks like GCN \cite{kipf2016semi} and GAT \cite{velivckovic2017graph}. This decline in performance can be attributed to the normalization method's limitations in stabilizing the diffusion process for dense datasets. Consequently, alternative normalization techniques were explored, ultimately leading to the multiplication by squared diagonal matrices, as it will be introduced in the next subsection. This normalization approach proved to be the most effective one in multiple and different dataset settings.

\hfill 

\item \textbf{Bounded Confidence, $D^{-\frac{1}{2}}$-normalization}
Because of the poor scalability of the $\alpha$-nor\-ma\-li\-za\-tion method on dense and large datasets, we then investigated the option of performing a nor\-ma\-li\-za\-tion similarly to the symmetric Laplacian normalization performed in NSD \cite{bodnar2022neural}, that is $\Delta = D^{-\frac{1}{2}} L_{\mathcal{F}} D^{-\frac{1}{2}}$.

Following the intuitions expressed in Subsection \ref{questions} regarding the $D^{-\frac{1}{2}}$- normalization, the newly defined time-discrete diffusion model is the following:

\begin{equation}\label{eq:nonlin_expression_d}
\textbf{X}^{(t+1)}=\left( 1+\varepsilon \right) \textbf{X}^{(t)}- \sigma \left(D^{-\frac{1}{2}} \delta_{\mathcal{F}(t)}^{T} \Phi \left(D^{-\frac{1}{2}} \delta_{\mathcal{F}(t)}\left(\textbf{I}\otimes \textbf{W}_{1}^{(t)}\right) \textbf{X}^{(t)}\textbf{W}_{2}^{(t)}\right)\right).
\end{equation}

In order to understand how the normalization is performed in detail for each block of the coboundary matrix $\delta$, and how it strictly relates to how it was performed in the linear case, the reader may want to refer to Figure \ref{fig:diag_normalization}.

When implementing bounded confidence using a potential function $\psi_{e}$ as defined in Equation \ref{eq:potential_bounded}, if the derivative of this function is not 0 (or in other words, if there is some exchange of information between the considered edges), the nonlinear Laplacian becomes equal to the standard linear Laplacian. This occurs because if $\psi'_{e}(x) = 1$, the function $\Phi$ behaves exactly as the identity function. This motivates the definition of the normalization procedure for the coboundary operators: in this case, what emerges from their composition is exactly $\Delta_{\mathcal{F}}$. 

Although both this model and the previous one performed well on real-world datasets (see Tables \ref{tab:main_results} and \ref{tab:alpha}), it was challenging to find a synthetic dataset that clearly demonstrated the superiority of the implemented nonlinearity over other models in exhibiting the expected \textit{edge pruning} phenomenon. 
We created synthetic graph datasets for graph classification, where noisy edges should be disregarded for accurate node classification;
however, contrary to our initial expectations, the bounded confidence model did not outperform other benchmark GNN models or the standard linear sheaf models. 

This led us to the conclusion that it might be beneficial not to constrain the linearity to satisfy all the bounded confidence constraints strictly. Instead, allowing for a higher degree of freedom and expressiveness could potentially yield more intriguing results.
Based on this realization, we made further adjustments to the model, resulting in the following final implementation. Interestingly, this last version exhibited remarkable and innovative properties in a synthetic environment.

\begin{figure}[htp]
\centerline{\includegraphics[scale=0.23]{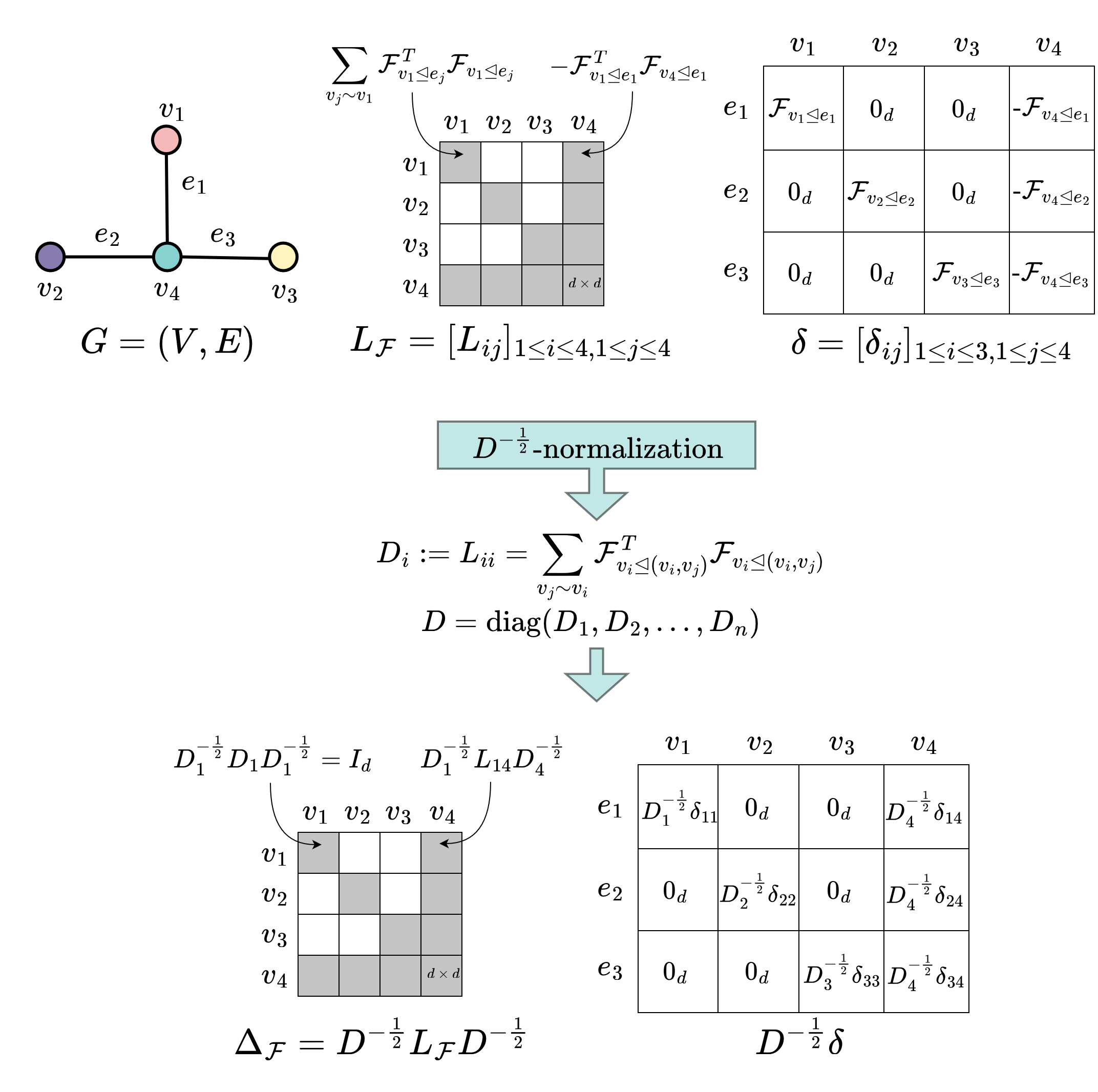}}
 \caption{The image shows how to compute the elements that are necessary for the $D^{-\frac{1}{2}}$-normalization, both in the linear and nonlinear setting. Starting from the Laplacian matrix $L_{\mathcal{F}}$ and from the coboundary operator $\delta$, the flow depicts how to first compute the diagonal matrix $D$. Then, starting from this, it illustrates how to obtain both the normalized linear Laplacian $\Delta_{\mathcal{F}}$ and the  factor $D^{-\frac{1}{2}}\delta$, which is deployed for the normalization of the coboundary operator in the nonlinear Laplacian.}
 \label{fig:diag_normalization}
\end{figure}

\hfill

\item \textbf{Nonlinearity as MLP, $D^{-\frac{1}{2}}$-normalization}

The last and final version of our model uses a MLP for implementing the nonlinearity, granting full freedom for its shape and behavior. It is defined by stacking a certain amount of linear layers followed by a ReLU \cite{brownlee2019gentle} activation function, and it is directly applied on the result of the coboundary operator. For what concerns the normalization procedure instead, the multiplication by $D^{-\frac{1}{2}}$ is still adopted.

The model's equation can be expressed as:

\begin{equation}\label{eq:nonlin_expression_d_mlp}
\textbf{X}^{(t+1)}=\left( 1+\varepsilon \right) \textbf{X}^{(t)}- \sigma \left(D^{-\frac{1}{2}} \delta_{\mathcal{F}(t)}^{T} \textbf{MLP} \left(D^{-\frac{1}{2}} \delta_{\mathcal{F}(t)}\left(\textbf{I}\otimes \textbf{W}_{1}^{(t)}\right) \textbf{X}^{(t)}\textbf{W}_{2}^{(t)}\right)\right).
\end{equation}

This version of the Nonlinear Sheaf Diffusion architecture has undergone testing on real-world datasets, as demonstrated in Table \ref{tab:main_results}. However, the most captivating aspects that distinguish this model as a superior choice compared to other benchmarks were observed evaluating it on synthetic datasets. These interesting findings are discussed in detail in the upcoming section.


\end{enumerate}

\clearemptydoublepage

\newpage

\chapter{Experiments}
\label{cha:experiments}

\section{Synthetic Experiments}\label{sec:synthetic}

As mentioned earlier, the synthetic dataset was initially designed to investigate whether the firstly implemented bounded confidence models would exhibit an \textit{edge pruning} effect, reducing or weakening connectivity between nodes based on their feature discrepancies during sheaf diffusion. Our goal was determining if the network could effectively utilize this element to identify and potentially disregard irrelevant or noisy edges in the graph during signal propagation.

However, the results did not align to the ones expected for these models. Interestingly though, the second implementation of the nonlinearity, with MLP, yielded surprising outcomes on the same datasets: this version did not simply ignore the edges but instead \textit{leveraged} them to enhance its classification capabilities. The following sections provide detailed information about the datasets used in the synthetic experiments.

\subsection{Dataset Idea}
The dataset was created with the idea of posing a 3-class node classification task. It consists of a sequence of graphs with the same nodes and node features, but with a different set of edges connecting them. More in detail, as the sequence progresses, random edges are added to the graphs (and some of the original ones may be removed, instead) in a specific manner. The goal of this construction is to simulate the existence of three distinct communities, such as humans belonging to different ethnicities. Initially the interactions occur within each community exclusively, resulting in three separate connected components in the graph. Later on, new connections are gradually introduced either between different communities or within the communities themselves, or of both types.

The reason behind the definition of new types of connections is to observe how significantly they impact the behavior of the models, that heavily rely on the edges of the graph to perform node classification. This happens because the node features for the different communities are designed in a way that makes it very difficult to classify individuals by solely relying on them, as it happens for linear models and MLPs.

\paragraph{Node features and communities}

As mentioned earlier, the nodes and their corresponding node features remain fixed across the different graphs within the dataset sequence. Specifically, each of the three communities consists of exactly 500 nodes, and for each community the node features are sampled from a different bivariate normal distribution. Figure \ref{fig:features_distribution} provides a detailed illustration of such definition. The underlying idea is that the node features alone should not be sufficient for accurate node classification, and that the models should primarily rely on the connections within the graphs for extracting the majority of information. To achieve this, the distributions are designed to overlap with each other for the most part, hence only very few \textit{outliers} are sampled from the tails of the distributions and not in the overlapping region. Specifically, the means of the three distributions are set to be $(0,0)$, $(1,1)$, and $(0,1)$ respectively. The three distributions share the same covariance matrix, which is a diagonal matrix with value 3 on the diagonal and 0 elsewhere (Subfigure \ref{3class_distributions}). When considering a specific community, sampling the features from a normal distribution implies that values near the mean of the distribution are obtained with higher probability with respect to values in the tails (Subfigure \ref{class1_distribution}).

\begin{figure}[hp]
\begin{subfigure}{0.48\textwidth}
    \includegraphics[width=3.10in]{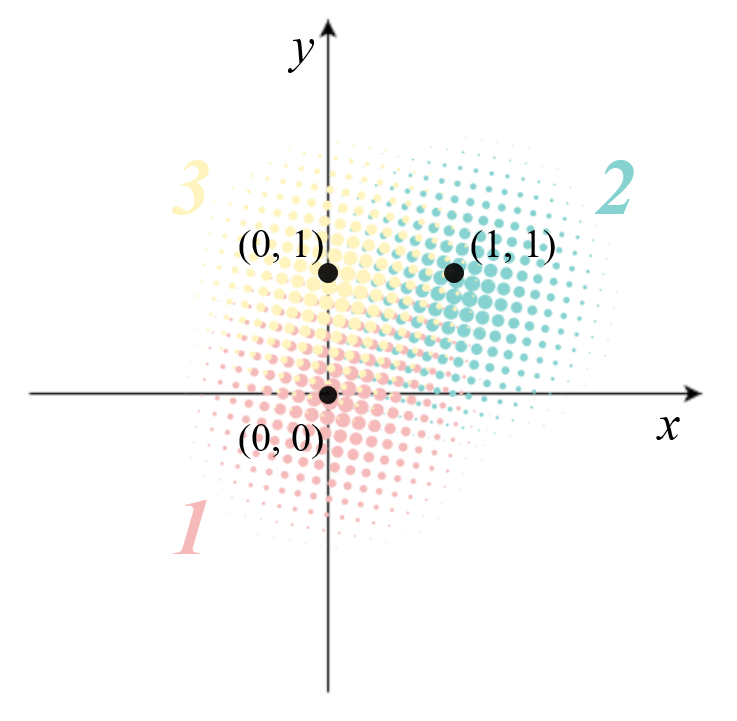}
    \caption{The three distributions are defined in a way to overlap for most of their region, but not completely. Hence, only tailored separation methods should be able to classify the individual nodes to their correct community.}
    \label{3class_distributions}
\end{subfigure}
\hfill
\begin{subfigure}{0.49\textwidth}
    \includegraphics[width=3.20in]{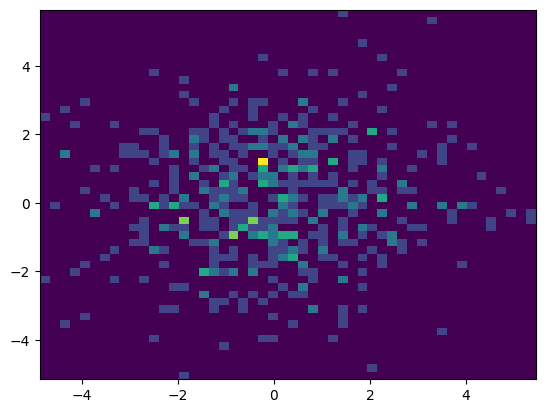}
    \caption{The node features are assigned by sampling from the distribution associated to the class they belong to: the case of the first community is represented here, in which brighter colors imply a higher frequency in sampling that value (that also reflects the shape of the distribution).}
    \label{class1_distribution}
\end{subfigure}
\hfill
\begin{subfigure}{0.49\textwidth}
    \includegraphics[width=3.50in]{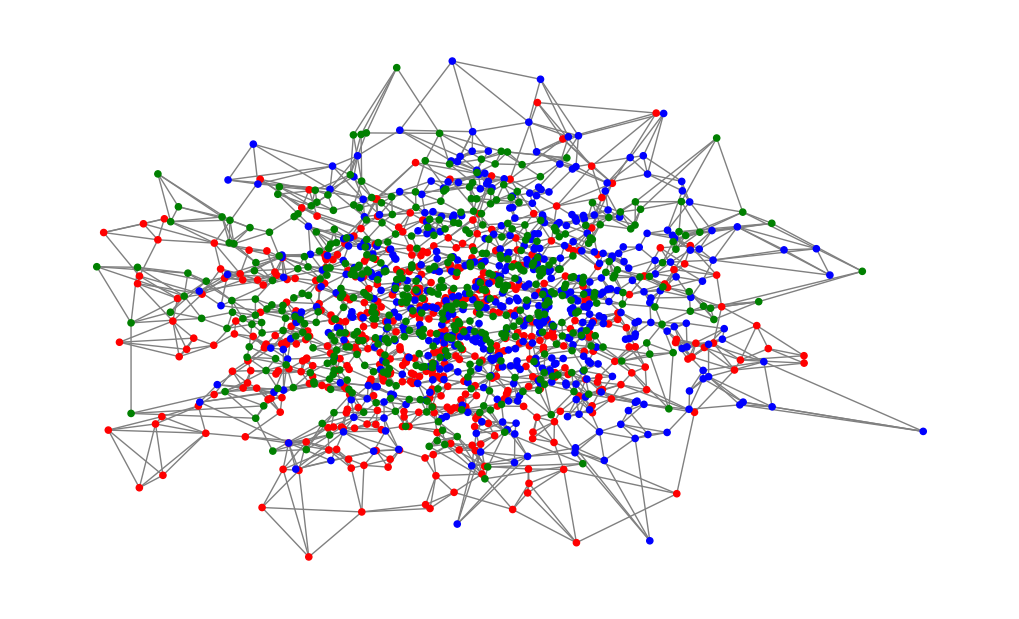}
    \caption{The initial graph, made up of three connected components associated to the three different communities, is plotted by assigning positional values to nodes, that exactly correspond to their features. Although nodes belonging to a specific class are predominant in some areas, they overlap in most of the space where they are distributed. This makes evident why classification is not trivial when random edges are added between different classes, in a way that relying only on the existence of an edge is not sufficient.}
    \label{3class_graph}
\end{subfigure}
\hfill
\begin{subfigure}{0.49\textwidth}
    \includegraphics[width=3.50in]{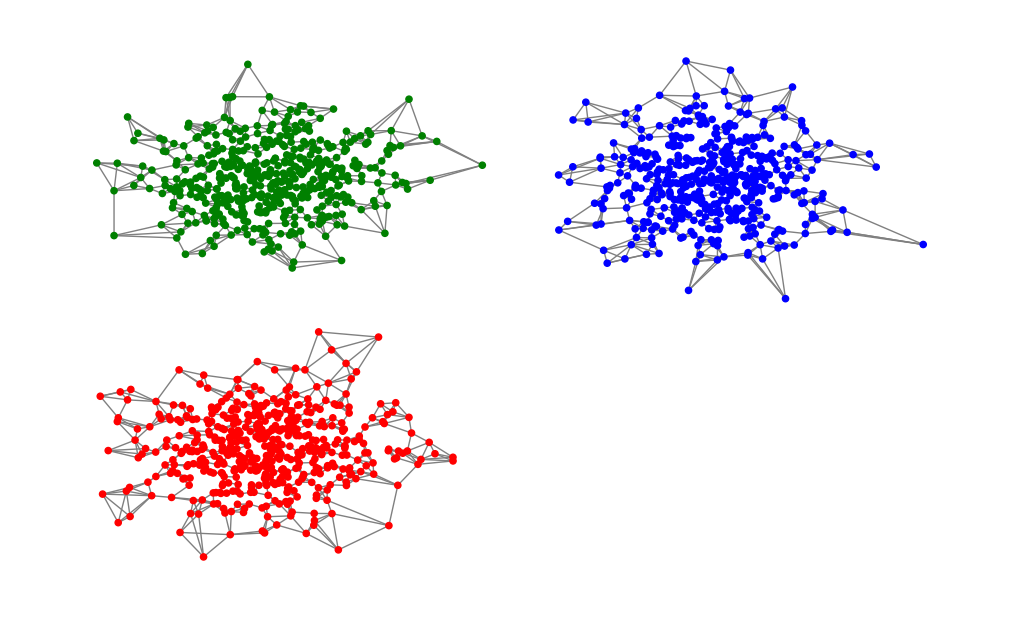}
    \caption{The initial graph, made up of three connected components associated to the three different communities, is plotted by assigning positional values to nodes that correspond to their features, but translated in a way to cancel out the overlap and better visualize the single community graph's layout.}
    \label{3class_graph_departed}
\end{subfigure}
\caption{The node features for the three communities are sampled from three distinct bivariate normal distributions with the same covariance matrix (a diagonal matrix with value 3 on the diagonal) and means equal to (0,0), (0,1) and (1,1) respectively.}
\label{fig:features_distribution}
\end{figure}

\newpage

\subsection{Edge Connections}
Edge connections are initially established using the $k$-NN algorithm within each of the three different communities. This means that at first each node is connected only to the $k$ nodes within its own class that are most similar to it, based on a similarity metric that considers the differences in node features. As a result, the graph is initially divided into three separate connected components, where each node's neighbors are its most similar counterparts.

Moving on to the definition of the dataset sequences, additional edges are randomly inserted into the entire graph, exploring different solutions for the following points:

\begin{itemize}

    \item \textbf{Total number of edges}: one approach to constructing the datasets maintains the initial set of edges generated with $k$-NN while adding the extra random edges (Figure \ref{fig:plots_knn_models_added}). In the other case, the same number of edges added to the network is removed in order to keep the total amount of edges in the graph constant (Figure \ref{fig:plots_knn_models_constant}).

    \item \textbf{Homophily of the edges}: when inserting the additional random edges into the graph, three different scenarios were considered, as it can be seen in both Figure \ref{fig:plots_knn_models_added} and \ref{fig:plots_knn_models_constant}. These involve adding edges either exclusively within a single community (\textit{intra-only random edges}), only between nodes belonging to different communities (\textit{inter-only random edges}), or both within individual communities and between different communities (\textit{inter+intra random edges}).
    
\end{itemize}

\begin{figure}[hp]
\centering
\begin{subfigure}{1\textwidth}
    \includegraphics[width=6.50in]{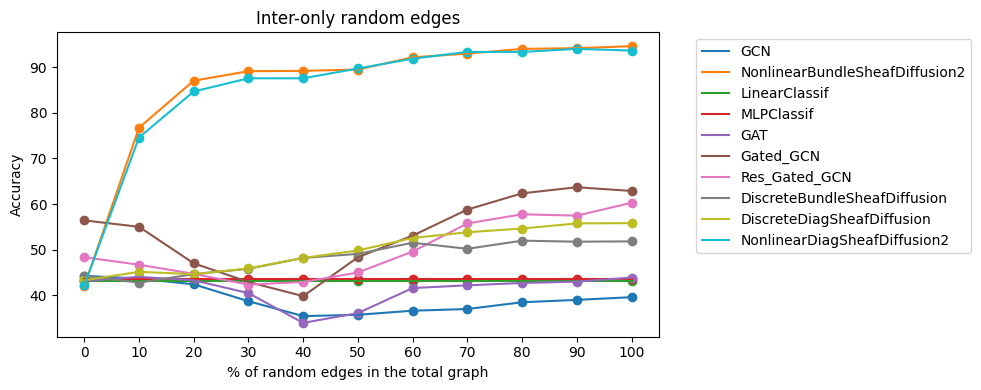}
    \label{constant_inter}
\end{subfigure}
\\
\bigskip
\centering
\begin{subfigure}{1\textwidth}
    \includegraphics[width=6.50in]{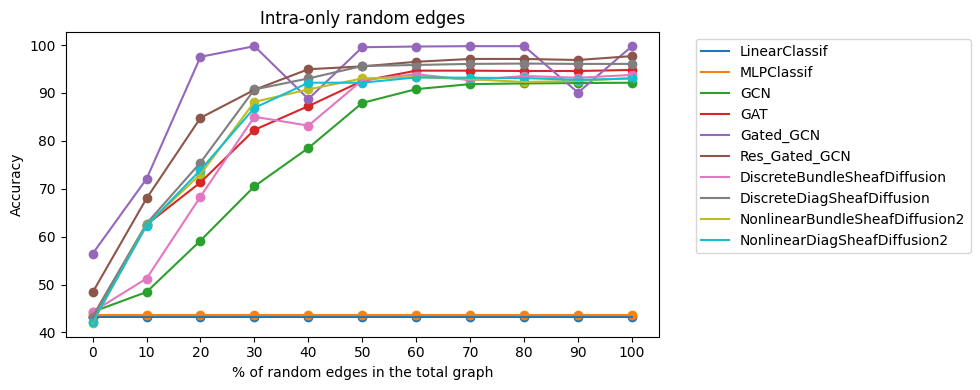}
    \label{constant_intra}
\end{subfigure}
\\
\bigskip
\begin{subfigure}{1\textwidth}
    \includegraphics[width=6.50in]{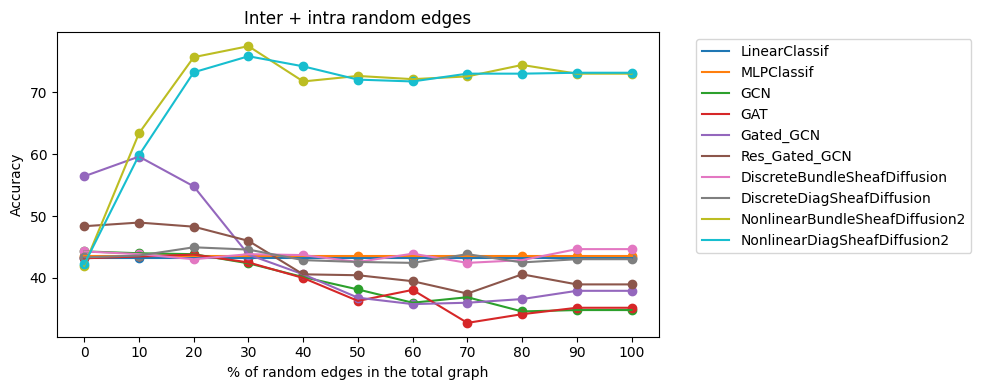}
    \label{constant_inter_intra}
\end{subfigure}
\bigskip
\caption{Accuracy results obtained when random edges are added to the initial graph configuration while removing the same amount of original edges, keeping the total amount of edges in the graph \textbf{constant}. The three plots showcase the outcomes that arise when only inter-class random edges are added (\textbf{upper} plot), only intra-class edges (\textbf{middle} plot) or both inter-class and intra-class edges (\textbf{lower} plot).}
\label{fig:plots_knn_models_constant}
\end{figure}

\begin{figure}[hp]
\centering
\begin{subfigure}{1\textwidth}
    \includegraphics[width=6.50in]{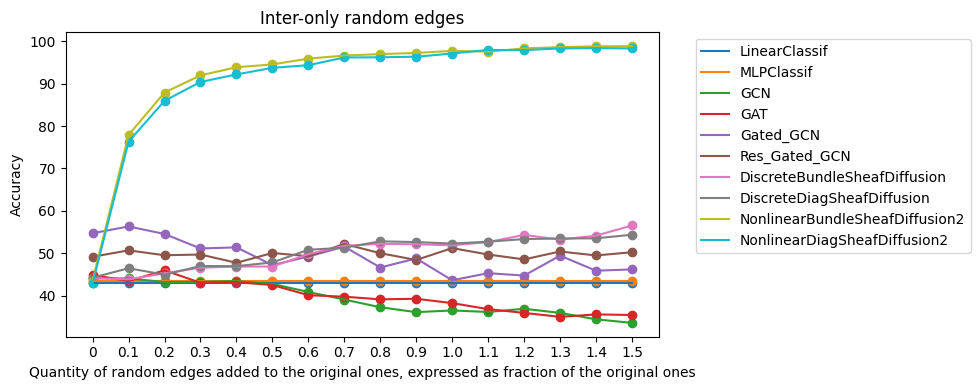}
    \label{added_inter}
\end{subfigure}
\\
\bigskip
\centering
\begin{subfigure}{1\textwidth}
    \includegraphics[width=6.50in]{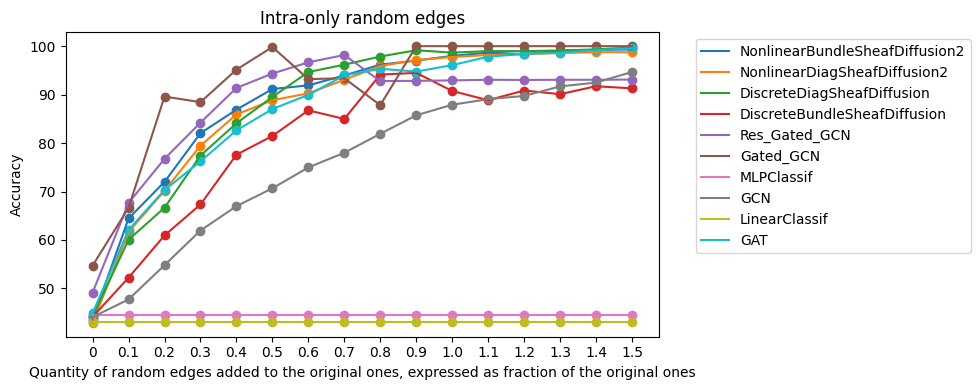}
    \label{added_intra}
\end{subfigure}
\\
\bigskip
\begin{subfigure}{1\textwidth}
    \includegraphics[width=6.50in]{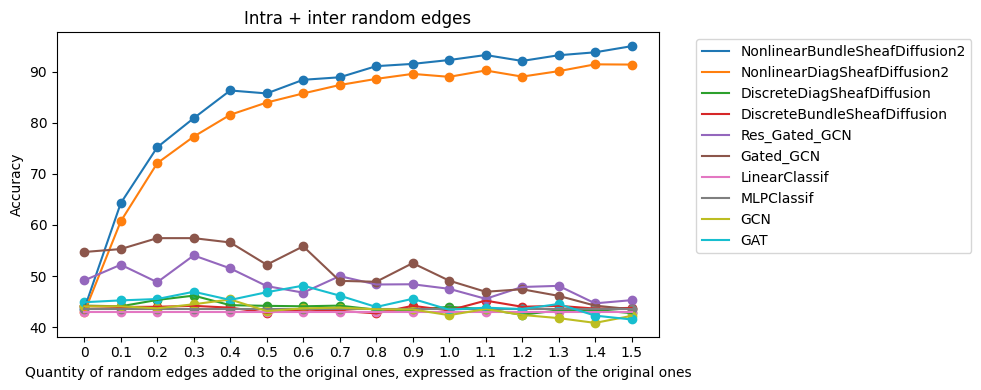}
    \label{added_inter_intra}
\end{subfigure}
\bigskip
\caption{Accuracy results obtained when random edges are added to the initial graph configuration, without removing any, leading to a progressively \textbf{increasing} total amount of edges in the graph. The three plots showcase the outcomes that arise when only inter-class random edges are added (\textbf{upper} plot), only intra-class edges (\textbf{middle} plot) or both inter-class and intra-class edges (\textbf{lower} plot).}
\label{fig:plots_knn_models_added}
\end{figure}

\paragraph{Results}

The results reveal a consistent pattern in both scenarios depicted in Figures \ref{fig:plots_knn_models_added} and \ref{fig:plots_knn_models_constant}. Across all models, it is evident from the plots that relying solely on the initial edge configuration generated through the $k$-NN algorithm makes it challenging to achieve accurate classifications. This difficulty arises because $k$-NN edges lack sufficient information for discrimination, as they only link similar nodes, leading to reduced classification accuracy for nodes with features sampled from the overlapping region of the distributions. 

However, when we introduce additional random edges within individual communities, we observe the emergence of shorter paths between areas of the distributions that are further apart. These can be effectively leveraged by all models, as demonstrated in the plots for the \textit{intra-only random edges} case.

A distinct behavior is observed when random edges are inserted not only within individual communities but also between different communities, or exclusively in the second way. The plots for the \textit{inter-only} and \textit{inter + intra random edges} scenarios indicate that the NLSD model is the only one capable of capitalizing on these edges effectively. As expected, the lowest accuracy values are achieved by all architectures when random edges are both inter- and intra-class, which lacks consistency in edge types, and when the initial $k$-NN edges are removed in parallel, which exacerbates the instability of the model compared to the scenario where all initial edges are retained.

\newpage

\subsection{Further Analysis}
In order to gain further insights into the factors driving the  accuracy results observed in the previous section, we conducted additional experiments. 

\paragraph{Comparison of classification patterns}

Our initial analysis involved plotting the configuration of correctly and incorrectly classified nodes within the three communities of the graph. This analysis considered the scenario in which only random inter-class edges are added while original edges are removed. The underlying idea was to assess how effectively our model could leverage these connections to enhance classification accuracy, as compared to other models that struggled to extract meaningful information from such edges. The results actually show that these models heavily rely on node features alone, and they often fall short of achieving satisfactory classification results.

Figure \ref{fig:final_correct_wrong} presents a comparison between the behavior of the Linear and Nonlinear orthogonal sheaf in this scenario, in the extreme setting in which 100\% of the edges in the graph are random inter-class edges. This case corresponds to the upper diagram in Figure \ref{fig:plots_knn_models_constant}, specifically for a percentage value of 100.
In the first Subfigure (\ref{nonlinear_redblue}), we observe the Nonlinear bundle sheaf case, where the number of incorrectly classified nodes is minimal, as expected from the high accuracy (Figure \ref{fig:plots_knn_models_constant}, upper plot). Additionally, these misclassified nodes are distributed across classes independently of the normal distributions used to sample node features.
On the other hand, the second Subfigure (\ref{linear_redblue}) illustrates the Linear bundle sheaf case, where the classification heavily depends on the node features themselves. The wrongly classified nodes primarily belong to the overlapping region of the distributions, which is the most challenging area for classification when based solely on node features.

\begin{figure}[hp]
\centering
\begin{subfigure}{0.8\textwidth}
    \includegraphics[width=5.30in]{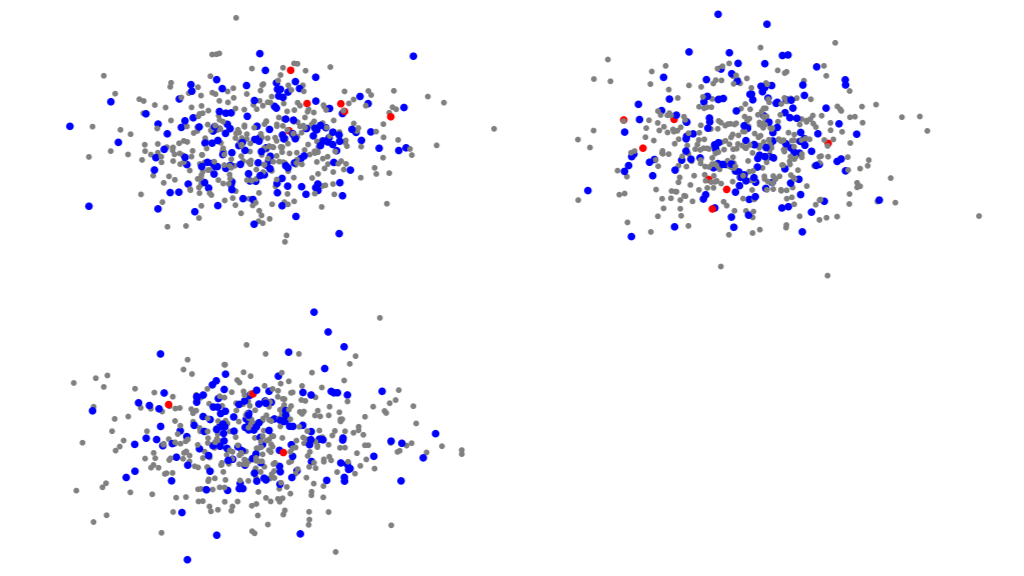}
    \caption{Nonlinear O(d) sheaf}\label{nonlinear_redblue}
\end{subfigure}
\\
\begin{subfigure}{0.8\textwidth}
    \includegraphics[width=5.30in]{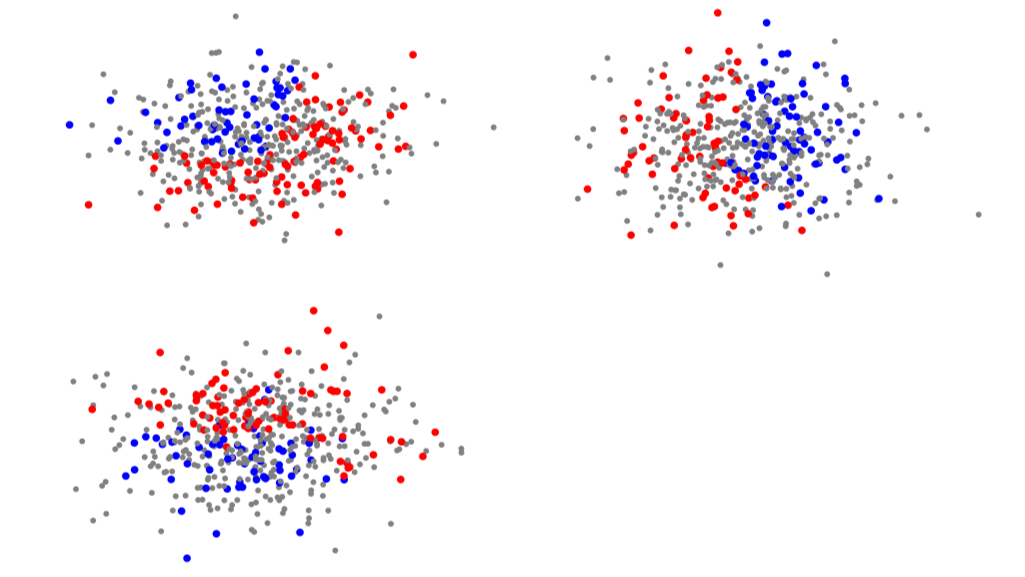}
    \caption{Linear O(d) sheaf}\label{linear_redblue}
\end{subfigure}
\caption{Comparison between Linear and Nonlinear bundle sheaf predictions when 100\% of the edges in the graph are random inter-class edges (Figure 5.2, upper plot). Correctly classified nodes highlighted in \textbf{\textcolor{blue}{blue}} and incorrectly classified nodes are highlighted in \textbf{\textcolor{red}{red}}.}
\label{fig:final_correct_wrong}
\end{figure}

\newpage

\paragraph{Single-layer analysis}

Further studies were conducted in the setting in which each architecture is composed by one single layer. This specific setup was chosen to facilitate a clearer understanding of how the configuration of incident edges to a node influences its classification outcome. With only one layer of message passing performed before making the classification decision, the impact of edge connectivity becomes more readily apparent. By examining the connectivity patterns of nodes leading to particular prediction patterns, we aimed to uncover any underlying relationships between the network structure and the observed classification outcomes.

Notably, even when considering single-layer models, the accuracy plot does not deviate significantly from the scenario involving three layers for each model, as shown in Figure \ref{fig:results_singlelayer}. In fact, our models exhibit a comparatively minor decrease in performance compared to the competitors.

\begin{figure}[htp]
\centering
\centerline{\includegraphics[scale=0.65]{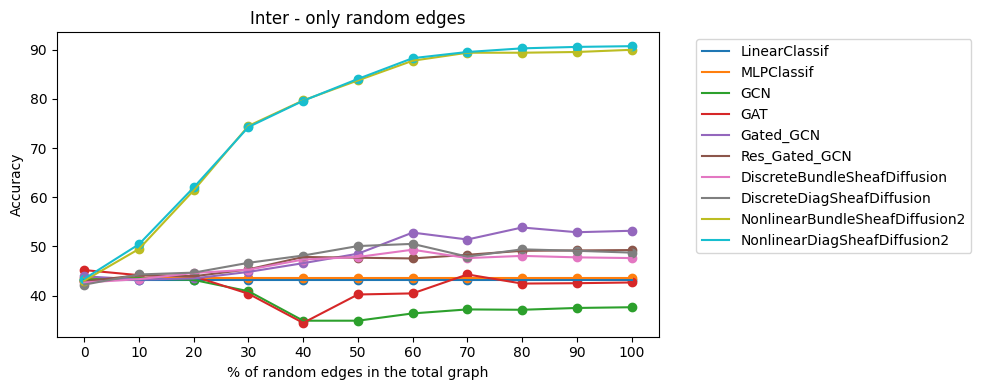}}
 \caption{Plots of accuracies achieved by the models on a dataset sequence constructed using $k$-NN connections as a starting point, and subsequently incorporating random inter-class edges while simultaneously removing the same amount of intra-class edges randomly. The distinctive aspect of this experiment is that each tested architecture is characterized by a \textbf{single layer}. }
 \label{fig:results_singlelayer}
\end{figure}

\paragraph{Comparison of connectivity patterns for \textit{wrong-first} and \textit{wrong-always} nodes}

Multiple analyses were carried out on the incident edges to two particular sets of nodes, which are:

\begin{itemize}
    \item \textit{Wrong-first} nodes, namely the ones that are classified incorrectly when the graph with the original edge set is considered, while they're classified correctly for every non-zero percentage value of random edges added to the graph (and the same amount of original edges being removed). In the upper plot of  Figure \ref{fig:plots_knn_models_constant}, these are nodes whose prediction is wrong for percentage value 0\%, and correct for all other values. Ideally, these should correspond to the nodes whose initial location in the feature distribution is not informative enough to classify them correctly, hence their node features are most likely sampled from a region of the feature space corresponding with an overlap of the distributions. Adding random edges, though, builds useful connections for such nodes, that should be leveraged to correctly learn how to classify them. 

    \item \textit{Wrong-always} nodes: these are the nodes that are wrongly classified for all percentage values of randomly added edges to the original graph (while the same amount is removed from the initial ones). These ideally correspond to nodes that both have node features not informative enough to correctly classify them at the start, and for which the addition of random edges doesn't provide useful incident edges for the classification.  
\end{itemize}

Studying the pattern of incident edges to these types of nodes for an increasing percentage of additional random edges helps in understanding how the alternation of the links in the graph affects the performance of the models. These analyses were carried out by defining one single layer in the network, as this ensures that only the immediate neighbors of each node are leveraged to solve the task. 

The initial analysis, conducted multiple times, involved sampling one node each from the sets of \textit{wrong-always} and \textit{wrong-first} nodes. We qualitatively compared their connectivity patterns by plotting their incident edges in different colors. The results are illustrated in Figures \ref{fig:1layer_linear}, \ref{fig:1layer_nonlinear_1}, and \ref{fig:1layer_nonlinear_3}. In these plots, each node is positioned based on its node feature, potentially translated depending on its community for visual clarity. This avoids overlap and provides a clearer overview of the edge pattern.

As a complementary analysis to validate the first one, the second study calculated the average number of incident edges for \textit{wrong-always} and \textit{wrong-first} nodes for different percentages of random edges in the graph dataset. The corresponding results are presented in Table \ref{tab:wwongalways_wrongfirst}

The considerations derived from the two investigations can be summarized as follows for the case of the Linear and Nonlinear Laplacians:

\begin{itemize}
    \item  \textbf{Linear}:
    after running several experiments, of which an example is shown in Figure \ref{fig:1layer_linear}, we observed that discerning relevant and consistent patterns for the edge configurations of \textit{wrong-first} and \textit{wrong-always} nodes, that would allow to discriminate between them, is not straightforward. This primarily involves the arrangement of links—the number of connections to  other communities and their distribution within those communities. The same reasoning applies to  the average number of edges incident to the two classes of nodes: it is evident from Table \ref{tab:wwongalways_wrongfirst} that there is no correlation between the amount of incident edges to a node and how it is classified by the model, even when considering increasing percentages of added random edges.
    \item \textbf{Nonlinear}: the plots in Figures \ref{fig:1layer_nonlinear_1} and \ref{fig:1layer_nonlinear_3} reveal distinctive patterns in the misclassification of nodes based on certain key observations. \textit{Wrong-first} nodes exhibit a higher volume of incident edges (see also Table \ref{tab:wwongalways_wrongfirst}). Notably, these nodes tend to possess multiple edges linking them to one or two other communities. Interestingly, when connected to only one other community, the edges are usually highly informative as they link to nodes that do not belong to the overlapping area.
Conversely, \textit{wrong-always} nodes display different characteristics. These nodes generally exhibit fewer connections in comparison  as the percentage increases  (see Table \ref{tab:wwongalways_wrongfirst}) and are often linked to 0 or only 1 other community (see Figures   \ref{fig:1layer_nonlinear_1} and \ref{fig:1layer_nonlinear_3}). The edges associated with these nodes prove to be less helpful for discrimination, especially within areas of overlapping distribution, highlighting the complexity of their classification.
\end{itemize}

In essence, the key point is that Nonlinear models showcase performance influenced by the quantity and types of node connections. These models can effectively utilize newly added links, relying on connectivity for classification rather than solely on node features. In contrast, Linear models demonstrate a behavior that remains unaffected by changes in connectivity introduced by random edges. This emphasizes their limited ability to leverage connectivity effectively.

\begin{table}[htbp]
  \centering
  \caption{Average number of incident edges to the set of \textit{wrong-first} nodes, that is the set of nodes that are wrongly classified in the 0\% case, and correctly classified in all other cases, and to the set of \textit{wrong-always} nodes, namely the ones that are wrongly classified for all percentages of random inter-class edges added to the graph.}
  \label{tab:wwongalways_wrongfirst}
  \begin{tabular}{c|cccc}
     & \multicolumn{2}{c}{\textbf{Wrong-first}} & \multicolumn{2}{c}{\textbf{Wrong-always}} \\ \hline
    \textbf{\% random edges} & \textbf{Linear} & \textbf{Nonlinear} & \textbf{Linear} & \textbf{Nonlinear} \\ \hline
    0 \% & 4.2 & 5.1 & 5.2 & 5.0 \\ 
    10 \% & 4.0 & 5.5 & 5.2 & 4.6 \\
    20 \% & 3.8 & 5.8 & 6.4 & 4.4 \\
    40 \% & 5.2 & 5.4 & 5.6 & 3.3 \\
    80 \% & 5.4 & 5.5 & 4.8 & 2.6 \\
    100 \% & 5.6 & 5.5 & 4.8 & 2.5 \\
  \end{tabular}
\end{table}

\begin{figure}[hp]
\centering
\begin{subfigure}{0.49\textwidth}
    \includegraphics[width=3.50in]{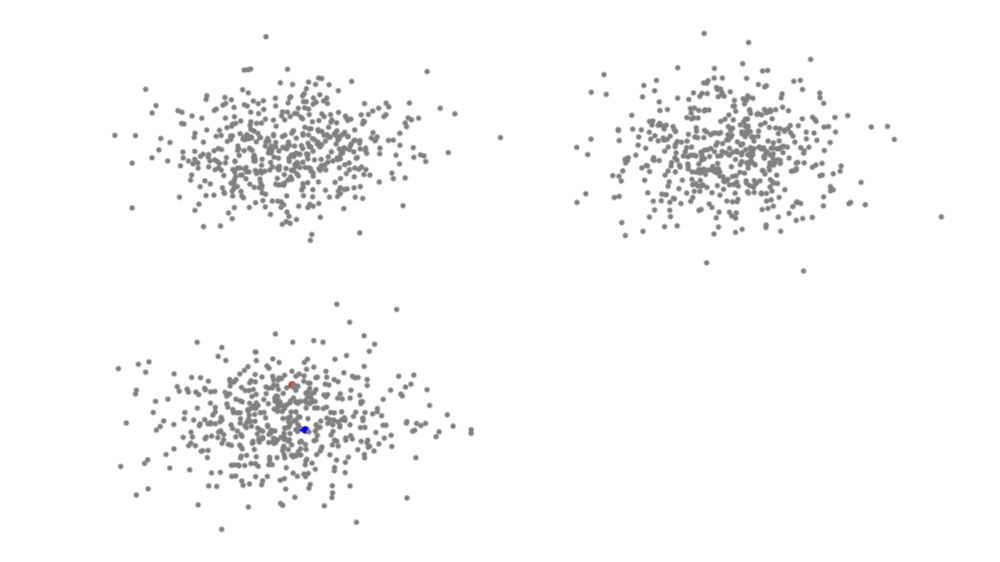}
    \caption{0\% of random edges}
    \label{0_lin}
\end{subfigure}
\centering
\begin{subfigure}{0.49\textwidth}
    \includegraphics[width=3.50in]{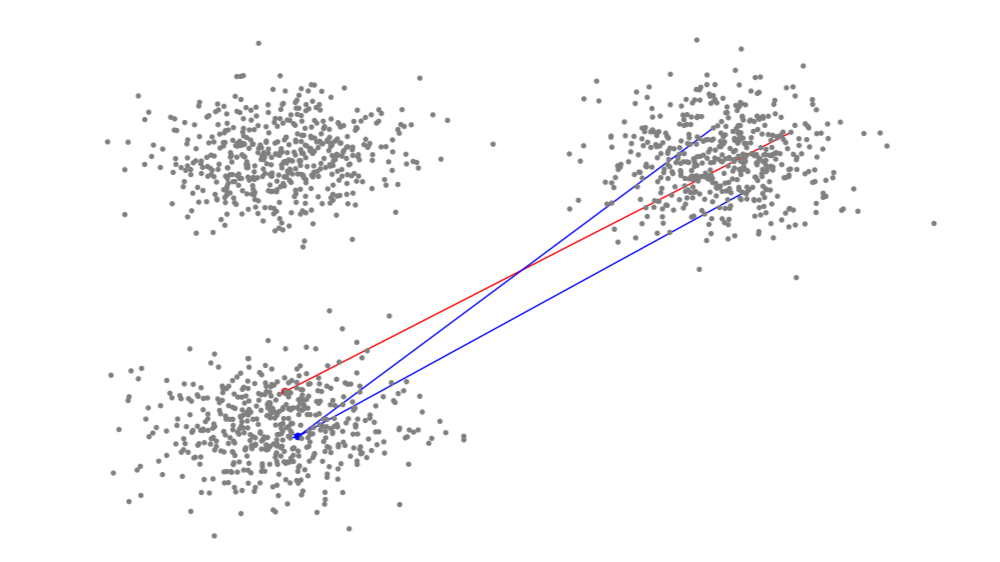}
    \caption{10\% of random edges}
    \label{10_lin}
\end{subfigure}
\\
\bigskip
\begin{subfigure}{0.49\textwidth}
    \includegraphics[width=3.50in]{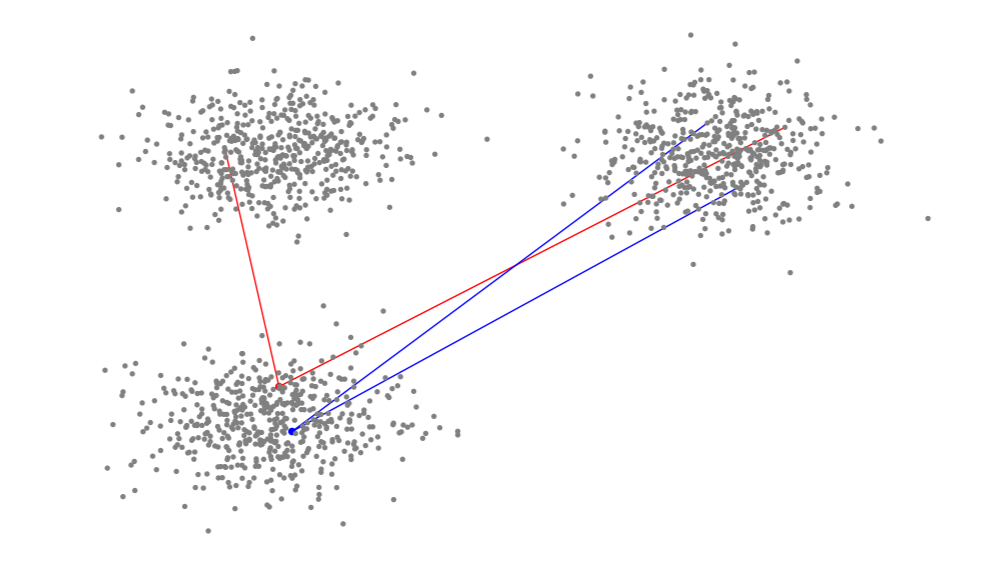}
    \caption{20\% of random edges}
    \label{20_lin}
\end{subfigure}
\begin{subfigure}{0.49\textwidth}
    \includegraphics[width=3.50in]{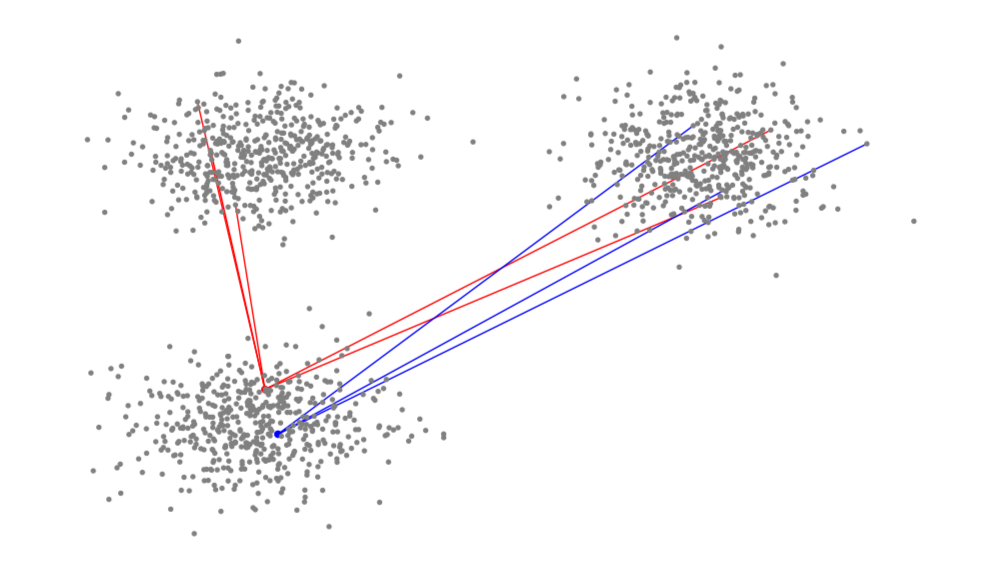}
    \caption{40\% of random edges}
    \label{40_lin}
\end{subfigure}
\\
\bigskip
\centering
\begin{subfigure}{0.49\textwidth}
    \includegraphics[width=3.50in]{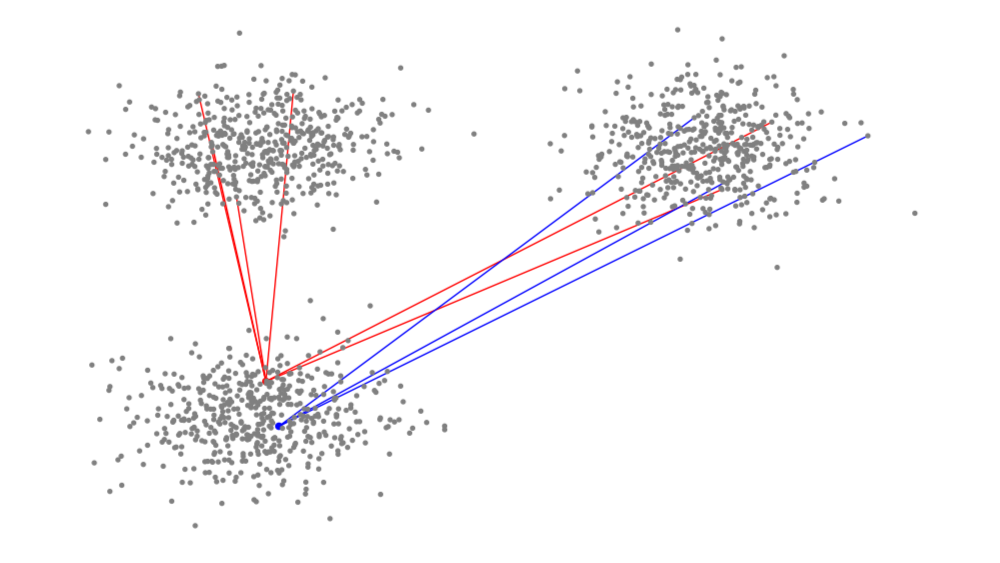}
    \caption{80\% of random edges}
    \label{80_lin}
\end{subfigure}
\begin{subfigure}{0.49\textwidth}
    \includegraphics[width=3.50in]{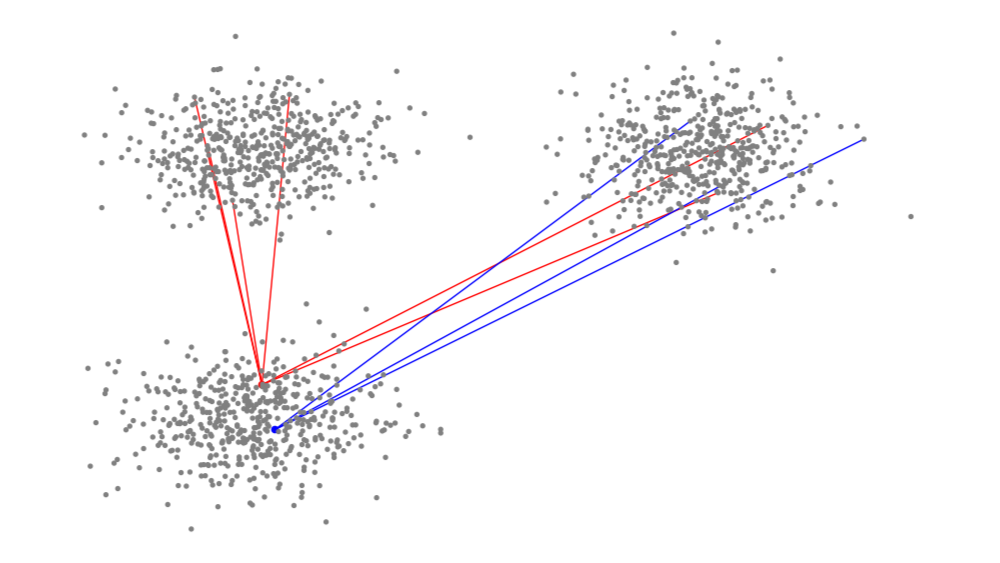}
    \caption{100\% of random edges}
    \label{100_lin}
\end{subfigure}
\bigskip
\caption{A 3-layer \textbf{Linear} SNN with orthogonal restriction maps is tested in the setting in which random inter-class edges are progressively added to the graph, while the same amount of original intra-class edges is removed. One node, marked in \textbf{\textcolor{red}{red}}, is randomly sampled from the set of nodes that are consistently misclassified across all percentages of inter-class edges added to the graph. Another node, marked in \textbf{\textcolor{blue}{blue}}, is sampled from the nodes that are misclassified when no random edges are present but correctly classified in all subsequent cases within the dataset sequence. Most notably, the incident edges of these selected nodes are color-coded to match their respective node color.}
\label{fig:1layer_linear}
\end{figure}

\begin{figure}[hp]
\centering
\begin{subfigure}{0.49\textwidth}
    \includegraphics[width=3.50in]{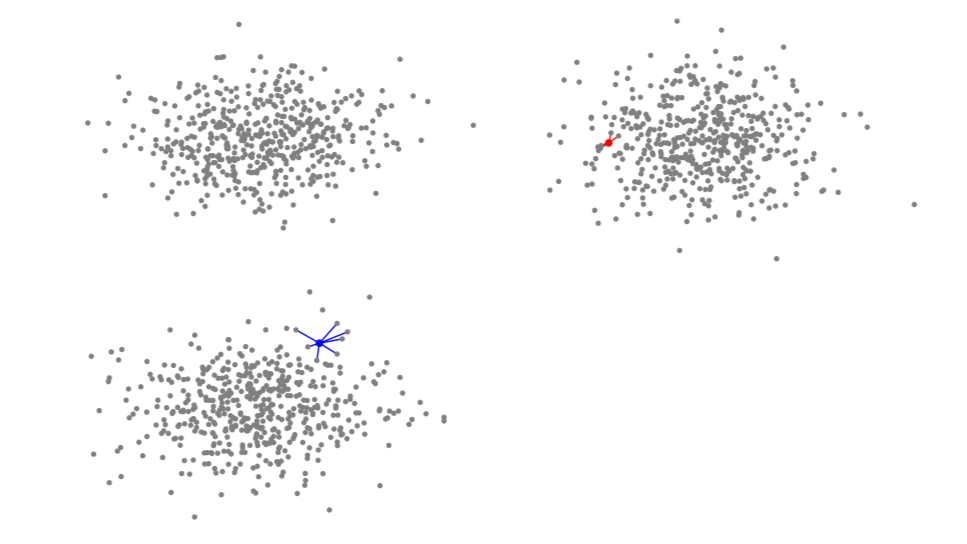}
    \caption{0\% of random edges}
\end{subfigure}
\centering
\begin{subfigure}{0.49\textwidth}
    \includegraphics[width=3.50in]{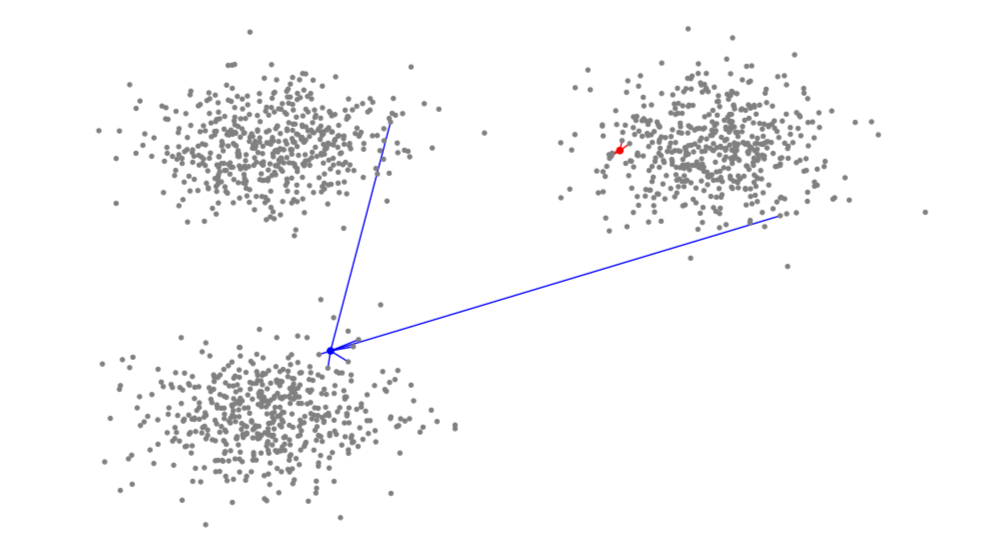}
    \caption{10\% of random edges}
\end{subfigure}
\\
\bigskip
\begin{subfigure}{0.49\textwidth}
    \includegraphics[width=3.50in]{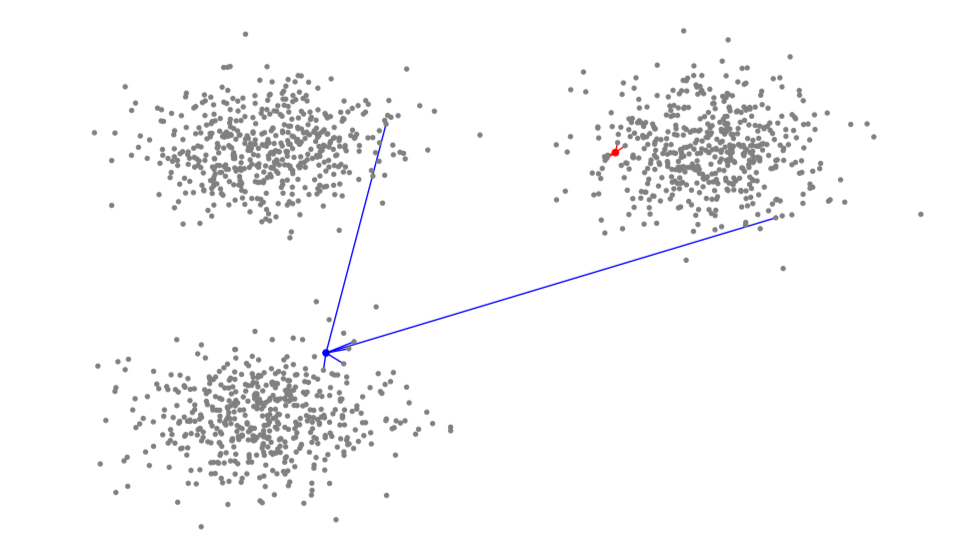}
    \caption{20\% of random edges}
\end{subfigure}
\begin{subfigure}{0.49\textwidth}
    \includegraphics[width=3.50in]{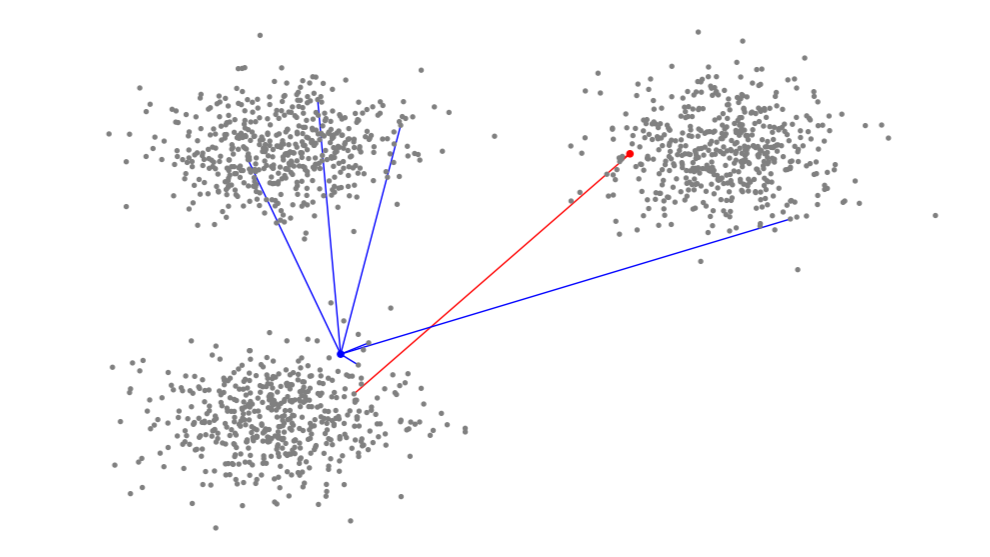}
    \caption{40\% of random edges}
    \label{40_nonlin_1}
\end{subfigure}
\\
\bigskip
\centering
\begin{subfigure}{0.49\textwidth}
    \includegraphics[width=3.50in]{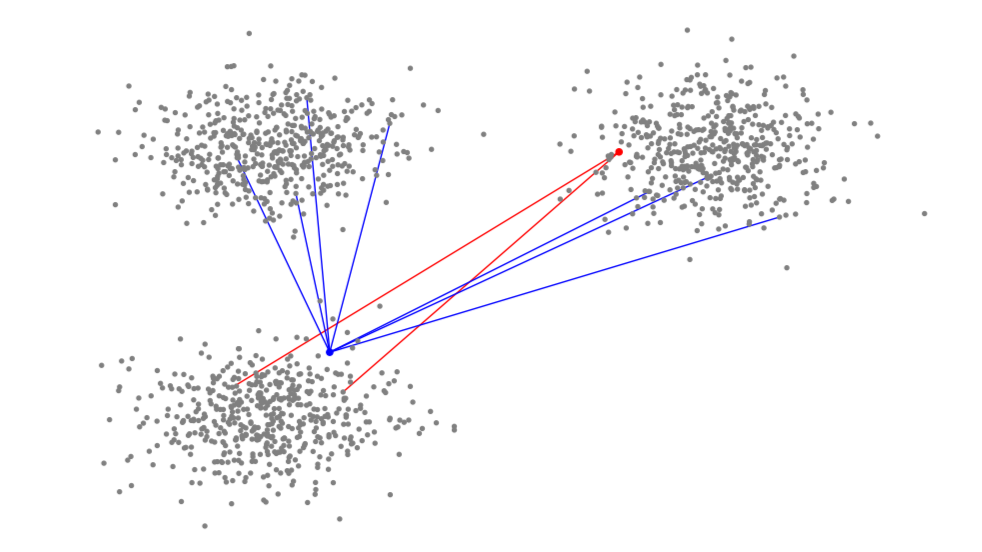}
    \caption{80\% of random edges}
    \label{80_nonlin_1}
\end{subfigure}
\begin{subfigure}{0.49\textwidth}
    \includegraphics[width=3.50in]{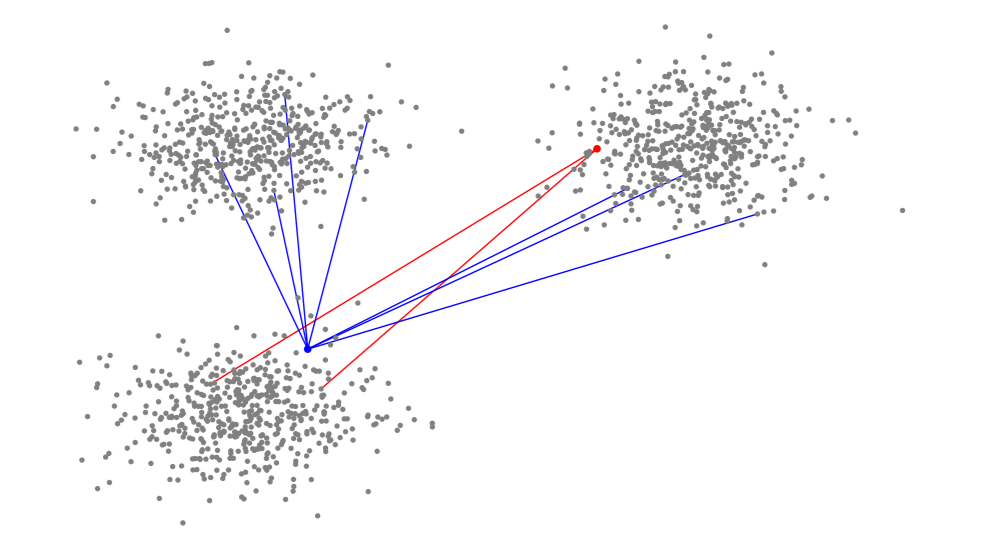}
    \caption{100\% of random edges}
    \label{100_nonlin_1}
\end{subfigure}
\bigskip
\caption{A 3-layer \textbf{Nonlinear} SNN with orthogonal restriction maps is tested in the setting in which random inter-class edges are progressively added to the graph, while the same amount of original intra-class edges is removed. One node, marked in \textbf{\textcolor{red}{red}}, is randomly sampled from the set of nodes that are consistently misclassified across all percentages of inter-class edges added to the graph. Another node, marked in \textbf{\textcolor{blue}{blue}}, is sampled from the nodes that are misclassified when no random edges are present but correctly classified in all subsequent cases within the dataset sequence. Most notably, the incident edges of these selected nodes are color-coded to match their respective node color.}
\label{fig:1layer_nonlinear_1}
\end{figure}

\begin{figure}[hp]
\centering
\begin{subfigure}{0.49\textwidth}
    \includegraphics[width=3.50in]{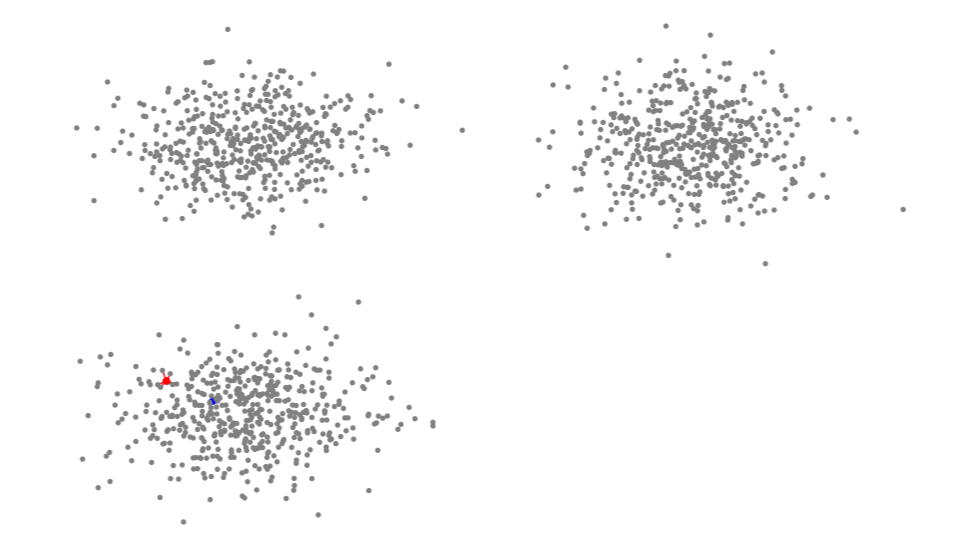}
    \caption{0\% of random edges}
\end{subfigure}
\centering
\begin{subfigure}{0.49\textwidth}
    \includegraphics[width=3.50in]{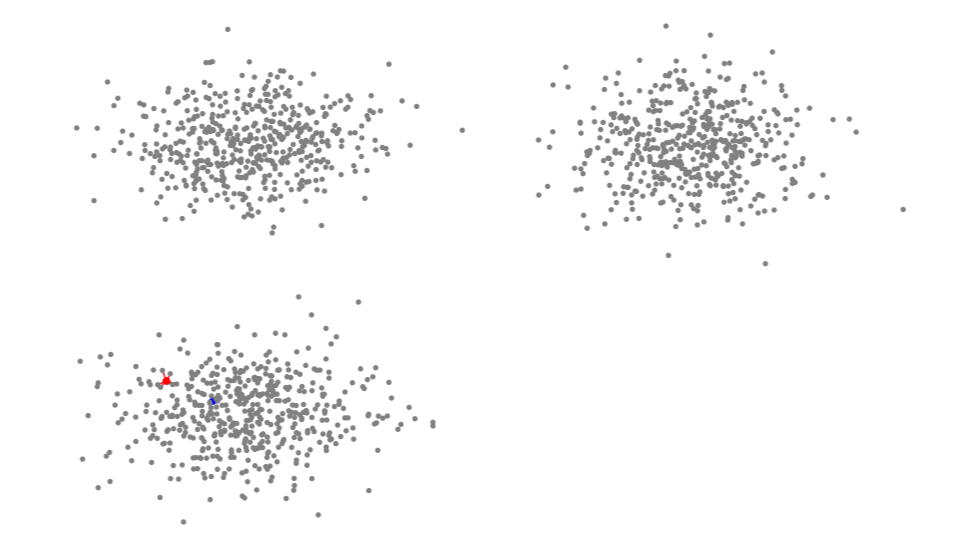}
    \caption{10\% of random edges}
\end{subfigure}
\\
\bigskip
\begin{subfigure}{0.49\textwidth}
    \includegraphics[width=3.50in]{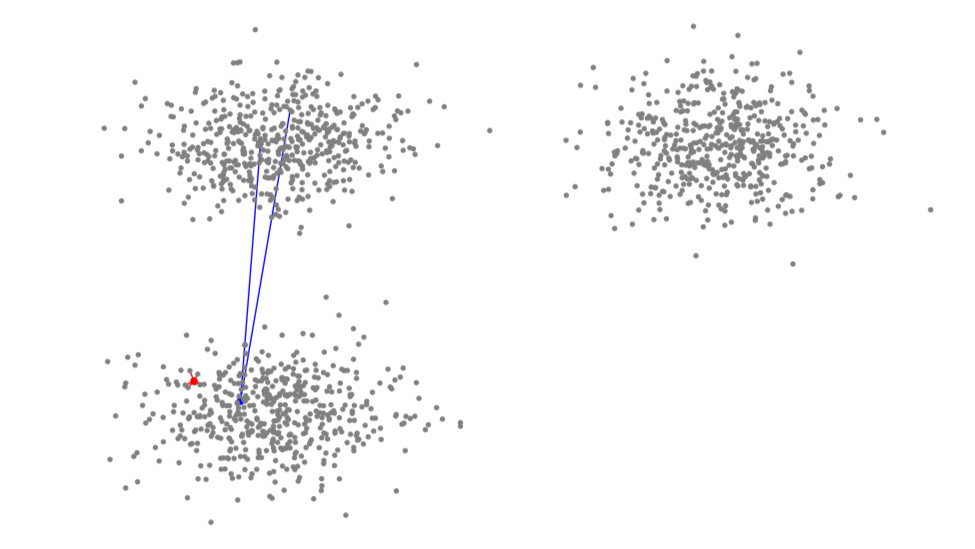}
    \caption{20\% of random edges}
\end{subfigure}
\begin{subfigure}{0.49\textwidth}
    \includegraphics[width=3.50in]{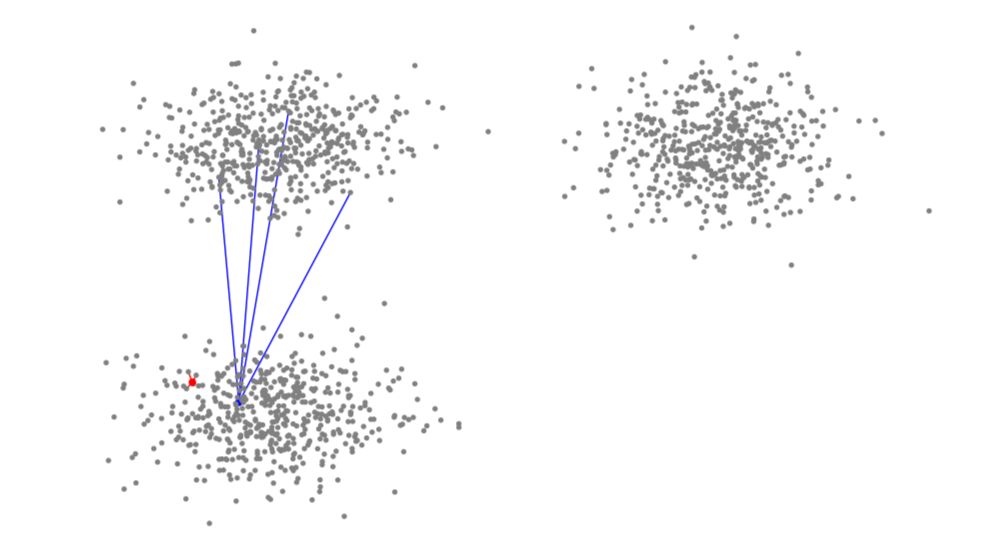}
    \caption{40\% of random edges}
\end{subfigure}
\\
\bigskip
\centering
\begin{subfigure}{0.49\textwidth}
    \includegraphics[width=3.50in]{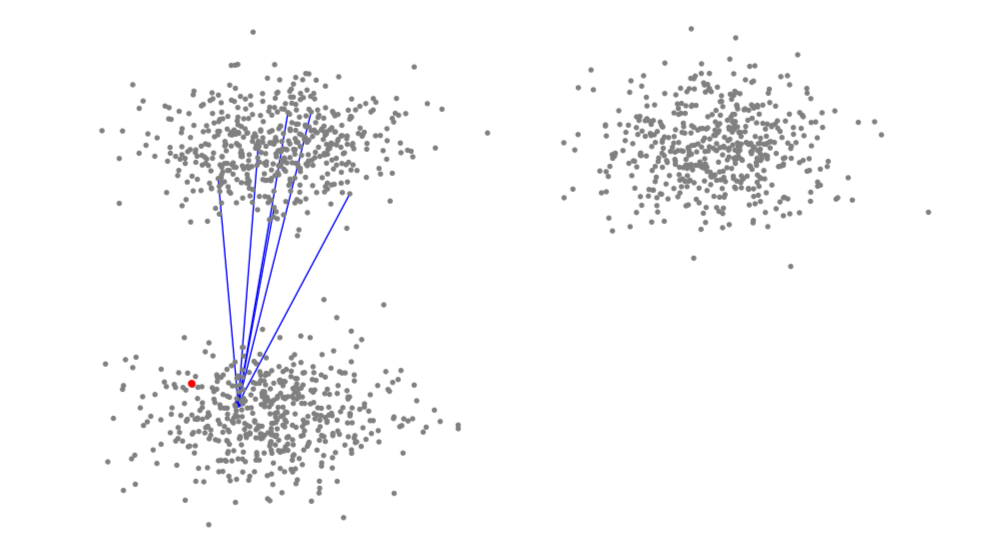}
    \caption{80\% of random edges}
\end{subfigure}
\begin{subfigure}{0.49\textwidth}
    \includegraphics[width=3.50in]{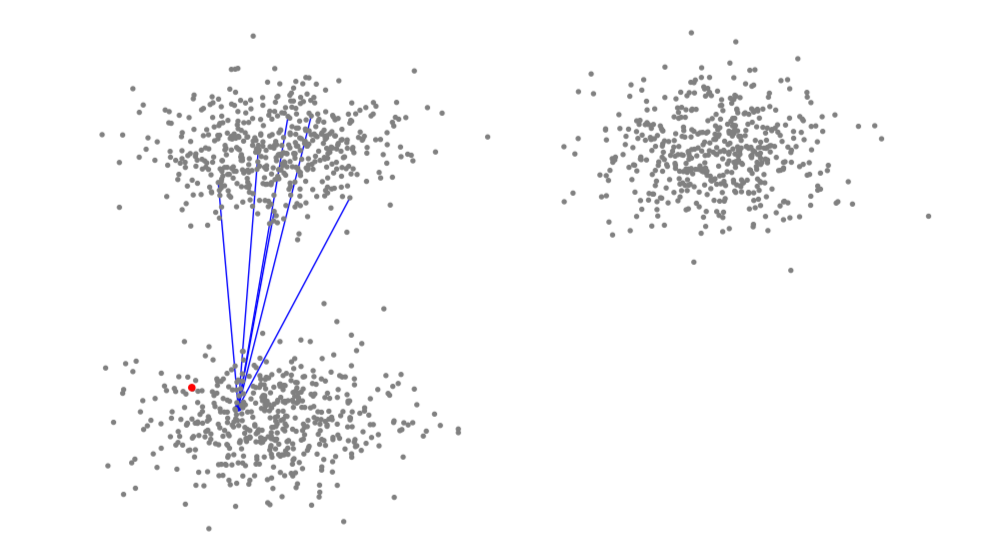}
    \caption{100\% of random edges}
\end{subfigure}
\bigskip
\caption{A 3-layer \textbf{Nonlinear} SNN with orthogonal restriction maps is tested in the setting in which random inter-class edges are progressively added to the graph, while the same amount of original intra-class edges is removed. One node, marked in \textbf{\textcolor{red}{red}}, is randomly sampled from the set of nodes that are consistently misclassified across all percentages of inter-class edges added to the graph. Another node, marked in \textbf{\textcolor{blue}{blue}}, is sampled from the nodes that are misclassified when no random edges are present but correctly classified in all subsequent cases within the dataset sequence. Most notably, the incident edges of these selected nodes are color-coded to match their respective node color.}
\label{fig:1layer_nonlinear_3}
\end{figure}

\newpage

\section{Real-World Experiments}\label{sec:real_world}

We evaluate 6 different variants of the model proposed in this project on several homophilic and heterophilic real-world datasets \cite{pei2020geom,namata2012query,rozemberczki2021multi,sen2008collective,tang2009social}, comparing their
performance against several other models from the graph representation learning literature, and we report the results in Table \ref{tab:main_results}. 
The real-world datasets on which the models were tested are characterized by an edge homophily coefficient $h$ ranging from $h = 0.11$ (very heterophilic) to $h=0.81$ (very homophilic). The results are collected over 10
fixed splits by performing 10-fold Cross Validation, where 48\%, 32\%, and 20\% of nodes per class are used for training, validation and test, respectively. The results shown in \ref{tab:main_results} are selected by taking the test accuracy corresponding to the highest validation accuracy.

Similarly to NSD \cite{bodnar2022neural}, our model comes with a few variations depending on the type of restriction
maps that it learns. The only difference is that we just consider diagonal and orthogonal matrices for our case, and not general-structure matrices. The reason is that despite their flexibility, general matrices are more challenging to normalize, they make the model more difficult to train, and the higher amount of free parameters that they bring into play leads to the risk of overfitting, as it was pointed out by Bodnar et al. in their study \cite{bodnar2022neural}. The normalization issue in particular poses a relevant problem in our framework, as discussed in the previous sections. Furthermore, the experimental results in NSD and its subsequent variants \cite{bodnar2022neural, barbero2022sheaf, suk2022surfing} show that the general one also constitutes the worst performing type of matrix in terms of accuracy.

The benchmark models to which the performance is compared are exactly the same ones considered in the analysis carried out by Bodnar et al. for the NSD \cite{bodnar2022neural} method. Some of them are standard benchmarks widely used for validation of new architectures for most GNN tasks, and they are described and analyzed in Section \ref{sec:sec_gnn}.
Additionally, the best results obtained by NSD \cite{bodnar2022neural} are reported, as well as the ones by NSP \cite{suk2022surfing} and SheafAN \cite{barbero2022sheaf}, two recently proposed variants of the original sheaf diffusion framework that are introduced in Section \ref{sec:section_snn}.

For the experiments on the  real-world datasets we considered the following versions of our model:

\begin{itemize}
    \item \textbf{BC-s-Diag-NLSD}: Bounded Confidence model, with a single threshold value learnt for all pairs of neighboring nodes in the graph, diagonal restriction maps and normalization through diagonal matrix;
    \item \textbf{BC-s-O(d)-NLSD}: Bounded Confidence model, with a single threshold value learnt for all pairs of neighboring nodes in the graph, orthogonal   restriction maps and normalization through diagonal matrix;
    \item \textbf{BC-m-Diag-NLSD}: Bounded Confidence model, with multiple threshold values for different couples of neighboring nodes obtained as outputs of a MLP, diagonal restriction maps and normalization through diagonal matrix;
    \item \textbf{BC-m-O(d)-NLSD}: Bounded Confidence model, with multiple threshold values for different couples of neighboring nodes obtained as outputs of a MLP, orthogonal restriction maps and normalization through diagonal matrix;
    \item \textbf{MLP-Diag-NLSD}: nonlinearity fully defined by a 2-layer MLP, diagonal restriction maps and normalization through diagonal matrix;
    \item \textbf{MLP-O(d)-NLSD}: nonlinearity fully defined by a 2-layer MLP, orthogonal restriction maps and normalization through diagonal matrix.
    
\end{itemize}

The experiments on the real-world benchmark datasets show that almost all versions of the proposed model achieve comparable results with respect to NSD \cite{bodnar2022neural} and its variants SheafAN\cite{barbero2022sheaf} and NSP \cite{suk2022surfing}. This implies that our architectures maintain the original desirable properties brought in by sheaf diffusion, along with providing additional ones that can be useful in specific scenarios, as displayed in Section \ref{sec:synthetic}. We
demonstrate that also our framework can achieve state-of-the-art results in heterophilic settings, and the variants that exhibit the best performance in accuracy are the Bounded Confidence ones, in particular when orthogonal restriction maps are adopted.

Interestingly, the results suggest that the inclusion of a nonlinearity in the definition of the sheaf Laplacian does not have a substantial negative effect on the performance. While this outcome may be expected when adopting an MLP for nonlinearity, at it allows for greater flexibility, it is instead quite surprising for the implementation of bounded confidence. The reason is that the learned threshold values are generally very small in absolute value, always much less than $1$. This suggests that either the outputs of the coboundary operator are consistently very low or that, when they are extremely high, they can be safely disregarded.

\begin{table*}[t]
    \centering
    \caption{Results on node classification datasets sorted by their homophily level. Top three models are coloured by \textbf{\textcolor{red}{First}}, \textbf{\textcolor{blue}{Second}}, \textbf{\textcolor{violet}{Third}}. Our models are marked {\bf NLSD} (Non Linear Sheaf Diffusion), and are split in two categories according to the nonlinearity implementation: {\bf BC-s/m} (Bounded Confidence with single/multiple thresholds) and {\bf MLP} (Multi-Layer Perceptron).}

    \resizebox{1.0\textwidth}{!}{%
    \begin{tabular}{l ccccccccc}
    \toprule 
         &
         \textbf{Texas} &  
         \textbf{Wisconsin} & 
         \textbf{Film} &
         \textbf{Squirrel} &
         \textbf{Chameleon} &
         \textbf{Cornell} &
         \textbf{Citeseer} & 
         \textbf{Pubmed} & 
         \textbf{Cora} \\
         
         Hom level &
         \textbf{0.11} &
         \textbf{0.21} & 
         \textbf{0.22} & 
         \textbf{0.22} & 
         \textbf{0.23} &
         \textbf{0.30} &
         \textbf{0.74} &
         \textbf{0.80} &
         \textbf{0.81} \\ 
         
         \#Nodes &
         183 &
         251 & 
         7,600 &
         5,201 & 
         2,277 &
         183 &
         3,327 &
         18,717 &
         2,708 \\
         
         \#Edges &
         295 &
         466 & 
         26,752 & 
         198,493 & 
         31,421 &
         280 &
         4,676 &
         44,327 &
         5,278 \\
         
         \#Classes &
         5 &
         5 & 
         5 &
         5 &
         5 & 
         5 &
         7 &
         3 &
         6 \\ \midrule

         \textbf{BC-s-Diag-NLSD} &
         $86.49 {\scriptstyle \pm 4.19}$ & 
         $88.24 {\scriptstyle \pm 3.40}$ & 
         $37.32 {\scriptstyle \pm 1.02}$ & 
         $49.26
 {\scriptstyle \pm 2.04}$ & 
         $64.69 {\scriptstyle \pm 1.34}$ & 
         ${86.49} {\scriptstyle \pm 6.40}$ & 
         $75.45 {\scriptstyle \pm 1.45}$ & 
         $89.28 {\scriptstyle \pm 0.36}$ & 
         $87.00 {\scriptstyle \pm 1.23}$ \\

         \textbf{BC-s-O(d)-NLSD} &
         $\second{87.30} {\scriptstyle \pm 4.99}$ & 
         $\first{89.41} {\scriptstyle \pm 2.66}$ & 
         $37.32 {\scriptstyle \pm 1.05}$ & 
         $52.33
 {\scriptstyle \pm 2.05}$ & 
         $65.42 {\scriptstyle \pm 1.30}$ & 
         $\third{86.76} {\scriptstyle \pm 5.98}$ & 
         $75.95 {\scriptstyle \pm 1.61}$ & 
         $89.39 {\scriptstyle \pm 0.43}$ & 
         $86.52 {\scriptstyle \pm 1.37}$ \\ 
         
         \textbf{BC-m-Diag-NLSD} &
         $86.22 {\scriptstyle \pm 3.91}$ & 
         $88.63 {\scriptstyle \pm 3.26}$ & 
         $37.19 {\scriptstyle \pm 0.87}$ & 
         $52.84 {\scriptstyle \pm 1.61}$ & 
         $66.54 {\scriptstyle \pm 1.05}$ & 
         ${86.49} {\scriptstyle \pm 5.54}$ & 
         $75.72 {\scriptstyle \pm 1.11}$ & 
         $89.24 {\scriptstyle \pm 0.33}$ & 
         $86.80 {\scriptstyle \pm 1.33}$ \\ 
         
         \textbf{BC-m-O(d)-NLSD} &
         $\first{87.57} {\scriptstyle \pm 5.43}$ & 
         $\second{89.22} {\scriptstyle \pm 3.54}$ & 
         $37.43 {\scriptstyle \pm 1.19}$ & 
         $\third{54.62} {\scriptstyle \pm 2.82}$ & 
         $66.27 {\scriptstyle \pm 2.27}$ & 
         $\first{87.30} {\scriptstyle \pm 6.74}$ & 
         $76.03 {\scriptstyle \pm 1.56}$ & 
         $89.34 {\scriptstyle \pm 0.38}$ & 
         $86.84 {\scriptstyle \pm 0.88}$ \\  \midrule

         \textbf{MLP-Diag-NLSD} &
         ${86.22} {\scriptstyle \pm 3.91}$ & 
         $\third{89.02} {\scriptstyle \pm 3.19}$ & 
         $\third{37.45} {\scriptstyle \pm 0.94}$ & 
         $47.63
 {\scriptstyle \pm 1.51}$ & 
         $62.57 {\scriptstyle \pm 2.31}$ & 
         $\third{86.76} {\scriptstyle \pm 4.60}$ & 
         $ 75.76 {\scriptstyle \pm 1.62}$ & 
         $89.52 {\scriptstyle \pm 0.32}$ & 
         $86.38 {\scriptstyle \pm 1.20}$ \\

         \textbf{MLP-O(d)-NLSD} &
         ${86.22} {\scriptstyle \pm 4.90}$ & 
         $\third{89.02} {\scriptstyle \pm 3.84}$ & 
         $37.22 {\scriptstyle \pm 1.15}$ & 
         $51.96
 {\scriptstyle \pm 2.65}$ & 
         $65.37 {\scriptstyle \pm 2.73}$ & 
         $\second{87.03} { \scriptstyle \pm 4.49}$ & 
         $76.11 {\scriptstyle \pm 1.81}$ & 
         $\third{89.60} {\scriptstyle \pm 0.29}$ & 
         $86.20 {\scriptstyle \pm 1.24}$ \\ \midrule

         NSP (best) &
         $\third{87.03} {\scriptstyle \pm 5.51}$ &
         $\third{89.02} {\scriptstyle \pm 3.84}$ &
         ${37.12} {\scriptstyle \pm 1.31}$ & 
         ${50.11} {\scriptstyle \pm 2.03}$ & 
         ${62.85} {\scriptstyle \pm 1.98}$ &
         ${76.49} {\scriptstyle \pm 5.28}$ & 
         ${76.85} {\scriptstyle \pm 1.48}$ &
         $89.42 {\scriptstyle \pm 0.33}$ &
         $87.38 {\scriptstyle \pm 1.14}$\\

         SheafAN (best) &
         $- {\scriptstyle}$ &
         $- {\scriptstyle }$ &
         $- {\scriptstyle}$ & 
         $- {\scriptstyle }$ & 
         $\second{70.77} {\scriptstyle \pm 1.42}$ &
         $85.68 {\scriptstyle \pm 4.53}$ & 
         $\first{78.02} {\scriptstyle \pm 1.15}$ &
         $- {\scriptstyle }$ &
         $\first{88.37} {\scriptstyle \pm 1.25}$ \\ 
         
         NSD (best) &
         ${85.95} {\scriptstyle \pm 5.51}$ &
         $\first{89.41} {\scriptstyle \pm 4.74}$ &
         $\first{37.81} {\scriptstyle \pm 1.15}$ & 
         $\first{56.34} {\scriptstyle \pm 1.32}$ & 
         $\third{68.68} {\scriptstyle \pm 1.73}$ &
         ${86.49} {\scriptstyle \pm 7.35}$ & 
         $\third{77.14} {\scriptstyle \pm 1.85}$ &
         ${89.49} {\scriptstyle \pm 0.40}$ &
         $87.30 {\scriptstyle \pm 1.15}$\\

         GGCN &
         ${84.86} {\scriptstyle \pm 4.55}$ &
         $86.86 {\scriptstyle \pm 3.29}$ &
         $\second{37.54} {\scriptstyle \pm 1.56}$ &
         $\second{55.17} {\scriptstyle \pm 1.58}$ & 
         $\first{71.14} {\scriptstyle \pm 1.84}$ &
         ${85.68} {\scriptstyle \pm 6.63}$ &
         $\third{77.14} {\scriptstyle \pm 1.45}$ &
         $89.15 {\scriptstyle \pm 0.37}$ &
         $\second{87.95} {\scriptstyle \pm 1.05}$ \\
         
         H2GCN &
         ${84.86} {\scriptstyle \pm 7.23}$ &
         ${87.65} {\scriptstyle \pm 4.98}$ &
         $35.70 {\scriptstyle \pm 1.00}$ &
         $36.48 {\scriptstyle \pm 1.86}$ & 
         $60.11 {\scriptstyle \pm 2.15}$ &
         $82.70 {\scriptstyle \pm 5.28}$ &
         $77.11 {\scriptstyle \pm 1.57}$ &
         ${89.49} {\scriptstyle \pm 0.38}$ &
         $\third{87.87} {\scriptstyle \pm 1.20}$ \\

         GPRGNN &
         $78.38 {\scriptstyle \pm 4.36}$ &
         $82.94 {\scriptstyle \pm 4.21}$ &
         $34.63 {\scriptstyle \pm 1.22}$ & 
         $31.61 {\scriptstyle \pm 1.24}$ &
         $46.58 {\scriptstyle \pm 1.71}$ &
         $80.27 {\scriptstyle \pm 8.11}$ &
         $77.13 {\scriptstyle \pm 1.67}$ &
         $87.54 {\scriptstyle \pm 0.38}$ &
         $\second{87.95} {\scriptstyle \pm 1.18}$ \\ 
         
         FAGCN &
         $82.43 {\scriptstyle \pm 6.89}$ &
         $82.94 {\scriptstyle \pm 7.95}$ &
         $34.87 {\scriptstyle \pm 1.25}$ &
         $42.59 {\scriptstyle \pm 0.79}$ &
         $55.22 {\scriptstyle \pm 3.19}$ & 
         $79.19 {\scriptstyle \pm 9.79}$ &
         N/A & 
         N/A & 
         N/A \\
         
         MixHop &
         $77.84 {\scriptstyle \pm 7.73}$ &
         $75.88 {\scriptstyle \pm 4.90}$ &
         $32.22 {\scriptstyle \pm 2.34}$ & 
         $43.80 {\scriptstyle \pm 1.48}$ &
         $60.50 {\scriptstyle \pm 2.53}$ &
         $73.51 {\scriptstyle \pm 6.34}$ &
         $76.26 {\scriptstyle \pm 1.33}$ &
         $85.31 {\scriptstyle \pm 0.61}$ &
         $87.61 {\scriptstyle \pm 0.85}$ \\

         GCNII &
         $77.57 {\scriptstyle \pm 3.83}$ &
         $80.39 {\scriptstyle \pm 3.40}$  &
         ${37.44} {\scriptstyle \pm 1.30}$ &
         $38.47 {\scriptstyle \pm 1.58}$ &
         $63.86 {\scriptstyle \pm 3.04}$ &
         $77.86 {\scriptstyle \pm 3.79}$ &
         $\second{77.33} {\scriptstyle \pm 1.48}$ &
         $\first{90.15} {\scriptstyle \pm 0.43}$ &
         $\first{88.37} {\scriptstyle \pm 1.25}$ \\
         
         Geom-GCN &
         $66.76 {\scriptstyle \pm 2.72}$ &
         $64.51 {\scriptstyle \pm 3.66}$ &
         $31.59 {\scriptstyle \pm 1.15}$ & 
         $38.15 {\scriptstyle \pm 0.92}$ & 
         $60.00 {\scriptstyle \pm 2.81}$ &
         $60.54 {\scriptstyle \pm 3.67}$ &
         $\first{78.02} {\scriptstyle \pm 1.15}$ &
         $\second{89.95} {\scriptstyle \pm 0.47}$ &  
         $85.35 {\scriptstyle \pm 1.57}$ \\ 
         
         PairNorm &
         $60.27 {\scriptstyle \pm 4.34}$ &
         $48.43 {\scriptstyle \pm 6.14}$ &
         $27.40 {\scriptstyle \pm 1.24}$ & 
         $50.44 {\scriptstyle \pm 2.04}$ & 
         $62.74 {\scriptstyle \pm 2.82}$ &
         $58.92 {\scriptstyle \pm 3.15}$ &
         $73.59 {\scriptstyle \pm 1.47}$ &
         $87.53 {\scriptstyle \pm 0.44}$ &
         $85.79 {\scriptstyle \pm 1.01}$ \\ 
         
         GraphSAGE &
         $82.43 {\scriptstyle \pm 6.14}$ &
         $81.18 {\scriptstyle \pm 5.56}$ &
         $34.23 {\scriptstyle \pm 0.99}$ & 
         $41.61 {\scriptstyle \pm 0.74}$ & 
         $58.73 {\scriptstyle \pm 1.68}$ &
         $75.95 {\scriptstyle \pm 5.01}$ &
         $76.04 {\scriptstyle \pm 1.30}$ &
         $88.45 {\scriptstyle \pm 0.50}$ &
         $86.90 {\scriptstyle \pm 1.04}$\\
         
         GCN &
         $55.14 {\scriptstyle \pm 5.16}$ &
         $51.76 {\scriptstyle \pm 3.06}$ &
         $27.32 {\scriptstyle \pm 1.10}$ & 
         $53.43 {\scriptstyle \pm 2.01}$ &
         $64.82 {\scriptstyle \pm 2.24}$ &
         $60.54 {\scriptstyle \pm 5.30}$ &
         $76.50 {\scriptstyle \pm 1.36}$ &
         $88.42 {\scriptstyle \pm 0.50}$ &
         $86.98 {\scriptstyle \pm 1.27}$ \\ 
         
         GAT &
         $52.16 {\scriptstyle \pm 6.63}$ &
         $49.41 {\scriptstyle \pm 4.09}$ &
         $27.44 {\scriptstyle \pm 0.89}$ & 
         $40.72 {\scriptstyle \pm 1.55}$ &
         $60.26 {\scriptstyle \pm 2.50}$ &
         $61.89 {\scriptstyle \pm 5.05}$ &
         $76.55 {\scriptstyle \pm 1.23}$ &
         $87.30 {\scriptstyle \pm 1.10}$ & 
         $86.33 {\scriptstyle \pm 0.48}$ \\ 
         
         MLP &
         $80.81 {\scriptstyle \pm 4.75}$ &
         $85.29 {\scriptstyle \pm 3.31}$ &
         $36.53 {\scriptstyle \pm 0.70}$ & 
         $28.77 {\scriptstyle \pm 1.56}$ & 
         $46.21 {\scriptstyle \pm 2.99}$ &
         $81.89 {\scriptstyle \pm 6.40}$ &
         $74.02 {\scriptstyle \pm 1.90}$ &
         $87.16 {\scriptstyle \pm 0.37}$ &
         $75.69 {\scriptstyle \pm 2.00}$ \\ 

         \bottomrule
         
    \end{tabular}
    }
    \vspace{-12pt}
    \label{tab:main_results}
\end{table*}

The initial version of our model that made use of the $\alpha$-normalization schema was also tested on a subset of the real-world benchmark datasets from Table \ref{tab:main_results}, and the results are in Appendix \ref{appendix_exp}.

\clearemptydoublepage

\newpage

\chapter{Conclusion}
\label{cha:conclusion}

This project specifically focused on investigating the benefits of introducing a nonlinear Laplacian in SNNs for graph-related tasks. Our curiosity about the impact of a nonlinearity in the Laplacian on diffusion dynamics, signal propagation, and network performance in discrete-time settings motivated such exploration.
Through experimental analysis, we aimed to validate the practical effectiveness of different versions of the model using real-world and synthetic datasets. 

Additionally, in this thesis, we conducted a comprehensive analysis of the implementation choices, thoroughly examining and justifying the various modifications made before defining a final version of the proposed model. We ensured that each decision was based on sound reasoning and aimed to enhance its performance and applicability.

The experimental results obtained from real-world settings were satisfactory, demonstrating the practical effectiveness of the proposed model. Furthermore, it is worth noting that the synthetic setting yielded particularly intriguing outcomes and this motivated further investigation into the underlying factors that generated this phenomenon. 
We leave future work open for exploring and delving deeper into this interesting behavior, which could lead to valuable insights and potentially uncover practical applications in real-world scenarios where the model can be effectively employed.

      \clearpage

    \endgroup

    %
    \addcontentsline{toc}{chapter}{Bibliography}
    \bibliographystyle{plain}
    \bibliography{biblio}

    \titleformat{\chapter}
        {\normalfont\Huge\bfseries}{Appendix \thechapter}{1em}{}
    \clearpage
    
    \appendix
    
\newpage

\chapter{Experiments}\label{appendix_exp}

\section{Synthetic Experiments: Erdős-Rényi Connections}

An alternative approach with respect to the one in Section \ref{sec:synthetic} that was explored for constructing the initial connections within the three communities in the synthetic datasets is through the Erdős-Rényi model for random graphs generation \cite{erdHos1960evolution}. With this method, each pair of nodes is independently connected with a predetermined probability, resulting in a random graph with varying connectivity patterns. We employed this model to initialize the graphs corresponding to each community, connecting pairs of the 500 nodes with a fixed probability of 0.005.

In contrast to constructing the initial graph using the \textit{k}-NN algorithm (Section \ref{sec:synthetic}), the Erdős-Rényi model adds edges randomly. As a result, these edges can connect nodes within the same community that have dissimilar node features. Such edges facilitate information flow within the community and contribute to improve node classification, as observed previously.

The dataset's sequence of graphs is constructed by retaining the initial set of edges and randomly adding noisy edges between the different communities. The results of the experiments are depicted in Figure \ref{fig:erdos_standard}. As anticipated, the performance of all models, including ours, tends to decrease with respect to the initial level, because both the initial edges and the newly added ones are random and may connect nodes with dissimilar node features. 

\begin{figure}[htp]
\centering
\centerline{\includegraphics[scale=0.9]{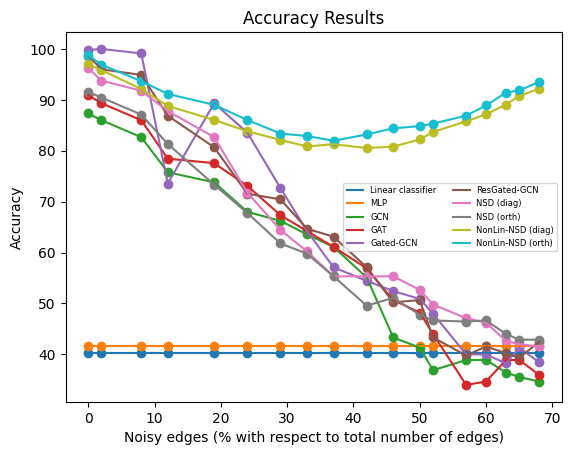}}
 \caption{Plots of the accuracy values on the sequence of datasets obtained from the three community graphs starting with Erdős-Rényi initial connections.}
 \label{fig:erdos_standard}
\end{figure}

While it was relatively easy for our models to differentiate between the two types of edges in the \textit{k}-NN setting, they face some difficulties in this case due to the inter-class edges not exclusively connecting nodes with similar features. However, the interesting thing is that when the percentage of noisy edges in the overall set is above a certain threshold, our models learn to distinguish them from the original edges within the communities, being able to  recognize and exploit them to effectively discriminate between the different communities at inference time.

We also conducted an additional test to examine the relevance of the relative distance in the mean values of the bivariate distributions used for sampling features. We wanted to understand how the choice of these parameters affects the performance of our model compared to others. To do this, we replicated the previous experiment but modified the datasets by redefining the means associated with the three communities as (0,0), (0.1, 0.1), and (0, 0.1). As a consequence, the three resulting distributions became almost completely overlapped. This additional complexity made the task even more challenging, as reflected in the results shown in Figure \ref{fig:erdos_close}.

Although this test only considered a noisy edge percentage of at most 35\% with respect to the total number of edges, it still showed that our model outperformed the other benchmarks as the percentage increased, even if the improvement was less pronounced in this case. This behavior mirrors the pattern observed in Figure \ref{fig:erdos_standard}, but the increased difficulty of the task makes the results more unstable.

\begin{figure}[htp]
\centering
\centerline{\includegraphics[scale=0.9]{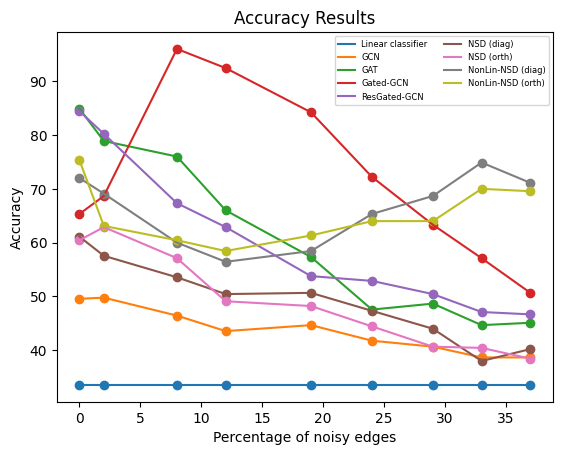}}
 \caption{Plots of the accuracy values on the sequence of datasets obtained from the three community graphs starting with Erdős-Rényi initial connections. In this case the means for the three distributions are set to (0,0), (0.1, 0.1) and (0, 0.1).}
 \label{fig:erdos_close}
\end{figure}

\paragraph{Results} The experiments conducted with the Erdős-Rényi initial connections reveal that our model surpasses the other benchmarks, even when both the original inter-class edges and newly added intra-class edges are defined using the same random procedure. This similarity in their placement does not hinder our model's ability to learn and utilize the "noisy" edges. When there are enough noisy edges to be recognized as a specific type that can be exploited, our model effectively identifies and takes advantage of them, distinguishing itself from the other benchmark models.

\newpage

\section{Real-World Experiments with $\alpha$-Normalization}


This section reports the results obtained by building and testing the initial version of our model, that is the Bounded Confidence dynamics performing $\alpha$-normalization.

The experiments were carried out in the same fashion and on the same datasets described in Section \ref{sec:real_world}, and considering the same benchmark models for comparison. 
The results are reported in Table \ref{tab:alpha}, and they are compared to the accuracy values obtained for the other versions of our model described in the Section devoted to real-world experiments (\ref{sec:real_world}). 
The model variants introduced in this occasion are:

\begin{itemize}
    \item \textbf{BC-s-Diag-$\alpha$NLSD}: Bounded Confidence model, with a single threshold value learnt for all pairs of neighboring nodes in the graph, with diagonal restriction maps, and \textbf{normalization through learned $\alpha$};
    \item \textbf{BC-s-O(d)-$\alpha$NLSD}: Bounded Confidence model, with a single threshold value learnt for all pairs of neighboring nodes in the graph, with orthogonal restriction maps, and \textbf{normalization through learned $\alpha$};
    \item \textbf{BC-m-Diag-$\alpha$NLSD}: Bounded Confidence model, with multiple threshold values for different couples of neighboring nodes defined as output of a MLP, with diagonal restriction maps, and \textbf{normalization through learned $\alpha$};
    \item \textbf{BC-m-O(d)-$\alpha$NLSD}: Bounded Confidence model, with multiple threshold values for different couples of neighboring nodes defined as output of a MLP, with orthogonal restriction maps, and \textbf{normalization through learned $\alpha$}.
    
\end{itemize}

\begin{table*}[th]
    \centering
    \caption{Results on node classification datasets sorted by their homophily level. Top three models are coloured by \textbf{\textcolor{red}{First}}, \textbf{\textcolor{blue}{Second}}, \textbf{\textcolor{violet}{Third}}. Our models are marked {\bf $\alpha$NLSD} (Non Linear Sheaf Diffusion with $\alpha$-normalization), and are split in two categories according to the nonlinearity implementation: {\bf BC-s/m} (Bounded Confidence with single/multiple thresholds) and {\bf MLP} (Multi-Layer Perceptron).}
    \resizebox{1.0\textwidth}{!}{%
    \begin{tabular}{l ccHccccHc}
    \toprule 
         &
         \textbf{Texas} &  
         \textbf{Wisconsin} & 
         \textbf{Film} &
         \textbf{Squirrel} &
         \textbf{Chameleon} &
         \textbf{Cornell} &
         \textbf{Citeseer} & 
         \textbf{Pubmed} & 
         \textbf{Cora} \\
         
         Hom level &
         \textbf{0.11} &
         \textbf{0.21} & 
         \textbf{0.22} & 
         \textbf{0.22} & 
         \textbf{0.23} &
         \textbf{0.30} &
         \textbf{0.74} &
         \textbf{0.80} &
         \textbf{0.81} \\ 
         
         \#Nodes &
         183 &
         251 & 
         7,600 &
         5,201 & 
         2,277 &
         183 &
         3,327 &
         18,717 &
         2,708 \\
         
         \#Edges &
         295 &
         466 & 
         26,752 & 
         198,493 & 
         31,421 &
         280 &
         4,676 &
         44,327 &
         5,278 \\
         
         \#Classes &
         5 &
         5 & 
         5 &
         5 &
         5 & 
         5 &
         7 &
         3 &
         6 \\ \midrule

         \textbf{BC-s-Diag-$\alpha$NLSD} &
         $\first{87.57} {\scriptstyle \pm 4.05}$ &
         $\second{89.22} {\scriptstyle \pm 2.01}$ &
         $- $ & 
         $36.20 {\scriptstyle \pm 3.00}$ & 
         $56.21 {\scriptstyle \pm 2.81}$ &
         $\third{87.01} {\scriptstyle \pm 6.02}$ & 
         $74.84 {\scriptstyle \pm 1.44}$ &
         $- $ &
         $86.04 {\scriptstyle \pm 1.34}$\\

         \textbf{BC-s-O(d)-$\alpha$NLSD} &
         $86.22 {\scriptstyle \pm 5.85}$ &
         $88.43 {\scriptstyle \pm 3.10}$ &
         $- $ & 
         $35.47 {\scriptstyle \pm 2.13}$ & 
         $55.33 {\scriptstyle \pm 3.01}$ &
         $- $ & 
         $75.05 {\scriptstyle \pm 1.68}$ &
         $- $ &
         $84.99 {\scriptstyle \pm 1.66}$\\
         
         \textbf{BC-m-Diag-$\alpha$NLSD} &
         $86.49 {\scriptstyle \pm 3.63}$ &
         $88.63 {\scriptstyle \pm 3.14}$ &
         $- $ & 
         $- $ & 
         $51.36 {\scriptstyle \pm 3.17}$ &
         $77.30 {\scriptstyle \pm 6.30}$ & 
         $74.58 {\scriptstyle \pm 1.46}$ &
         $- $ &
         $84.39 {\scriptstyle \pm 1.66}$\\
         
         \textbf{BC-m-O(d)-$\alpha$NLSD} &
         $86.22 {\scriptstyle \pm 4.90}$ &
         $88.24 {\scriptstyle \pm 3.40}$ &
         $- $ & 
         $- $ & 
         $48.77 {\scriptstyle \pm 2.52}$ & 
         $- $ &
         $74.00 {\scriptstyle \pm 1.73}$ &
         $- $ &
         $84.00 {\scriptstyle \pm 1.71}$\\ \midrule

         \textbf{BC-s-Diag-NLSD} &
         $86.49 {\scriptstyle \pm 4.19}$ & 
         $88.24 {\scriptstyle \pm 3.40}$ & 
         $37.32 {\scriptstyle \pm 1.02}$ & 
         $49.26
 {\scriptstyle \pm 2.04}$ & 
         $64.69 {\scriptstyle \pm 1.34}$ & 
         ${86.49} {\scriptstyle \pm 6.40}$ & 
         $75.45 {\scriptstyle \pm 1.45}$ & 
         $(t)89.28 {\scriptstyle \pm 0.36}$ & 
         $87.00 {\scriptstyle \pm 1.23}$ \\

         \textbf{BC-s-O(d)-NLSD} &
         $\second{87.30} {\scriptstyle \pm 4.99}$ & 
         $\first{89.41} {\scriptstyle \pm 2.66}$ & 
         $37.32 {\scriptstyle \pm 1.05}$ & 
         $52.33
 {\scriptstyle \pm 2.05}$ & 
         $65.42 {\scriptstyle \pm 1.30}$ & 
         ${86.76} {\scriptstyle \pm 5.98}$ & 
         $75.95 {\scriptstyle \pm 1.61}$ & 
         $(t)89.39 {\scriptstyle \pm 0.43}$ & 
         $86.52 {\scriptstyle \pm 1.37}$ \\ 
         
         \textbf{BC-m-Diag-NLSD} &
         $86.22 {\scriptstyle \pm 3.91}$ & 
         $88.63 {\scriptstyle \pm 3.26}$ & 
         $37.19 {\scriptstyle \pm 0.87}$ & 
         $52.84 {\scriptstyle \pm 1.61}$ & 
         $66.54 {\scriptstyle \pm 1.05}$ & 
         ${86.49} {\scriptstyle \pm 5.54}$ & 
         $75.72 {\scriptstyle \pm 1.11}$ & 
         $(t)89.24 {\scriptstyle \pm 0.33}$ & 
         $86.80 {\scriptstyle \pm 1.33}$ \\ 
         
         \textbf{BC-m-O(d)-NLSD} &
         $\first{87.57} {\scriptstyle \pm 5.43}$ & 
         $\second{89.22} {\scriptstyle \pm 3.54}$ & 
         $37.43 {\scriptstyle \pm 1.19}$ & 
         $\third{54.62} {\scriptstyle \pm 2.82}$ & 
         $66.27 {\scriptstyle \pm 2.27}$ & 
         $\first{87.30} {\scriptstyle \pm 6.74}$ & 
         $76.03 {\scriptstyle \pm 1.56}$ & 
         $(t)89.34 {\scriptstyle \pm 0.38}$ & 
         $86.84 {\scriptstyle \pm 0.88}$ \\  \midrule

         \textbf{MLP-Diag-NLSD} &
         ${86.22} {\scriptstyle \pm 3.91}$ & 
         $\third{89.02} {\scriptstyle \pm 3.19}$ & 
         $\third{37.45} {\scriptstyle \pm 0.94}$ & 
         $(t)47.63
 {\scriptstyle \pm 1.51}$ & 
         $62.57 {\scriptstyle \pm 2.31}$ & 
         ${86.76} {\scriptstyle \pm 4.60}$ & 
         $ 75.76 {\scriptstyle \pm 1.62}$ & 
         $89.52 {\scriptstyle \pm 0.32}$ & 
         $86.38 {\scriptstyle \pm 1.20}$ \\

         \textbf{MLP-O(d)-NLSD} &
         ${86.22} {\scriptstyle \pm 4.90}$ & 
         $\third{89.02} {\scriptstyle \pm 3.84}$ & 
         $37.22 {\scriptstyle \pm 1.15}$ & 
         $ (t)51.96
 {\scriptstyle \pm 2.65}$ & 
         $65.37 {\scriptstyle \pm 2.73}$ & 
         $\second{87.03} { \scriptstyle \pm 4.49}$ & 
         $76.11 {\scriptstyle \pm 1.81}$ & 
         $\third{89.60} {\scriptstyle \pm 0.29}$ & 
         $86.20 {\scriptstyle \pm 1.24}$ \\ \midrule

         NSP (best) &
         $\third{87.03} {\scriptstyle \pm 5.51}$ &
         $\third{89.02} {\scriptstyle \pm 3.84}$ &
         ${37.12} {\scriptstyle \pm 1.31}$ & 
         ${50.11} {\scriptstyle \pm 2.03}$ & 
         ${62.85} {\scriptstyle \pm 1.98}$ &
         ${76.49} {\scriptstyle \pm 5.28}$ & 
         ${76.85} {\scriptstyle \pm 1.48}$ &
         $89.42 {\scriptstyle \pm 0.33}$ &
         $87.38 {\scriptstyle \pm 1.14}$\\

         SheafAN (best) &
         $- {\scriptstyle}$ &
         $- {\scriptstyle }$ &
         $- {\scriptstyle}$ & 
         $- {\scriptstyle }$ & 
         $\second{70.77} {\scriptstyle \pm 1.42}$ &
         $85.68 {\scriptstyle \pm 4.53}$ & 
         $\first{78.02} {\scriptstyle \pm 1.15}$ &
         $- {\scriptstyle }$ &
         $\first{88.37} {\scriptstyle \pm 1.25}$ \\ 
         
         NSD (best) &
         ${85.95} {\scriptstyle \pm 5.51}$ &
         $\first{89.41} {\scriptstyle \pm 4.74}$ &
         $\first{37.81} {\scriptstyle \pm 1.15}$ & 
         $\first{56.34} {\scriptstyle \pm 1.32}$ & 
         $\third{68.68} {\scriptstyle \pm 1.73}$ &
         ${86.49} {\scriptstyle \pm 7.35}$ & 
         $\third{77.14} {\scriptstyle \pm 1.85}$ &
         ${89.49} {\scriptstyle \pm 0.40}$ &
         $87.30 {\scriptstyle \pm 1.15}$\\

         GGCN &
         ${84.86} {\scriptstyle \pm 4.55}$ &
         $86.86 {\scriptstyle \pm 3.29}$ &
         $\second{37.54} {\scriptstyle \pm 1.56}$ &
         $\second{55.17} {\scriptstyle \pm 1.58}$ & 
         $\first{71.14} {\scriptstyle \pm 1.84}$ &
         ${85.68} {\scriptstyle \pm 6.63}$ &
         $\third{77.14} {\scriptstyle \pm 1.45}$ &
         $89.15 {\scriptstyle \pm 0.37}$ &
         $\second{87.95} {\scriptstyle \pm 1.05}$ \\
         
         H2GCN &
         ${84.86} {\scriptstyle \pm 7.23}$ &
         ${87.65} {\scriptstyle \pm 4.98}$ &
         $35.70 {\scriptstyle \pm 1.00}$ &
         $36.48 {\scriptstyle \pm 1.86}$ & 
         $60.11 {\scriptstyle \pm 2.15}$ &
         $82.70 {\scriptstyle \pm 5.28}$ &
         $77.11 {\scriptstyle \pm 1.57}$ &
         ${89.49} {\scriptstyle \pm 0.38}$ &
         $\third{87.87} {\scriptstyle \pm 1.20}$ \\

         GPRGNN &
         $78.38 {\scriptstyle \pm 4.36}$ &
         $82.94 {\scriptstyle \pm 4.21}$ &
         $34.63 {\scriptstyle \pm 1.22}$ & 
         $31.61 {\scriptstyle \pm 1.24}$ &
         $46.58 {\scriptstyle \pm 1.71}$ &
         $80.27 {\scriptstyle \pm 8.11}$ &
         $77.13 {\scriptstyle \pm 1.67}$ &
         $87.54 {\scriptstyle \pm 0.38}$ &
         $\second{87.95} {\scriptstyle \pm 1.18}$ \\ 
         
         FAGCN &
         $82.43 {\scriptstyle \pm 6.89}$ &
         $82.94 {\scriptstyle \pm 7.95}$ &
         $34.87 {\scriptstyle \pm 1.25}$ &
         $42.59 {\scriptstyle \pm 0.79}$ &
         $55.22 {\scriptstyle \pm 3.19}$ & 
         $79.19 {\scriptstyle \pm 9.79}$ &
         N/A & 
         N/A & 
         N/A \\
         
         MixHop &
         $77.84 {\scriptstyle \pm 7.73}$ &
         $75.88 {\scriptstyle \pm 4.90}$ &
         $32.22 {\scriptstyle \pm 2.34}$ & 
         $43.80 {\scriptstyle \pm 1.48}$ &
         $60.50 {\scriptstyle \pm 2.53}$ &
         $73.51 {\scriptstyle \pm 6.34}$ &
         $76.26 {\scriptstyle \pm 1.33}$ &
         $85.31 {\scriptstyle \pm 0.61}$ &
         $87.61 {\scriptstyle \pm 0.85}$ \\

         GCNII &
         $77.57 {\scriptstyle \pm 3.83}$ &
         $80.39 {\scriptstyle \pm 3.40}$  &
         ${37.44} {\scriptstyle \pm 1.30}$ &
         $38.47 {\scriptstyle \pm 1.58}$ &
         $63.86 {\scriptstyle \pm 3.04}$ &
         $77.86 {\scriptstyle \pm 3.79}$ &
         $\second{77.33} {\scriptstyle \pm 1.48}$ &
         $\first{90.15} {\scriptstyle \pm 0.43}$ &
         $\first{88.37} {\scriptstyle \pm 1.25}$ \\
         
         Geom-GCN &
         $66.76 {\scriptstyle \pm 2.72}$ &
         $64.51 {\scriptstyle \pm 3.66}$ &
         $31.59 {\scriptstyle \pm 1.15}$ & 
         $38.15 {\scriptstyle \pm 0.92}$ & 
         $60.00 {\scriptstyle \pm 2.81}$ &
         $60.54 {\scriptstyle \pm 3.67}$ &
         $\first{78.02} {\scriptstyle \pm 1.15}$ &
         $\second{89.95} {\scriptstyle \pm 0.47}$ &  
         $85.35 {\scriptstyle \pm 1.57}$ \\ 
         
         PairNorm &
         $60.27 {\scriptstyle \pm 4.34}$ &
         $48.43 {\scriptstyle \pm 6.14}$ &
         $27.40 {\scriptstyle \pm 1.24}$ & 
         $50.44 {\scriptstyle \pm 2.04}$ & 
         $62.74 {\scriptstyle \pm 2.82}$ &
         $58.92 {\scriptstyle \pm 3.15}$ &
         $73.59 {\scriptstyle \pm 1.47}$ &
         $87.53 {\scriptstyle \pm 0.44}$ &
         $85.79 {\scriptstyle \pm 1.01}$ \\ 
         
         GraphSAGE &
         $82.43 {\scriptstyle \pm 6.14}$ &
         $81.18 {\scriptstyle \pm 5.56}$ &
         $34.23 {\scriptstyle \pm 0.99}$ & 
         $41.61 {\scriptstyle \pm 0.74}$ & 
         $58.73 {\scriptstyle \pm 1.68}$ &
         $75.95 {\scriptstyle \pm 5.01}$ &
         $76.04 {\scriptstyle \pm 1.30}$ &
         $88.45 {\scriptstyle \pm 0.50}$ &
         $86.90 {\scriptstyle \pm 1.04}$\\
         
         GCN &
         $55.14 {\scriptstyle \pm 5.16}$ &
         $51.76 {\scriptstyle \pm 3.06}$ &
         $27.32 {\scriptstyle \pm 1.10}$ & 
         $53.43 {\scriptstyle \pm 2.01}$ &
         $64.82 {\scriptstyle \pm 2.24}$ &
         $60.54 {\scriptstyle \pm 5.30}$ &
         $76.50 {\scriptstyle \pm 1.36}$ &
         $88.42 {\scriptstyle \pm 0.50}$ &
         $86.98 {\scriptstyle \pm 1.27}$ \\ 
         
         GAT &
         $52.16 {\scriptstyle \pm 6.63}$ &
         $49.41 {\scriptstyle \pm 4.09}$ &
         $27.44 {\scriptstyle \pm 0.89}$ & 
         $40.72 {\scriptstyle \pm 1.55}$ &
         $60.26 {\scriptstyle \pm 2.50}$ &
         $61.89 {\scriptstyle \pm 5.05}$ &
         $76.55 {\scriptstyle \pm 1.23}$ &
         $87.30 {\scriptstyle \pm 1.10}$ & 
         $86.33 {\scriptstyle \pm 0.48}$ \\ 
         
         MLP &
         $80.81 {\scriptstyle \pm 4.75}$ &
         $85.29 {\scriptstyle \pm 3.31}$ &
         $36.53 {\scriptstyle \pm 0.70}$ & 
         $28.77 {\scriptstyle \pm 1.56}$ & 
         $46.21 {\scriptstyle \pm 2.99}$ &
         $81.89 {\scriptstyle \pm 6.40}$ &
         $74.02 {\scriptstyle \pm 1.90}$ &
         $87.16 {\scriptstyle \pm 0.37}$ &
         $75.69 {\scriptstyle \pm 2.00}$ \\  

         \bottomrule
         
    \end{tabular}
    }
    \vspace{-12pt}
    \label{tab:alpha}
\end{table*}

The results show that these models perform quite satisfactorily in the case of small datasets (e.g. Texas, Winsconsin), but their performance drops significantly when dealing with datasets with a very high amount of edges (e.g. Squirrel, Chameleon). In the latter cases, the accuracy values are even worse with respect to the simplest GNN benchmarks, such as GCN \cite{kipf2016semi} and GAT \cite{velivckovic2017graph}. This is caused by the normalization method being not enough to stabilize the diffusion process when extremely dense datasets are involved. The inability of handling such datasets properly lead to trying out new kinds of normalization, until coming up with the idea of using the squared diagonal matrices, which turned out to be the technique working best in practice for every graph configuration. 

\chapter{Ablation Studies}

\section{Diag-NSD vs BC-s-Diag-$\alpha$NLSD Components}

An ablation study was conducted to assess the importance of different components in the layer-wise operations. This study compares the behavior of the Diagonal NSD model \cite{bodnar2022neural} (with a linear Laplacian) to one of the initial models developed for this project, namely the Diagonal Bounded Confidence model with $\alpha$-normalization and a single learned threshold value for all edges in the graph.

The ablation study focuses on examining the impact of the following components:

 \begin{itemize}
     \item \textbf{Layer-dependent Laplacian}:     when this component is present, a different set of restriction maps is learned for each layer, meaning that a distinct MLP (from Proposition \ref{prop_mlp}) is learned. Otherwise, the same restriction maps are used for all layers.
     \item $\textbf{W}_{2}$: this component refers to the presence or absence of the right weight matrix.
     \item \textbf{Activation function}: this component pertains to the inclusion or exclusion of the activation function in the layer-wise operations, corresponding to $\sigma$ in Equations \ref{eq:nonlin_expression_general} and \ref{eq:diffusion_bodnar_discrete}.
 \end{itemize}
 
The experiments were performed on the Chameleon dataset with the same methodology as the real-world experiments detailed in Table \ref{tab:alpha}. This involved hyperparameter tuning and selecting the test accuracy associated with the best validation accuracy.

\begin{table*}[th]
    \centering
    \caption{Ablation study on Chameleon, with a comparison between a model with linear Laplacian (\textbf{Diag-NSD}) and one with nonlinear Laplacian (\textbf{BC-s-Diag-$\alpha$NLSD}).}
    \resizebox{1.0\textwidth}{!}{%
    \begin{tabular}{ccccc}
    \toprule 
         \textbf{Layer-dependent Laplacian} &  
         $\textbf{W}_{2}$ & 
         \textbf{Activation function $\sigma$} &
         \textbf{Diag-NSD} &
         \textbf{BC-s-Diag-$\alpha$NLSD} \\
         \midrule

         \xmark & \xmark & \xmark & 
         ${68.68} {\scriptstyle \pm 1.73}$ &          $64.69 {\scriptstyle \pm 1.34}$ \\

         \xmark &  &  &          ${67.47} {\scriptstyle \pm 1.69}$ &          $65.79 {\scriptstyle \pm 1.53}$ \\

          & \xmark &  &          ${67.06} {\scriptstyle \pm 1.59}$ &          $64.52 {\scriptstyle \pm 1.72}$ \\

          &  & \xmark &          ${68.68} {\scriptstyle \pm 2.15}$ &          $65.46 {\scriptstyle \pm 1.36}$ \\

         \xmark & \xmark &  &          ${67.06} {\scriptstyle \pm 1.23}$ &          $64.41 {\scriptstyle \pm 1.70}$ \\

          & \xmark & \xmark &          ${67.89} {\scriptstyle \pm 1.08}$ &          $64.50 {\scriptstyle \pm 1.39}$ \\

         \xmark &  & \xmark &          ${67.61} {\scriptstyle \pm 1.67}$ &          $66.03 {\scriptstyle \pm 1.31}$ \\
         
          &  &  &          ${67.87} {\scriptstyle \pm 1.30}$ &          $65.21 {\scriptstyle \pm 1.48}$ \\   
         \midrule
         \bottomrule
         
    \end{tabular}
    }
    \vspace{-12pt}
    \label{tab:ablation}
\end{table*}

\

Surprisingly, the results of the ablation study demonstrated that the presence or absence of the considered  components do not significantly impact the behavior of either model: the findings revealed comparable performance across different configurations.

\end{document}